\documentclass[twoside,11pt]{article}
\usepackage{jmlr2e}

\ShortHeadings{Path Signatures on Lie Groups}{Darrick Lee and Robert Ghrist}
\firstpageno{1}

\usepackage{graphicx}
\usepackage{color}
\usepackage{amsmath}
\usepackage{shuffle}
\usepackage{mathtools}
\usepackage{tikz-cd}
\usepackage{subfigure}
\usepackage[linesnumbered,ruled,vlined]{algorithm2e}

\makeatletter
\newcommand{\nosemic}{\renewcommand{\@endalgocfline}{\relax}}
\newcommand{\dosemic}{\renewcommand{\@endalgocfline}{\algocf@endline}}
\newcommand{\pushline}{\Indp}
\newcommand{\popline}{\Indm\dosemic}
\let\oldnl\nl
\newcommand{\nonl}{\renewcommand{\nl}{\let\nl\oldnl}}
\makeatother

\newcommand{\N}{\mathbb{N}}

\newcommand{\R}{\mathbb{R}}

\newcommand{\E}{\mathbb{E}}

\newcommand{\fg}{\mathfrak{g}}
\newcommand{\fso}{\mathfrak{so}}

\newcommand{\fr}{\mathfrak{r}}
\newcommand{\fh}{\mathfrak{h}}

\newcommand{\cP}{\mathcal{P}}

\newcommand{\cF}{\mathcal{F}}

\newcommand{\cH}{\mathcal{H}}
\newcommand{\cX}{\mathcal{X}}
\newcommand{\cM}{\mathcal{M}}

\newcommand{\MMD}{\mathrm{MMD}}
\newcommand{\IdInit}{\mathrm{IdInit}}
\newcommand{\timerm}{\mathrm{Time}}
\newcommand{\SW}{\mathrm{SWin}}

\newcommand{\style}[1]{{\em #1}}

\title{Path Signatures on Lie Groups}
\author{\name Darrick Lee \email ldarrick@sas.upenn.edu \\
       \addr Department of Mathematics \\ University of Pennsylvania \\ Philadelphia, PA 19104, USA
       \AND
       \name Robert Ghrist \email ghrist@seas.upenn.edu \\
       \addr Departments of Mathematics and Electrical \& Systems Engineering \\ University of Pennsylvania\\ Philadelphia, PA 19104, USA}
\editor{}

\begin{document}

\maketitle

\begin{abstract}%
    Path signatures are powerful nonparametric tools for time series analysis, shown to form a universal and characteristic feature map for Euclidean valued time series data. We lift the theory of path signatures to the setting of Lie group valued time series, adapting these tools for time series with underlying geometric constraints. We prove that this generalized path signature is universal and characteristic. To demonstrate universality, we analyze the human action recognition problem in computer vision, using $SO(3)$ representations for the time series, providing comparable performance to other shallow learning approaches, while offering an easily interpretable feature set. We also provide a two-sample hypothesis test for Lie group-valued random walks to illustrate its characteristic property. Finally we provide algorithms and a Julia implementation of these methods. 
\end{abstract}



\section{Introduction}
Time series data is ubiquitous in modern data science, and may take values in a variety of forms. Perhaps the most common is a collection of simultaneous multivariate real-valued time series $\{\gamma^i\}_{i=1}^N$, where $\gamma^i : [0,1] \rightarrow \R$. In this case, we may consider the entire collection $\gamma = (\gamma^1, \ldots, \gamma^N)$ as a path through Euclidean space, $\gamma: [0,1] \rightarrow \R^N$. The \style{path signature} is a feature set that completely characterizes such paths, and has recently been applied to several tasks in machine learning~\citep{chevyrev_primer_2016, lyons_rough_2014}. Recent work has provided the path signature with strong theoretical properties; namely that it is a universal and characteristic kernel for time series in Euclidean space $\R^N$~\citep{chevyrev_signature_2018}.

However, in many scenarios, the data may have some geometric constraints, and may be better represented by elements of a (non-Euclidean) manifold. In this case, the time-varying data can be modelled as path on such a manifold, rather than on Euclidean space. \style{Lie groups} are smooth manifolds equipped with a compatible group structure. Paths (or time series) valued in Lie groups model a number of natural phenomena, including the following.

\begin{itemize}
    \item The \style{special Euclidean group} $SE(n)$ is the Lie group of all rigid body motions in $\R^n$. The group $SE(3)$ is often used to model the position and pose of a rigid body, such as a component of a robotic arm or an element of a drone swarm, with $k$ such components or elements collectively giving rise to a path $SE(3)^k$ ~\citep{selig_lie_2004}.
    
    \item The \style{special orthogonal group} $SO(n)$ is the Lie group of all rotations in $\R^n$; this is a Lie subgroup of $SE(n)$. The Lie group $SO(3)^k$ has recently been used to represent the pose of a human by recording the relative rotations of $k$ pairs of body parts~\citep{vemulapalli_rolling_2016}. Thus, human movement can be represented as a path in $SO(3)^k$. This representation has been used in the computer vision problem of human action recognition, and Lie group methods have achieved state-of-the-art results in this domain~\citep{huang_deep_2017}. 
    

    \item The state of an oscillator may be described as an element of the \style{circle} $S^1$, and collective behavior of a network of oscillators can be describe by an element of the \style{$n$-torus}, $T^n = (S^1)^n$. The time evolution of oscillator networks can therefore be modelled as a path on $T^n$ ~\citep{strogatz_kuramoto_2000}. 

    \item The Euclidean space $\R^N$ is the simplest example of a Lie group, where the group operation is addition. The classical path signature for Euclidean space can be viewed as a special case of path signatures on Lie groups. 
\end{itemize}

In this paper, we extend path signatures to time series valued in Lie groups, and show that this extension is also a universal and characteristic kernel. 



\subsection{Contributions}

We lift the theory of path signatures for time series valued in Euclidean space to the setting of time series valued in Lie groups, restricting ourselves to the class of piecewise regular paths on Lie groups.

\begin{definition}
    Let $G$ be a Lie group. A path $\gamma: [a,b] \rightarrow G$ is \style{regular} if $\gamma'_t$ is continuous and nonvanishing on the entire interval $[a,b]$. Such a path is \style{piecewise regular} if there exists a partition $a = t_0 < t_1 < \ldots < t_n = b$ such that $\gamma$ is regular on each open subinterval $(t_i, t_{i+1})$ for all $i$. The \style{pathspace} -- the space of all piecewise regular paths on the unit interval, $\gamma: [0,1] \rightarrow G$ -- will be denoted $PG$.
\end{definition}

Let $G$ be a Lie group of dimension $N$, and let $\fg$ be its Lie algebra (the tangent space at the identity). We denote the underlying vector space of $\fg$ by $\overline{\fg} \cong \R^N$. The \style{path signature} is a function on paths,
\begin{align*}
    S: PG \rightarrow T((\overline{\fg})) ,
\end{align*}
valued in a formal power series of tensors, $T((\overline{\fg}))$, where we may view the coefficients as descriptors (or features) of the underlying path (or time series). Path signatures for general manifolds were originally defined by~\cite{chen_integration_1958}, but not in a manner conducive to data analysis. Path signatures for Lie group valued data have been previously considered by~\cite{celledoni_signatures_2019} in a preliminary empirical study, showing promising qualitative classification results, but extensions of theoretical results and detailed quantitative comparisons were not provided. This paper gives a computationally clean derivation for path signatures on Lie groups tuned for use in data analysis, and provides a thorough discussion of its theoretical properties in the context of kernel methods. \medskip

Our generalization is designed to be analogous to the Euclidean case as much as possible, for ease of applicability. For example, the definition of the path signature for $\gamma: [0,1] \rightarrow G$ depends only on the derivative $\gamma':[0,1] \rightarrow \fg$. We exploit one of the key properties of Lie groups --- that tangent vectors at a point correspond to elements of its Lie algebra $\fg$, a vector space. This will permit a signature construction making use of iterated integrals as per the Euclidean case.

In the Euclidean case $G = \R^N$, the Lie group is often conflated with its Lie algebra $\fr = \R^N$, and the fact that the integration is performed in the Lie algebra is often not made. By clarifying and emphasizing this point, the generalization to Lie groups illuminates understanding of the classical Euclidean case. \medskip

From a machine learning perspective, the basic properties of the path signature as a feature map provide several benefits.

\begin{itemize}
    \item The signature is a feature set for a path as a whole, and can be used to compare time series with varying numbers of time points.
    
    \item Defined as iterated line integrals, the path signature is invariant under reparametrization, and thus only depends on the \style{order} in which events occur.
    
    \item The signature is left translation invariant, meaning the signatures of paths that differ by a constant element $g \in G$ will be the same. This implies that the signature only depends on the dynamics of the time series and is unconcerned with the initial point.
    
    \item The antisymmetrization of the second degree signature tensor can be viewed as an indicator of lead-lag behavior in the time series. In the case of Lie groups, the interpretation will be considered in terms of left-invariant vector fields. 
\end{itemize}

However, the most crucial property is that the path signature fully characterizes paths up to tree-like equivalence; that is, the map $S$ is injective, up to quotienting $PG$ out by an equivalence relation. This fact is originally due to~\cite{chen_integration_1958} for the case of piecewise regular paths on Lie groups, and later generalized by~\cite{hambly_uniqueness_2010} to the case of bounded variation paths in $\R^n$. \medskip

Our main contribution is to apply this injectivity result to prove that a normalized variant of the signature, $S: PG \rightarrow T((\overline{\fg}))$, is a \style{universal} and \style{characteristic} feature map for time series in $G$, when we equip $T((\overline{\fg}))$ with the structure of a Hilbert space. This is proved in Section~\ref{ssec:universal}. This was originally shown for the Euclidean case by~\cite{chevyrev_signature_2018}. Such feature maps can be used to  two large classes of machine learning problems, in the context of kernel methods.

\begin{enumerate}
    \item \textbf{(Studying functions on $PG$)} Solving a classification problem on $PG$ can be reduced to finding a function $f : PG \rightarrow \R$ such that the level set $f = 0$ provides the decision boundary. The \style{universality} of the normalized signature map states that any continuous bounded function $f: PG \rightarrow \R$ can be approximated using a linear functional $f(\cdot) \approx \langle \ell, \tilde{S}(\cdot)\rangle$. This allows us to reduce a nonlinear optimization problem into a linear one, greatly reducing the complexity.
    
    \item \textbf{(Studying measures on $PG$)} Two-sample hypothesis testing on $PG$ requires the computation of a set of statistics that is rich enough to distinguish any two probability measures on $PG$. The \style{characteristicness} of the normalized signature map states that the kernel mean embedding (KME) is injective with respect to the normalized signature
    \begin{align*}
        \overline{\Phi}: \cM(PG) \rightarrow T((\overline{\fg})), \quad \overline{\Phi}(\mu) = \E_\mu[S],
    \end{align*}
    where $\cM(PG)$ denotes all finite regular Borel measures on $PG$, and $S$ is appropriately normalized. This allows us to consider probability measures as elements of a linear space; furthermore, the norm induced by the Hilbert space structure coincides with the maximum mean discrepancy (MMD) between measures.
\end{enumerate}

We perform two experiments that demonstrate the efficacy of the path signature for these two classes of problems. First, we consider the computer vision problem of human action recognition in Section~\ref{ssec:g3d}. We show that the path signature method is much easier to use than shallow learning methods previously applied to this problem~\citep{vemulapalli_human_2014, vemulapalli_rolling_2016} while providing comparable results. Second, in Section~\ref{ssec:randomwalk}, we consider a hypothesis testing problem for simulated random walks on the Lie group $SO(3)$. Here, we show that the Lie group valued path signature vastly outperforms the Euclidean path signature. \medskip

Along the way, we will establish extensions of other properties of the path signatures to Lie groups and discuss several concepts related to path signatures and data analysis on Lie groups more broadly. A summary of these contributions is given below. 

\begin{enumerate}
    \item We provide a detailed exposition of Lie group valued time series, and discuss a notion of scaling for such time series in Section~\ref{ssec:paths_liegroups}. Scaling of data is sometimes required when the data needs to be normalized, and we discuss how scaling affects the path signature in Section~\ref{ssec:pathsignature}. We also discuss the continuous interpretation of discrete time series on Lie groups in Section~\ref{ssec:discretets}.
    
    \item For $G$ an $N$-dimensional Lie group, we give a signature-preserving bijection between $PG$ and $P\R^N$ in Section~\ref{ssec:relationship_g_rn}, which provides a Euclidean representation of Lie group valued time series. This bijection allows the exportation of Euclidean data analysis tools to Lie group valued data. With the metric introduced in Section~\ref{ssec:stability}, this bijection is an isometry.
    
    \item 
    It is well known that the Euclidean path signature is equivariant with respect to linear transformations~\citep{friz_multidimensional_2010}. We show that path signatures are equivariant under Lie group homomorphisms in general. Namely, given a homomorphism of Lie groups $F: G_1 \rightarrow G_2$, where $\fg_1$ and $\fg_2$ are the respective Lie algebras, we define the action of this homomorphism on the tensor algebra $F_*: T((\overline{\fg}_1)) \rightarrow T((\overline{\fg}_2))$, and show in Section~\ref{ssec:equivariance} that
    \begin{align*}
        S(F \gamma) = F_* S(\gamma)
    \end{align*}
    for all $\gamma \in PG$.
    
    \item An important feature of the path signature is the interpretability of lower level signature terms. We discuss the extension of the lead-lag interpretation of second level signature terms for Euclidean paths, as well as a topological interpretation of the first level signature terms for abelian Lie groups in Section~\ref{ssec:topological}.
    
    \item Path transformations, such as appending the time parameter or using a sliding window, are often used as a preprocessing step for Euclidean path signatures~\citep{chevyrev_primer_2016}. We discuss these transformations in the context of breaking reparametrization or left-translation invariance in Section~\ref{ssec:pathtrans}. Empirical studies~\citep{fermanian_embedding_2019} have shown that the sliding window transformation (also called the lead-lag transformation) provides good classification results, despite the lack of a theoretical explanation. We propose one explanation, which is that the sliding window transformation breaks left-translation invariance, and we provide empirical evidence in the experiments in Section~\ref{ssec:g3d}.
    
    \item We provide both algorithmic details and a Julia package for the computation of path signatures valued in Lie groups, which can be found at \url{https://github.com/ldarrick/PathSignatures}. For details, see Appendix~\ref{apx:implementation}.
\end{enumerate}



\subsection{Previous and related work}
The concept of path signatures is relatively new in data science and machine learning~\citep{lyons_rough_2014,chevyrev_primer_2016, giusti_iterated_2020}, but has deep roots in topology and geometry.  Chen originally defined the path signature for piecewise regular paths on manifolds and proved several basic properties in a sequence of papers~\citep{chen_iterated_1954, chen_integration_1957, chen_integration_1958}. He later studied the geometry and topology of path spaces and loop spaces by constructing a rational cochain model of these spaces, in which path signatures constitute $0$-cochains~\citep{chen_iterated_1977}.

\cite{lyons_differential_1998} developed the concept of the path signature in a different direction, using the path signature as a construction to lift bounded variation paths on $\R^N$ to paths of power series of tensors $T((\R^N))$. This initiated the study of \style{rough paths}, which can be thought of as a generalization of the path signature to highly irregular paths. This theory was then used to study stochastic processes and stochastic differential equations~\citep{lyons_system_2007, lyons_differential_2007, friz_multidimensional_2010}.

Within machine learning, path signatures have been used to study real-valued time series data in a variety of settings. Examples can be found in the study of financial time series~\citep{gyurko_extracting_2013, lyons_feature_2014}, handwritten character recognition~\citep{yang_deepwriterid_2016}, human action recognition using position data~\citep{yang_developing_2019}, identifying psychological or neurological disorders~\citep{moore_using_2019, zimmerman_dissociating_2018, arribas_signature-based_2017} and featurizing the output of persistent homology in topological data analysis~\citep{chevyrev_persistence_2020}. Additionally, experiments with path signatures on Lie groups have previously been performed~\citep{celledoni_signatures_2019}, though theoretical results were not provided, and thus suggests further study.

The theoretical aspects of path signatures in the context of kernel methods were developed in~\cite{kiraly_kernels_2019} and~\cite{chevyrev_signature_2018}. The present paper is largely inspired by these two papers. The concept of using the path signature as a kernel for time series was first proposed in~\cite{kiraly_kernels_2019}, and efficient algorithms for computing the kernel were developed. The path signature for Euclidean space was shown to be a universal and characteristic feature map in~\cite{chevyrev_signature_2018}. This exploits the recently formalized duality between universal and characteristic kernels in~\cite{simon-gabriel_kernel_2018}. 

It is well known that Euclidean path signatures are translation invariant, and we will show that Lie group path signatures are left translation invariant. \cite{diehl_invariants_2019} has considered the related problem of determining the Euclidean path signature terms which are invariant under some matrix Lie group action. 

\medskip

We begin in Section~\ref{sec:paths_on_lie_groups} by reviewing basic facts on Lie groups and Lie algebras, and provide an exposition on continuous and discrete time series on Lie groups. We then define the path signature for Lie groups in Section~\ref{sec:pathsignature_on_lie_groups}, and discuss the bijection between $PG$ and $P\R^N$, the equivariance of the path signature, detecting lead-lag behavior in time series, and path transformations. In Section~\ref{sec:kernel}, we provide a brief overview of kernel methods and prove our main result, which shows that the path signature kernel is universal and characteristic. Finally, in Section~\ref{sec:experiments}, we apply the path signature on Lie groups to a human action classification problem and a hypothesis testing problem involving random walks on $SO(3)$.

\subsection{Notation}
Throughout this paper, we will denote the time parameter for a path $\gamma: [0,1] \rightarrow G$ using a subscript $t$, meaning $\gamma_t \coloneqq \gamma(t)$. Derivatives are shown using the prime notation, as in $\gamma'_t \coloneqq \frac{d\gamma}{dt} (t)$. If we have a path in Euclidean space $\alpha:[0,1] \rightarrow \R^N$, we will use superscripts to represent the components, such as $\alpha = (\alpha^1, \alpha^2, \ldots, \alpha^N)$. If $G$ is a Lie group, we will use $\fg$ to denote its Lie algebra and use $\bar{\fg}$ to be the underlying vector space of $\fg$ (forgetting the Lie bracket structure). 

Continuous paths will often be denoted using the lowercase Greek symbols $\alpha, \beta, \gamma$, and the space of all piecewise regular paths in $G$ is denoted $PG$. For $T \in \N$, we let $[T] = \{1, \ldots, T\}$ denote the finite set of integers up to $T$. Discrete time series will be distinguished using the hat notation $\hat{\gamma}:[T] \rightarrow G$, and the space of all discrete time series in $G$ will be denoted $\hat{P}G$.  

There are also several parameters that will be used consistently throughout the paper. Unless otherwise specified, we reserve the following symbols for the given meaning. 
\begin{itemize}
    \item $N$ is the dimension of the Lie group $G$ that paths take values in;
    \item $T+1$ is the length of a discrete time series (so that the discrete derivative will be of length $T$);
    \item $M$ is the level of the truncated signature.
\end{itemize}


\section{Lie groups, paths, and time series}
\label{sec:paths_on_lie_groups}
We begin this section by recalling several basic facts about Lie groups~\cite{alexandrino_lie_2015}, followed by paths on Lie groups and the interpretation of sampled time series on Lie groups, stressing the differences from sampled time series on Euclidean space.


\subsection{A review of Lie groups}
Recall that a \style{Lie group} $G$ is a smooth manifold with a group structure such that the multiplication and inversion maps are both smooth. Let $g_0 \in G$. The \style{left translation map by $g_0$}, written as $L_{g_0}: G \rightarrow G$, is defined to be $L_{g_0}(g) = g_0 g$. The \style{right translation map} $R_{g_0}:G \rightarrow G$ is defined analogously. This induces a mapping on tangent spaces $L_{g_0*}: T_g G \rightarrow T_{g_0 g} G$. A vector field $X$ on $G$ is called \style{left-invariant} if
\begin{equation*}
    L_{g_0*} X(g) = X(L_{g_0}g) = X(g_0 g)
\end{equation*}
for all $g_0, g \in G$. This implies that all left-invariant vector fields $X$ are defined by their value at the identity $e \in G$,
\begin{equation*}
    X(g) = L_{g*} X(e),
\end{equation*}
and thus, we obtain a one-to-one correspondence between left-invariant vector fields and the tangent space at the identity, which we denote by $\fg \coloneqq T_e G$. Vector fields act on smooth functions $f: G \rightarrow \R$, and we define an operation of left-invariant vector fields $X$ and $Y$ by
\begin{align*}
    [X,Y](f) \coloneqq X(Y(f)) - Y(X(f)),
\end{align*}
where $[X,Y]$ is also left-invariant. This provides $\fg$ with the structure of a Lie algebra, where the Lie bracket $[\cdot, \cdot] : \fg \times \fg \rightarrow \fg$ is a bilinear mapping such that for all $ X,Y,Z \in \fg$
\begin{align*}
    [X, Y] & = -[X,Y] \\
    [X,[Y,Z]] + [Z,[X,Y]] + [Y,[Z,X]] & = 0.
\end{align*}

Similarly, left translation induces a map $L_{g_0}^* : T_{g_0 g}^* G \rightarrow T_g^* G$ on cotangent spaces. A $1$-form $\omega \in T^* G$ is called \style{left-invariant} if
\begin{equation*}
    L_{g_0}^*\omega(g_0 g) 
    = L_{g_0}^*\omega(L_{g_0} g)
    = \omega(g)
\end{equation*}
for all $g, g_0 \in G$. Again, we obtain a correspondence between left-invariant 1-forms and the cotangent space at the identity via the property
\begin{equation*}
    \omega(g) = L_{g^{-1}}^*\omega(e).
\end{equation*}
Thus, we may identify the left-invariant 1-forms by the dual of the Lie algebra, $\fg^*$.


\begin{remark}
We can think of tangent vectors (derivatives) of a path $\gamma: [0,T] \rightarrow G$ as elements of the Lie algebra $\fg$ in two ways. First, for a tangent vector $v \in T_g G$, we can compute the pushforward the tangent vector along the left multiplication map $L_{g^{-1}}v \in T_e G = \fg$. Second, a basis of the Lie algebra provides a global frame for $G$, meaning, it provides a basis for $T_g G$ for all $g$. By considering $v \in T_g G$ in terms of this basis, we may also think of $v$ as an element of $\fg$. 

In summary, the structure of the Lie group allows us to consider tangent vectors at \style{any} point on $G$ using a single vector space: a fact repeatedly used throughout this paper. 
\end{remark}

Given a left-invariant vector field $X \in \fg$, there exists a unique 1-parameter subgroup $\rho_X : \R \rightarrow G$ such that $\rho_X(0) = e$ and $\rho_X'(0) = X(e)$. This is defined by the integral curve of $X$ which passes through the identity at $t=0$. 

\begin{definition}
    The \style{Lie exponential map} of $G$ is defined as
    \begin{align*}
        \exp: \fg \rightarrow G, \quad \exp(X) \coloneqq \rho_X(1),
    \end{align*}
    where $\rho_X$ is the 1-parameter subgroup defined above. 
\end{definition}

This exponential map provides a way to move between a Lie group and its Lie algebra.

\begin{proposition}
    The exponential map $\exp: \fg \rightarrow G$ is smooth and $d(\exp)_0 = \mathrm{id}$. Thus, $\exp$ is a diffeomorphism between an open neighborhood of the origin $0 \in \fg$ and an open neighborhood of the identity $e \in G$.
\end{proposition}

Thus, if elements are near the origin, we can define an inverse map. 

\begin{definition}
    Suppose $U \subset \fg$ is a neighborhood of the origin such that the exponential map is a diffeomorphism. Let $V = \exp(U)$. The \style{logarithm map} on $V$ is defined to be
    \begin{align*}
        \log: V \rightarrow \fg, \quad \log(g) \coloneqq \exp^{-1}(g).
    \end{align*}
\end{definition}

A homomorphism of Lie groups $F: G \rightarrow H$ is a smooth map which is also a group homomorphism, and a homomorphism of Lie algebras $\phi: \fg \rightarrow \fh$ is a linear map that preserves the Lie bracket $F([X,Y]) = [F(X), F(Y)]$ for all $X, Y \in \fg$. A Lie group homomorphism $F: G \rightarrow H$ induces a Lie algebra homomorphism $F_*: \fg \rightarrow \fh$ between the respective Lie algebras by the induced map between the tangent spaces at the identity $F_*: T_e G \rightarrow T_e H$. 


\begin{example}
    The \style{special orthogonal group} $SO(3)$ --- orientation-preserving rotations of $\R^3$ --- will be the running example used throughout this paper. This is a matrix Lie group and can be explicitly described as the space of all $3 \times 3$ orthogonal matrices ($AA^\intercal = A^\intercal A = I$) with determinant $+1$. The Lie algebra of $SO(3)$ is $\fso(3)$, which consists of all $3 \times 3$ skew-symmetric matrices ($B = -B^\intercal$). An explicit basis for $\fso(3)$ is
    \begin{align*}
        e_1 = \begin{pmatrix} 0 & -1 & 0 \\ 1 & 0 & 0 \\ 0 & 0 & 0 \end{pmatrix}, \quad 
        e_2 = \begin{pmatrix} 0 & 0 & 1 \\ 0 & 0 & 0 \\ -1 & 0 & 0 \end{pmatrix}, \quad 
        e_3 = \begin{pmatrix} 0 & 0 & 0 \\ 0 & 0 & -1 \\ 0 & 1 & 0 \end{pmatrix}.
    \end{align*}
    
    We will denote the duals of these basis vectors to be $\omega_i = e_i \in \fg^*$. For all matrix Lie groups, the Lie exponential and logarithm are simply the matrix exponential and logarithm. Suppose $\theta \in \R$. The exponential map in these three basis directions gives us
    \begin{align*}
        \exp(\theta e_1) &= \begin{pmatrix} \cos\theta & -\sin\theta & 0 \\ \sin\theta & \cos\theta & 0 \\ 0 & 0 & 1 \end{pmatrix}, \\
        \exp(\theta e_2) &= \begin{pmatrix} \cos\theta & 0 & \sin\theta \\ 0 & 1 & 0 \\ -\sin\theta & 0 & \cos\theta \end{pmatrix}, \\
        \exp(\theta e_3)  &= \begin{pmatrix} 1 & 0 & 0 \\ 0 & \cos\theta & -\sin\theta \\ 0 & \sin\theta & \cos\theta \end{pmatrix}.
    \end{align*}
    These are exactly the rotation matrices about the z, y, and x axes respectively. Therefore, we may think of the basis vectors $e_i$ of the Lie algebra as infinitesimal rotations in the respective directions. In particular, given a path $\gamma \in P(SO(3))$, the value $\omega_i(\gamma'_t)$ corresponds to the infinitesimal rotation of $\gamma$ at time $t$ in the direction of $e_i$. If we integrate this over the domain of the path,
    \begin{align*}
        \int_0^1 \omega_i(\gamma'_t) dt,
    \end{align*}
    we obtain the cumulative rotation of $\gamma$ in the direction of $e_i$ over the unit interval. This interpretation will be important to keep in mind when we define the path signature in Section~\ref{sec:pathsignature_on_lie_groups}.
\end{example}

Finally, we briefly discuss the Riemannian structure of Lie groups. Recall that a \style{Riemannian metric} on a smooth manifold $M$ is the assignment of an inner product $\langle \cdot , \cdot \rangle_p$ to the tangent space $T_p M$ for every point $p \in M$, which varies smoothly. Specifically, this means that if $X, Y$ are smooth vector fields defined on a neighborhood of $p$, then the map $p \mapsto \langle X_p, Y_p\rangle_p$ is smooth. On a Lie group, we often want a Riemannian metric that is compatible with the algebraic structure of $G$. A Riemannian metric is \style{left-invariant} if
\begin{align*}
    \langle X, Y \rangle_g = \langle L_{h*} X, L_{h*} Y\rangle_{hg}
\end{align*}
for all $g, h \in G$ and $X, Y \in T_g G$, and a right-invariant Riemannian metric is defined similarly. Such left-invariant metrics can simply be defined on the tangent space at the identity. 

\begin{proposition}
    There is a one-to-one correspondence between left-invariant metrics on a Lie group $G$ and inner products on its Lie algebra $\fg$. 
\end{proposition}

Namely, evaluating the inner product $\langle \, , \, \rangle_g$ simply corresponds to viewing the tangent vectors as elements of the identity, and then evaluating the chosen inner product on $\fg$. We will assume that all Riemannian metrics under discussion are left-invariant, and simply call them Riemannian metrics. 


\subsection{Paths on Lie groups}
\label{ssec:paths_liegroups}

A Riemannian metric $\langle \cdot , \cdot \rangle$, where we now omit the subscript $g$ since it is left-invariant,  provides a notion of length for piecewise regular paths on $G$. Suppose $\gamma \in PG$. Then the \style{length} of $\gamma$ is defined to be 
\begin{equation*}
    \ell(\gamma) \coloneqq \int_0^1 \sqrt{\langle \gamma_t', \gamma_t'\rangle} dt.
\end{equation*}

This allows us to define a metric on the Lie group. If $g_1, g_2 \in G$, then the distance between $g$ and $h$ is defined to be the infimum length of paths connecting $g_1$ and $g_2$,
\begin{equation*}
    d(g_1, g_2) \coloneqq \inf \left \{ \ell(\gamma) \, : \, \gamma \in PG, \, \gamma_0 = g_1, \, \gamma_1 = g_2 \right\}.
\end{equation*}
Note that since the Riemannian metric is left invariant, this metric is also left invariant,
\begin{equation*}
    d(h g_1, h g_2) = d(g_1, g_2),
\end{equation*}
for all $h \in G$. The more familiar notion of length in the path signature literature is the 1-variation of a path.

\begin{definition}
    Suppose $(X, d_X)$ is a metric space and let $\gamma \in PX$. The \style{1-variation of $\gamma$} on $[0,1]$ is defined as
    \begin{equation}
        |\gamma|_{1-var} = \sup_{(t_i)} \sum_i d_X(\gamma_{t_i}, \gamma_{t_{i+1}}),
    \end{equation}
    where the sum is taken over all partitions $0 = t_1 \leq \ldots \leq t_n = 1$ of $[0,1]$. 
\end{definition}

Using the metric induced by the Riemannian metric, we may consider the 1-variation length of paths in $G$. Under the piecewise regular hypothesis, these two lengths are equivalent.

\begin{lemma}[\cite{burtscher_length_2015}]
    Let $\gamma \in PG$. We have $\ell(\gamma) = |\gamma|_{1-var}$.
\end{lemma}

At this point, in the case of paths on Euclidean space, we may use the $1$-variation to define a metric on $P\R^N_0$, which are the paths which start at the origin. Given a Lie group $G$ with a left-invariant Riemannian metric, we could follow the same procedure to obtain a metric space structure on $PG_e$. However, this is not the metric space structure on $PG$ that is the most compatible with the path signature. We will defer this discussion until Section~\ref{ssec:stability}.\medskip

The path space $P\R^N$ is endowed with a vector space structure since $\R^N$ itself is a vector space. Similarly, we can endow $PG$ with a group structure by pointwise multiplication, where the identity is the constant path at the identity, and the inverse to a path $\gamma \in PG$ is the pointwise inverse. However, we are missing a notion of scaling for paths in $PG$, and such an operation is important to have in machine learning, since algorithms may require normalization of data. Such a scaling is obtained by proving a correspondence between paths in $G$ and paths in $\fg$, and then transferring the scaling operation from $\fg$ to $G$. \medskip

This is done by considering paths on $G$ from the point of view of differential equations. We have the following existence and uniqueness theorem for first order ordinary differential equations. Let $\overline{P}\fg$ denote the space of piecewise continuous paths $\gamma: [0,1] \rightarrow \fg$ which are right continuous, meaning $\lim_{t\downarrow t_0} \gamma_{t} = \gamma_{t_0}$.

\begin{theorem}
\label{thm:euode}
    Let $f \in \overline{P}\fg$, so that $f:[0,1] \rightarrow \fg$ is piecewise continuous and right continuous, where we consider elements of $\fg$ as left-invariant vector fields. Then, the solution of the first order ODE
    \begin{equation}
    \label{eq:ode}
        \gamma'_t = f_t(\gamma_t), \quad \gamma_0 = g
    \end{equation}
    exists and is unique.
\end{theorem}

Note that in this theorem, we consider a function $\gamma:[0,1] \rightarrow G$ to be a solution of this ODE if the differential equation holds at all points except the points of discontinuity of $f$. This implies that we can represent piecewise regular paths in $G$ as paths in the Lie algebra $\fg$, along with its initial point. Let $PG_g \subset PG$ be defined as
\begin{align*}
    PG_g = \{ \gamma \in PG \, : \, \gamma_0 = g\}.
\end{align*}

\begin{corollary}
\label{cor:lgla_corr}
Suppose $G$ is a Lie group and $\fg$ its Lie algebra. The map $\Psi_g : \overline{P}\fg \rightarrow PG_g$, which takes $f \in \overline{P}\fg$ to the solution of the ODE in Equation~\ref{eq:ode} with initial condition $\gamma_0 = g$, is a bijection.
\end{corollary}
\begin{proof}
    Firstly, the map $\Psi_g$ is well defined by the existence and uniqueness theorem above. The inverse to $\Psi_g$ can be defined by taking the derivative at every differentiable point. Suppose $\gamma \in PG_g$, and let $d(\gamma) \subset [0,1]$ denote the set of points such that $\gamma$ is differentiable. Note that $[0,1] - d(\gamma)$ is a finite set since $\gamma$ is piecewise regular. Now, define $\Psi_g^{-1}(\gamma)(t) = \gamma'_t$ for all $t \in d(\gamma)$, and at the nondifferentiable points by right continuity
    \begin{align*}
        \Psi_g^{-1}(\gamma)(t) = \lim_{s \downarrow t} \gamma_s'.
    \end{align*}
    
    This map is well defined: $\Psi_g^{-1}(\gamma)(t)$ is continuous for every $t \in [0,1] - d(\gamma)$, and right continuous by definition. 
\end{proof}

We can view $\overline{P}\fg$ as a Lie algebra, with pointwise vector space operations, and pointwise Lie bracket. Because the group structure of $PG$ and the Lie algebra structure of $\overline{P}\fg$ are defined pointwise, the map $\Psi_g$ is compatible with Lie algebra morphisms induced by Lie group morphisms. Namely, if $F: G \rightarrow H$ is a Lie group morphism, we obtain a group homomorphism $F: PG \rightarrow PH$ by applying the map pointwise. Analogously, if $F_*: \fg \rightarrow \fh$ is the induced Lie algebra morphism, we obtain a Lie algebra morphism $F_*: \overline{P}\fg \rightarrow \overline{P} \fh$. The following lemma is immediate since the group structure on $PG$ and the Lie algebra structure on $\overline{P}\fg$ are defined pointwise.

\begin{lemma}
    Suppose $F: G \rightarrow H$ is a morphism of Lie groups, and $F_*: \fg \rightarrow \fh$ is the induced morphism of Lie algebras. Then the following diagram commutes
    \[
    \begin{tikzcd}
        \overline{P}\fg \ar[r, "F_*"] & \overline{P}\fh \ar[d,"\Psi_{F(g)}"] \\
        PG_g \ar[u,"\Psi_g^{-1}"] \ar[r,"F"] & PH_{F(g)}.
    \end{tikzcd}
    \]
\end{lemma}

The map $\Psi_g$ allows us to view paths on Lie groups as paths in a linear space, while retaining all first order differential information. We can use the fact that many operations for paths on $\R^N$ are defined via operations on the Lie algebra, and thus generalize these operations to Lie groups. \medskip

For a path $\alpha \in P\R^N$ and $\lambda \geq 0$, denote the \style{vector space scaling} operation as
\begin{equation*}
    (\lambda \alpha)_t \coloneqq \lambda \alpha_t.
\end{equation*}
However, another way of viewing the scaling operation for paths that begin at the origin is by scaling in the Lie algebra. Suppose $\lambda \geq 0$, and denote the vector space scaling in a Lie algebra $\fg $ by $c_\lambda: \fg \rightarrow \fg$.
\begin{lemma}
    Let $\alpha \in P\R^N_0$. Then
    \begin{equation*}
        \lambda \alpha = \Psi_{0} \circ c_\lambda \circ \Psi_{0}^{-1} (\alpha).
    \end{equation*}
\end{lemma}
\begin{proof}
    In $\R^N$, the map $\Psi_0$ is simply integration in $\R^N$, and $\Psi_0^{-1}$ is differentiation. Thus, we have
    \begin{align*}
        \left(\Psi_0 \circ c_\lambda \circ \Psi_0^{-1}(\alpha)\right)_t & = \int_0^t \lambda \alpha'_s ds \\
        & = \lambda \int_0^t \alpha'_s ds\\
        & = (\lambda \alpha)_t.
    \end{align*}
\end{proof}

We use this fact as motivation to define scaling on Lie groups.

\begin{definition}
\label{def:scaling}
    Suppose $G$ is a Lie group and $\fg$ its Lie algebra. Let $\gamma \in PG$ and $\lambda \geq 0$. We define the \style{Lie algebra scaling} of $\gamma$ by $\lambda$ to be
    \begin{equation}
        \lambda \cdot \gamma \coloneqq \Psi_{\gamma_0} \circ c_\lambda \circ \Psi_{\gamma_0}^{-1}(\gamma).
    \end{equation}
\end{definition}

\begin{remark}
    We highlight three important differences between vector space scaling for paths in $\R^N$ and Lie algebra scaling for paths in an arbitrary Lie group $G$, and provide a reason for each.
    
    \begin{enumerate}
        \item Returning to the setting of paths in $\R^N$, the two notions of scaling differ slightly when the path does not start at the origin. If we have $\alpha \in P\R^N$ such that $\alpha_0 = x$, then $(\lambda \alpha)_0 = \lambda x$, while $(\lambda \cdot \alpha)_0 = x$. However, if we align the initial points, the paths coincide,
        \begin{align*}
            (\lambda \alpha) - \lambda x = (\lambda \cdot \alpha) - x.
        \end{align*}
        This difference is due to the fact that arbitrary Lie groups do not have a natural scaling operation. However, if our Lie group was equipped with a suitable scaling operation, such as a Carnot group~\citep{le_donne_primer_2017}, then we would be able to do define a scaling operation that coincides with the vector space scaling in $P\R^N$. 
        
        \item We have only defined scaling by a nonnegative number. Definition~\ref{def:scaling} could be extended to all real numbers $\lambda$ without any changes, but the interpretation of negative scaling is more difficult in arbitrary Lie groups. For a path $\alpha \in P\R^N_0$, scaling by $\lambda = -1$ simply produces the pointwise inverse of a path. However, this is not the case in a general Lie group. For example, let $X, Y \in \fg$ and consider the piecewise path
        \begin{align*}
            \gamma_t = \left\{
            \begin{array}{cl}
                e^{2tX}& : t\in [0,\frac12) \\
                e^{X}e^{(2t-1)Y} & : t \in [\frac12,1].
            \end{array}
            \right.
        \end{align*}
        Here, we have $\gamma_1 = e^X e^Y$ and $(-1 \cdot \gamma)_1 = e^{-X} e^{-Y}$, which are not inverses in general. Thus, we see that the obstruction to this interpretation is the noncommutativity of arbitrary Lie groups. However, in the setting of abelian Lie groups, such an interpretation would hold.
        
        \item By definition, the vector space scaling in $P\R^N$ obeys the distributive law: $\lambda(\alpha + \beta) = (\lambda \alpha) + (\lambda \beta)$ for $\alpha, \beta \in P\R^N$ and $\lambda \in \R$. In other words, the vector space scaling is a pointwise Lie group homomorphism for $\R^N$. However, $c_\lambda : \fg \rightarrow \fg$ is not a morphism of Lie algebras in general since $c_\lambda([X,Y]) = \lambda [X, Y] \neq \lambda^2 [X,Y] = [c_\lambda X, c_\lambda Y]$. Thus, it cannot be the induced map of an underlying Lie group homomorphism for $G$, so the Lie algebra scaling for $G$ is not distributive, $\lambda \cdot(\alpha \beta) \neq (\lambda \cdot \alpha) (\lambda \cdot \beta)$, in general. In the case of an abelian Lie group $H$, the associated Lie algebra $\fh$ is abelian so that $[X,Y] = 0$ for all $X, Y \in \fh$, and thus Lie algebra scaling can be viewed as a pointwise Lie group morphism.
    \end{enumerate}
\end{remark}

Due to these remarks, we must keep in mind that the scaling operation for paths in Lie groups is not compatible with the algebraic structure of $G$.


\subsection{Discrete time series on $G$}
\label{ssec:discretets}
In this subsection, we will consider the interpretation of discrete time series on an arbitrary Lie group $G$, and also discuss derivative computations for these discrete time series. We will continue the theme of comparison with the corresponding notions in $\R^N$. \medskip

\begin{remark}
Here, we will assume that discrete time series are uniformly sampled at integer times. This does not result in any loss of generality due to the reparametrization invariance of the path signature, given in Proposition~\ref{prop:reparam}.
\end{remark}

Let $T \in \N$ and $\hat{x}: [T+1] \rightarrow \R^N$ be a discrete time series in $\R^N$ of length $T+1$. There is a natural interpretation of $\hat{x}$ as a continuous time series $x: [T+1] \rightarrow \R^N$ by linear interpolation between points. Namely, it is the interpolation with a constant derivative between the discrete points defined in $\hat{x}$. This is the interpretation that we implicitly take when we compute derivatives of discrete time series by finite differences $\hat{x}'_i = \hat{x}_{i+1}- \hat{x}_i$ to get the discrete derivative $\hat{x}':[T] \rightarrow \R^N$. Additionally, we can think about the continuous path $x$ as a geodesic interpolation of the discrete path $\hat{x}$. \medskip

However, the interpretation is more subtle in the case of arbitrary Lie groups. Suppose we have a discrete time series in $G$, which we denote by $\hat{\gamma}:[T+1] \rightarrow G$. We wish to obtain an interpolation such that the derivative, when viewed in the Lie algebra $\fg$, is constant between adjacent points. This can be achieved by taking the logarithm of the difference between adjacent points. We define the \style{discrete derivative} $\hat{\gamma}':[T] \rightarrow \fg$ of a discrete Lie group valued path by
\begin{equation}
\label{eq:disc_derivative}
    \hat{\gamma}'_i \coloneqq \log\left(\hat{\gamma}_i^{-1} \hat{\gamma}_{i+1}\right) \in \fg.
\end{equation}
Then, we can define the continuous interpolation $\gamma: [0, T+1] \rightarrow G$ using the exponential map such that for $t \in [i, i+1)$, the interpolation is
\begin{align*}
    \gamma_t \coloneqq \hat{\gamma}_t \exp\left( (t-i)\hat{\gamma}'_t \right).
\end{align*}

We note that this construction reduces to linear interpolation in the case of $G = \R^n$. This is due to the fact that for the additive Lie group $\R^N$, the exponential and logarithm map are both the identity and are both globally defined. Additionally, the group operation is addition, so we should interpret all of the products as sums. However, there are two essential differences between the case of arbitrary Lie groups and Euclidean space. \medskip

Firstly, for an arbitrary Lie group $G$, the logarithm map is only defined in a neighborhood of the identity. The two reasons the logarithm may not be defined in a larger neighborhood are the loss of injectivity and the loss of surjectivity of the exponential map. On any compact Lie group, the exponential map will not be injective at any point. In this case, we can define the logarithm to be the value closest to the origin, but non-injectivity may still occur. For example, the point antipodal to the identity in $S^1$ has no unique logarithm since there are two paths of equal distance to the identity. However, if we perturb the target point in either direction, there exists a unique shortest path. This implies that by undersampling the underlying time series, we may infer incorrect information. The case of $S^1$ is exactly the situation encountered in the Nyquist sampling theorem.

The exponential map is not always surjective, with the simplest examples being non-connected Lie groups. However, connected Lie groups such as $SL(2, \R)$ can still have non-surjective exponential maps. In these cases, discrete derivatives may not exist, and finer sampling is required so that the difference between adjacent points $\tilde{\gamma}_i^{-1} \tilde{\gamma}_{i+1}$ is closer to the identity and has a well-defined logarithm. However, for compact Lie groups such as $SO(3)$, the Lie exponential map is surjective. 

Secondly, the interpolation defined here may not be a geodesic connecting the two points. Suppose $h$ is a Riemannian metric on $G$. In general, geodesics do not coincide with the one-parameter subgroups of $G$. In other words, in these cases, the Riemannian exponential map and the Lie exponential map are \style{not} the same. However, for bi-invariant metrics, they coincide.

\begin{theorem}
    The Lie exponential map and the Riemannian exponential map at the identity agree on Lie groups with bi-invariant metrics.
\end{theorem}

Thus, for all Lie groups equipped with bi-invariant metrics, we may continue to interpret the interpolation as a geodesic interpolation. In fact, this holds for all compact Lie groups.

\begin{proposition}
    Every compact Lie group admits a bi-invariant metric. 
\end{proposition}

From this discussion, we find that for a compact Lie group $G$, the interpretation of discrete time series on $G$ is similar to the case of $\R^N$, with the main difference being the non-injectivity of the exponential map. 

    
        
    


\section{Path signatures on Lie groups}
\label{sec:pathsignature_on_lie_groups}

This subsection, based on the exposition of path signatures on Euclidean space given in~\cite{giusti_iterated_2020}, begins by defining the path signature for Lie groups. We show several basic properties which are well-known for path signatures on Euclidean space, culminating in the definition of tree-like equivalence for paths and the property that the signature is an injective group homomorphism. This material was originally developed by~\cite{chen_iterated_1954, chen_integration_1957, chen_integration_1958} and is not novel. 

We then prove a signature preserving bijection between paths on an $N$-dimensional Lie group $G$ and paths on $\R^N$, which highlights the extent to which the theory naturally extends to the case of Lie groups. This result provides a Euclidean representation of Lie group valued time series, and can thus be used to apply classical Euclidean data analysis techniques to Lie group valued time series.

We then consider the extension of the equivariance property of path signatures. This is followed by an interpretation of the second-level signature terms as indicators of lead-lag behavior between the directions corresponding to our choice of basis vectors for the Lie algebra $\fg$. Finally, we close this section by discussing computational aspects of the path signature for discrete time series, as well as symmetry breaking path transformations which can be used as a preprocessing step. \medskip

In this section, we use $(e_1, \ldots, e_N)$ to denote an ordered basis of $\fg$ and use $(\omega_1, \ldots, \omega_N)$ to denote the dual basis of $\fg^*$ such that $ \omega_i(e_j) = \delta_{i,j}$, where $\delta_{i,j}$ is the Kronecker delta. 


\subsection{Path signature as a group homomorphism}
\label{ssec:pathsignature}
Let $G$ be an $N$-dimensional Lie group. Recall that $PG$ denotes the space of piecewise regular paths $\gamma: [0,1] \rightarrow G$. 

\begin{definition}
\label{def:ps}
    Let $\gamma \in PG$. Suppose $\omega_1, \ldots, \omega_N \in \fg^*$ form a basis of $\fg^*$.  For $i \in [N]$, define a path $S^i(\gamma)_t:[0,1] \rightarrow \R$ as
    \begin{align*}
        S^i(\gamma)_t \coloneqq  \int_0^t \omega_i (\gamma'_s) ds.
    \end{align*}
    Next, let $I = (i_1, \ldots, i_m)$ be a multi-index, where $i_j \in [N]$. Higher order paths $S^I(\gamma)_t: [0,1] \rightarrow \R$ are inductively defined as
    \begin{equation}
        S^I(\gamma)_t \coloneqq \int_0^t S^{(i_1, \ldots, i_{m-1})}(\gamma)_s \omega_{i_m}(\gamma'_s) ds.
    \end{equation}
    The \style{path signature of $\gamma$ with respect to $I$} is defined to be $S^I(\gamma) \coloneqq S^I(\gamma)_1$. 
\end{definition}

We can also present the definition in a non-inductive way. Let $\Delta^m$ be the standard $m$-simplex
\begin{align*}
    \Delta^m = \{ (t_1, \ldots, t_m) \, : \, 0 \leq t_1 < t_2 < \ldots < t_m \leq 1\}.
\end{align*}
By collapsing the inductive definition, we can write the path signature of $\gamma$ with respect to $I = (i_1, \ldots, i_m)$ as
\begin{equation}
\label{eq:pathsignatureI}
    S^I(\gamma) = \int_{\Delta^m} \omega_{i_1}(\gamma'_{t_1}) \ldots \omega_{i_m}(\gamma'_{t_m}) \, dt_1 \ldots dt_m.
\end{equation}

We can amalgamate the path signatures with respect to every multi-index $I$ into an element of a tensor algebra.
\begin{definition}
    Suppose $V$ is a real vector space of dimension $N$. The \style{tensor algebra} with respect to $V$ is defined to be
    \begin{align*}
        T((V)) = \prod_{m\geq 0} V^{\otimes m}.
    \end{align*}
    Suppose $(e_1, \ldots, e_N)$ is an ordered basis for $V$. Suppose $\mathbf{s}, \mathbf{t} \in T((V))$. Let $\mathbf{t}_m \in V^{\otimes m}$ be the degree $m$ part of $\mathbf{t}$ and if $I = (i_1, \ldots, i_m)$ is a multi-index with $i_j \in [N]$, then $\mathbf{t}^I$ is the coefficient of $e_{i_1} \otimes \ldots \otimes e_{i_m}$ in $\mathbf{t}$. Addition and scalar multiplication is defined element-wise:
    \begin{itemize}
        \item $(\mathbf{s} + \mathbf{t})^I = \mathbf{s}^I + \mathbf{t}^I$,
        \item $(\lambda \mathbf{t})^I = \lambda \mathbf{t}^I$,
    \end{itemize}
    and multiplication is defined by tensor multiplication
    \begin{itemize}
        \item $(\mathbf{s} \otimes \mathbf{t})^I = \sum_{j=0}^m \mathbf{s}^{(i_1, \ldots, i_j)} \mathbf{t}^{(i_{j+1}, \ldots, i_m)}$.
    \end{itemize}
\end{definition}

Let $\bar{\fg}$ be the underlying vector space of the Lie algebra $\fg$. Let $e_1, \ldots, e_N$ be a basis for $\fg$. We define the \style{path signature} of $\Gamma \in PG$ to be
\begin{equation}
\label{eq:pathsignature}
    S(\gamma) \coloneqq 1 + \sum_{m\geq 1} \sum_{|I| = m} S^I(\gamma) e_{i_1} \otimes \ldots \otimes e_{i_m} \in T((\bar{\fg})).
\end{equation}

\begin{remark}
    For path signatures defined on Euclidean space $\R^N$, we often choose the standard 1-forms $(dx_1, \ldots, dx_N)$ to be the basis of $\fr$, the Lie algebra of $\R^N$. Suppose $\alpha \in P\R^N$. We can also write our path component-wise as $\alpha = (\alpha^1, \ldots, \alpha^N)$, where each $\alpha^i : [0,1] \rightarrow \R$. Then, evaluation of the standard 1-forms is simply $dx_i(\alpha'_t) = (\alpha^i)'_t$. Thus, in the Euclidean case, the definition of the path signature reduces to
    \begin{equation}
        S^I(\alpha) = \int_{\Delta^m} (\alpha^{i_1})'_{t_1} \ldots (\alpha^{i_m})'_{t_m} \, dt_1 \ldots dt_m.
    \end{equation}
\end{remark}

Let $\gamma \in PG$ and $g \in G$. The \style{left translation} of $\gamma$ by $g$ is defined to be the path $(g\gamma)_t \coloneqq g (\gamma_t)$, where we left translate the path $\gamma$ by $g$ pointwise (one can analogously define the \style{right translation} of a path). Similar to the case of Euclidean space, path signatures are left translation invariant and reparametrization invariant.

\begin{proposition}[Left translation invariance]
    Let $\gamma \in PG$ and $g \in G$. Then $S(g\gamma) = S(\gamma)$.
\end{proposition}
\begin{proof}
    It suffices to show that $S^I(g\gamma) = S^I(\gamma)$ for all multi-indices $I$. Note that we have
    \begin{align*}
        (g\gamma)'_t = L_{g*} \gamma'_t .
    \end{align*}
    Specifically, this implies that $\gamma'_t$ and $g\gamma'_t$ are represented by the same element in the Lie algebra $\fg$. Therefore for any $\omega \in \fg^*$, we have $\omega(g\gamma'_t) = \omega(\gamma'_t)$. Thus, $S^I(g\gamma) = S^I(\gamma)$ for all $I$. 
\end{proof}

\begin{proposition}[Reparametrization invariance]
\label{prop:reparam}
 Let $\gamma : [a,b] \rightarrow G$ be a piecewise regular path, and let $\phi: [c,d] \rightarrow [a,b]$ be a strictly increasing function. Then $S(\gamma \circ \phi) = S(\gamma)$.
\end{proposition}
\begin{proof}
    This is the Change of Variables Theorem. Reparametrization invariance of the first level of the signature is given as
    \[
        S^i(\gamma \circ \phi) = \int_c^d \omega_i((\gamma \circ \phi)'_t) dt 
        = \int_c^d \omega_i (\gamma'_{\phi_t}) \phi'_t dt 
        = \int_a^b \omega_i (\gamma'_\tau) d\tau 
        = S^i(\gamma).
    \]
    Invariance for higher order terms is shown by induction using the same argument. 
\end{proof}
In particular this proposition justifies our choice of only considering paths parametrized by $[0,1]$, as any other path can be reparametrized into this domain. Next, we would like to understand how scaling of paths in $G$ given in Definition~\ref{def:scaling} affects the path signature. Note that the vector space scaling in $\bar{\fg}$ induces a dilation map in $T((\bar{\fg}))$. Explicitly, we define the map $\delta_\lambda : T((\bar{\fg})) \rightarrow T((\bar{\fg}))$ as
\begin{equation}
\label{eq:tensordilation}
    \delta_\lambda\mathbf{t} \coloneqq (\mathbf{t}_0, \lambda \mathbf{t}_1, \lambda^2 \mathbf{t}_2, \ldots).
\end{equation}

\begin{proposition}
\label{prop:ps_scaling}
Let $\gamma \in PG$ and $\lambda \geq 0$. Then $S(\lambda \cdot \gamma) = \delta_\lambda S(\gamma)$.
\end{proposition}

\begin{proof}
    Consider the multi-index $I = (i_1, \ldots, i_k)$. Then,
    \begin{align*}
        S^I(\lambda \cdot \gamma) &= \int_{\Delta^k} \omega_{i_1}(\lambda \gamma'_{t_1}) \ldots \omega_{i_k}(\lambda \gamma'_{t_k}) \, dt_1 \ldots dt_k \\
        & = \lambda^k \int_{\Delta^k} \omega_{i_1}(\gamma'_{t_1}) \ldots \omega_{i_k}(\gamma'_{t_k}) \, dt_1 \ldots dt_k \\
        & = \lambda^k S^I(\gamma).
    \end{align*}
\end{proof}
We have seen that the group structure on $G$ allows us to define a group structure on $PG$ by pointwise multiplication. The group structure on $G$ allows us to obtain another group structure on a quotient of $PG$ where the group operation is given by concatenation. Let $\alpha, \beta \in PG$. The \style{concatenation} of $\alpha$ and $\beta$ is defined to be 
\begin{equation*}
    (\alpha * \beta)_t = \left\{
        \begin{array}{cl}
            \alpha_{2t} & : t\in [0,\frac12) \\
            \alpha_1(\beta_0)^{-1} \beta_{2t-1} & : t \in [\frac12,1].
        \end{array}
    \right.
\end{equation*}
The inverse of a path $\gamma$ is defined to be the same path, but in the reverse direction
\begin{equation*}
    (\gamma^{-1})_t = \gamma_{1-t}.
\end{equation*}

Concatenation or inversion of piecewise regular paths is still piecewise regular. In order to obtain an identity element, we must quotient out by an equivalence relation.

\begin{definition}
    A path $\gamma \in PG$ is called \style{reducible} if there exist paths $\alpha, \beta, \zeta \in PG$ such that $\gamma = \alpha*\zeta * \zeta^{-1} * \beta$, up to reparametrization. The path $\alpha * \beta$ is called a \style{reduction} of $\gamma$. We define the reduction of $\zeta * \zeta^{-1}$ to be $c_{e}$, the constant path at the identity $e \in G$.  A path $\gamma$ is \style{irreducible} if no reduction exists. An irreducible path $\tilde{\gamma}$ obtained by finitely many iterative reductions of a path $\gamma$ is called an \style{irreducible reduction} of $\gamma$. 
\end{definition}

\begin{figure}[!htbp]
\centering
	\includegraphics[width=0.6\textwidth]{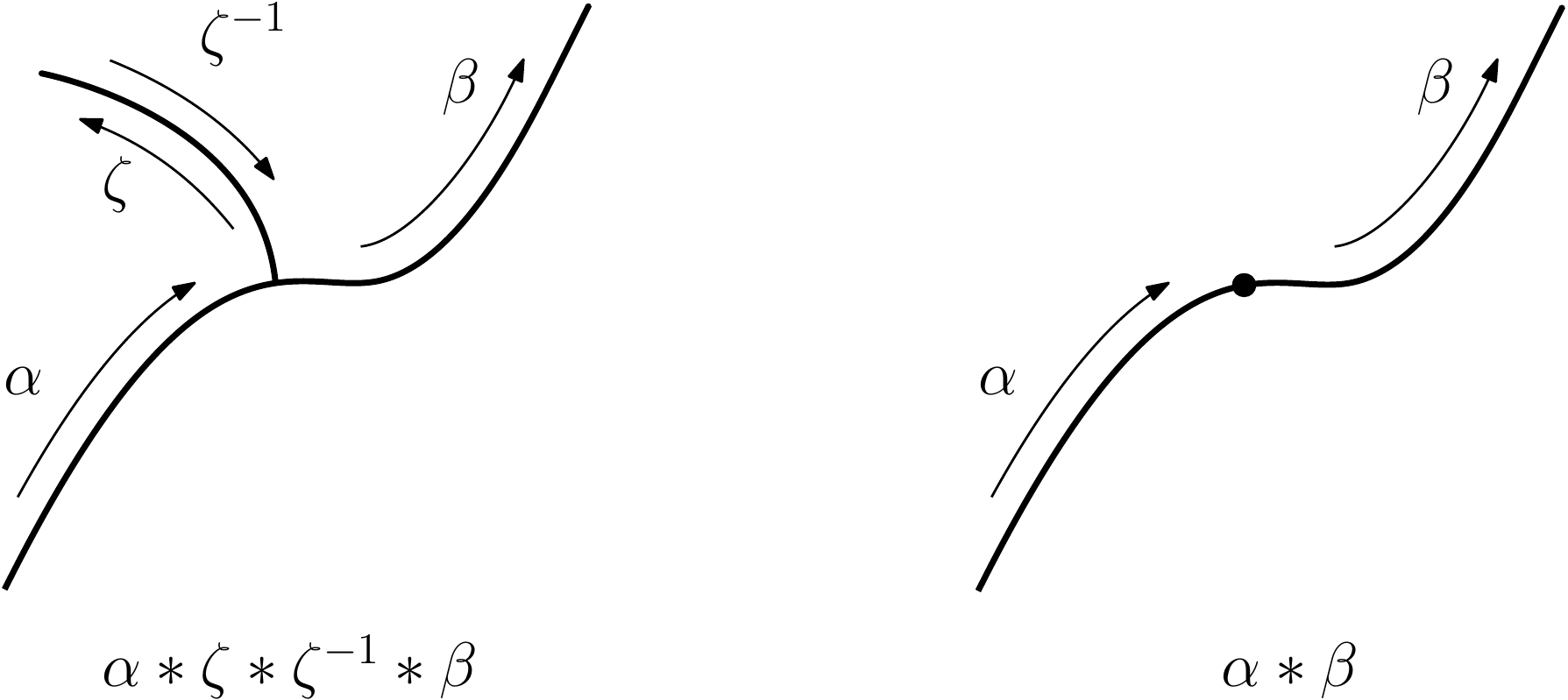}
	\caption{(Left) An example of a reducible path $\alpha * \zeta * \zeta^{-1} * \beta$. (Right) The irreducible reduction of the path on the left.}
\end{figure}

\begin{theorem}[\cite{chen_integration_1958}]
    Every piecewise regular path $\gamma \in PG$ has a unique irreducible reduction up to reparametrization.
\end{theorem}

This result allows us to define the notion of tree-like equivalence.

\begin{definition}
    A path $\gamma \in PG$ is a \style{tree-like path} if its irreducible reduction is $c_e$, the constant path at the identity. Two paths $\alpha, \beta$ are \style{tree-like equivalent}, $\alpha \sim_t \beta$, if $\alpha * \beta^{-1}$ is a tree-like path. 
\end{definition}

\begin{remark}
    The definition of tree-like equivalence includes translations. Indeed, suppose $\gamma \in PG$ and $g \in G$. Define $g\gamma$ and $\gamma g$ to be the left and right translations of the path $\gamma$ by $g$. Then $\gamma \sim_t g \gamma$ since $ \gamma *(g \gamma)^{-1} = \gamma * \gamma^{-1}$ by the definition of the concatenation operator. The same holds for right translations.
    
    Additionally, tree-like equivalence also includes reparametrization since the definition of reductions are reparametrization invariant. 
\end{remark}

\begin{proposition}
    Tree-like equivalence is an equivalence relation.
\end{proposition}
\begin{proof}
    Let $\gamma, \gamma_1, \gamma_2, \gamma_3 \in PG$. Note that the subscript here denotes distinct paths, and does not denote the time parameter. By definition the reduction of $\gamma * \gamma^{-1}$ is the constant path, so $\gamma \sim_t \gamma$.
    
    Next, if $\gamma = \alpha * \zeta * \zeta^{-1} * \beta$, for paths $\alpha, \beta, \zeta \in PG$, then $\gamma^{-1} = \beta^{-1} * \zeta * \zeta^{-1} * \alpha^{-1}$. Thus, a path is reducible if and only if its inverse is reducible. Additionally, the reduction $\beta^{-1} * \alpha^{-1}$ of $\gamma^{-1}$ is the inverse of the reduction $\alpha * \beta$ of $\gamma$. Now, suppose $\gamma_1 \sim_t \gamma_2$ so that $\gamma_1 * \gamma_2^{-1}$ is tree-like. By the above argument, $\gamma_2 * \gamma_1^{-1}$ is also tree-like, so $\gamma_2 \sim_t \gamma_1$.
    
    Finally, the concatentation $\alpha * \beta$ of two tree-like paths is also tree-like, by performing all the reductions of $\alpha$ and then performing all the reductions on $\beta$. Suppose $\gamma_1 \sim_t \gamma_2$ and $\gamma_2 \sim_t \gamma_3$. Then, $\gamma_1 *\gamma_3^{-1}$ is a reduction of $( \gamma_1 * \gamma_2^{-1}) * (\gamma_2 * \gamma_3^{-1})$, and the latter path is tree-like since it is a concatenation of two tree-like paths. By the uniqueness of irreducible reductions, $\gamma_1 * \gamma_3^{-1}$ is tree-like. Thus, $\gamma_1 \sim_t \gamma_3$. 
\end{proof}


We can now define $\widetilde{PG} \coloneqq PG/\sim_t$ to be the space of tree-like equivalence classes of piecewise regular paths in $G$. We define the identity element to be $[c_e] \in\widetilde{PG}$, the equivalence class of the constant path at the identity. Compatibility of concatenation and inversion are implicit in the above proof, and the group axioms are easily checked. Thus, we have shown the following.

\begin{proposition}
    The quotient $\widetilde{PG}$ is a group. 
\end{proposition}

We can now state Chen's injectivity theorem.

\begin{theorem}[\cite{chen_integration_1958}]
\label{thm:injectivity}
Suppose $G$ is a real Lie group. Let $\alpha, \beta \in PG$. Then $S(\alpha) = S(\beta)$ if and only if $\alpha$ and $\beta$ are tree-like equivalent.
\end{theorem}

Chen also showed that the signature is a group homomorphism. Namely, suppose $\alpha, \beta \in PG$. Chen's identity~\citep{chen_iterated_1954} states that
\begin{equation}
\label{eq:chens_identity}
    S(\alpha * \beta) = S(\alpha) \otimes S(\beta).
\end{equation}
Putting the previous results together, we obtain the following characterization.

\begin{proposition}
    The path signature map $S: \widetilde{PG} \rightarrow T((\overline{\fg}))$ is an injective group homomorphism.
\end{proposition}

We will also require an internal multiplicative structure on the path signature coefficients which is an immediate generalization of the Euclidean path signature. 
\begin{definition}
    Let $k$ and $l$ be non-negative integers. A \style{$(k,l)$-shuffle} is a permutation of $\sigma$ of the set $\{1, 2, \ldots, k+l\}$ such that
    \begin{align*}
        \sigma^{-1}(1) < \sigma^{-1}(2) < \ldots < \sigma^{-1}(k)
    \end{align*}
    and
    \begin{align*}
        \sigma^{-1}(k+1) < \sigma^{-1}(k+2) < \ldots < \sigma^{-1}(k+l).
    \end{align*}
    We denote by $Sh(k,l)$ the set of $(k,l)$-shuffles.
Given two finite ordered multi-indices $I = (i_1, \ldots, i_k)$ and $J = (j_1, \ldots, j_l)$ , let $R = (r_1, \ldots, r_k, r_{k+1}, \ldots r_{k+1}) = (i_1, \ldots, i_k, j_1, \ldots, j_l)$ be the concatenated multi-index. The \style{shuffle product} of $I$ and $J$ is defined to be the multiset
\begin{align*}
    I \shuffle J = \left\{ \left( r_{\sigma(1)}, \ldots r_{\sigma(k+l)}\right) \, : \, \sigma \in Sh(k,l)\right\}.
\end{align*}
\end{definition}

As an example, suppose $I = (1, 2)$ and $J = (2,3)$. Then
\begin{align*}
    I \shuffle J = \left\{ (1,2,2,3), (1,2,2,3), (2,1,2,3), (1,2,3,2), (2,1,3,2), (2,3,1,2) \right\}.
\end{align*}

\begin{theorem}
\label{thm:shuffle}
    Let $I$ and $J$ be multi-indices in $[N]$, of lengths $k$ and $l$ respectively, and suppose $\gamma \in PG$. Then
    \begin{equation}
        S^I(\gamma)S^J(\gamma) = \sum_{K \in I \shuffle J} S^K(\gamma).
    \end{equation}
\end{theorem}

\begin{proof}
    Let $R = (r_1, \ldots, r_k, r_{k+1}, \ldots r_{k+l}) = (i_1, \ldots, i_k, j_1, \ldots, j_l)$. Writing out the signature on the left side of the equation using Equation~\ref{eq:pathsignatureI}, we get 
    \begin{align*}
        \int_{\Delta^k} \omega_{i_1}(\gamma'_{t_1}) \ldots &\omega_{i_k}(\gamma'_{t_k}) dt_1 \ldots dt_k \int_{\Delta^l} \omega_{j_1}(\gamma'_{t_1}) \ldots \omega_{j_l}(\gamma'_{t_l})dt_1 \ldots dt_l \\
        &= \int_{\Delta^k \times \Delta^l}\omega_{r_1}(\gamma'_{t_1}) \ldots \omega_{r_{k+l}}(\gamma'_{t_{k+l}}) \, dt_1 \ldots dt_{k+l},
    \end{align*}
    and the sum on the right side is
    \begin{align*}
        \sum_{\sigma \in Sh(k,l)} \int_{\Delta^{k+l}} \omega_{\sigma(r_1)}(\gamma'_{t_1}) \ldots \omega_{\sigma(r_{k+l})}(\gamma'_{t_{k+l}}) dt_1 \ldots dt_{k+l}.
    \end{align*}

    The equivalence of the two formulas is given by the standard decomposition of $\Delta^k \times \Delta^l$ into $(k+l)$-simplices,
    \begin{align*}
        \Delta^k \times \Delta^l &= \left \{ (t_1, \ldots, t_{k+l}) \, : \, 0 < t_1 < \ldots < t_k < 1, \, 0 < t_{k+1} < \ldots < t_{k+l} < 1\right\}\\
        & = \bigsqcup_{\sigma \in Sh(k,l)} \left\{ (t_{\sigma(1)}, \ldots, t_{\sigma(k+l)}) \, : \, 0 < t_1 < \ldots < t_{k+l} < 1\right\}.
    \end{align*}
\end{proof}


\subsection{Relationship between paths in $G$ and $\R^N$}
\label{ssec:relationship_g_rn}
In this subsection, we define a signature-preserving bijection between piecewise regular paths in $G$ and piecewise regular paths in $\R^N$ which start at the identity and origin respectively. 

The idea behind the following proposition is that the path signature computation only requires the first derivative of paths. The Lie bracket is unused in the computation of path signatures, so we can simply consider the Lie algebras of Lie group as vector spaces. Thus, we can identify the underlying vector space of the Lie algebra $\fg$ with the underlying vector space of the Lie algebra $\fr$ of $\R^N$. We then use the correspondence $\Psi_g: \overline{P}\fg \rightarrow PG_g$ between piecewise continuous paths $\overline{P}\fg$ and piecewise regular paths $PG$ given in Corollary~\ref{cor:lgla_corr}, to map paths on $G$ to paths on $\R^N$.

In the following proposition, we abuse notation and consider elements of the Lie algebras $\fr$ of $\R^N$ and $\fg$ of $G$ as both the tangent space at the identity, and the vector space of left-invariant vector fields. Similarly, we consider elements of the dual of the Lie algebra $\fr^*$ and $\fg^*$ as both the cotangent space at the identity, and the vector space of left-invariant $1$-forms.

Because we will be using two different path signature functions, we will denote by $S_\R: P\R^N \rightarrow T((\R^N))$ the path signature for $\R^N$ with respect to the ordered basis of standard $1$-forms $(dx_1, \ldots, dx_N)$ of $\fr^*$. We denote $S_G: PG \rightarrow T((\R^N))$ to be the path signature for $G$ with respect to a given ordered basis $(\omega_1, \ldots, \omega_N)$ of $\fg^*$.

\begin{proposition}
\label{prop:relationship}
    Suppose $G$ is an $N$-dimensional Lie group with Lie algebra $\fg$. Let $\phi: \bar{\fr} \rightarrow \bar{\fg}$ be an isomorphism of vector spaces, and $\phi^*: \fg^* \rightarrow \fr^*$ be its dual. Let $(dx_1, \ldots, dx_N)$ denote the standard 1-forms of $\R^N$, and define $\omega_i = (\phi^*)^{-1}(dx_i)$. Let $S_G: PG \rightarrow T((\R^N))$ be the path signature map for $G$ with respect to the ordered basis $(\omega_1, \ldots, \omega_N)$ of $\fg^*$. Then, there exists a bijection $\Phi: P\R^N_0 \rightarrow PG_e$ such that $S_\R(\gamma) = S_G(\Phi(\gamma))$ for all $\gamma \in P\R^N$. 
\end{proposition}

\begin{proof}
    The construction of the map $\Phi$ is derived from Corollary~\ref{cor:lgla_corr}. Let $\Psi_\R: \overline{P}\fr \rightarrow P\R^N_0$ and $\Psi_G: \overline{P}\fg \rightarrow PG_e $ be the bijections from Corollary~\ref{cor:lgla_corr} for $\R^N$ and $G$ respectively. Now, define $\Phi$ by
    \begin{align*}
        \Phi: P\R^N_0 \xrightarrow{\Psi_\R^{-1}} \overline{P}\fr \xrightarrow{\phi} \overline{P}\fg \xrightarrow{\Psi_G} PG_e.
    \end{align*}
    The idea is that we start with a path $\gamma \in P\R^N_0$, and apply the following maps:
    \begin{enumerate}
        \item $\Psi_\R^{-1}$ : take the derivative $\gamma'$ to obtain a path in $\fr$
        \item $\phi$ : identify the underlying vector space of $\fr$ with $\fg$
        \item $\Psi_G$ : solve the differential equation (Equation~ \ref{eq:ode}) with the identity initial condition to obtain a path in $G$.
    \end{enumerate}
    
    \begin{figure}[!htbp]
    \centering
	\includegraphics[width=\textwidth]{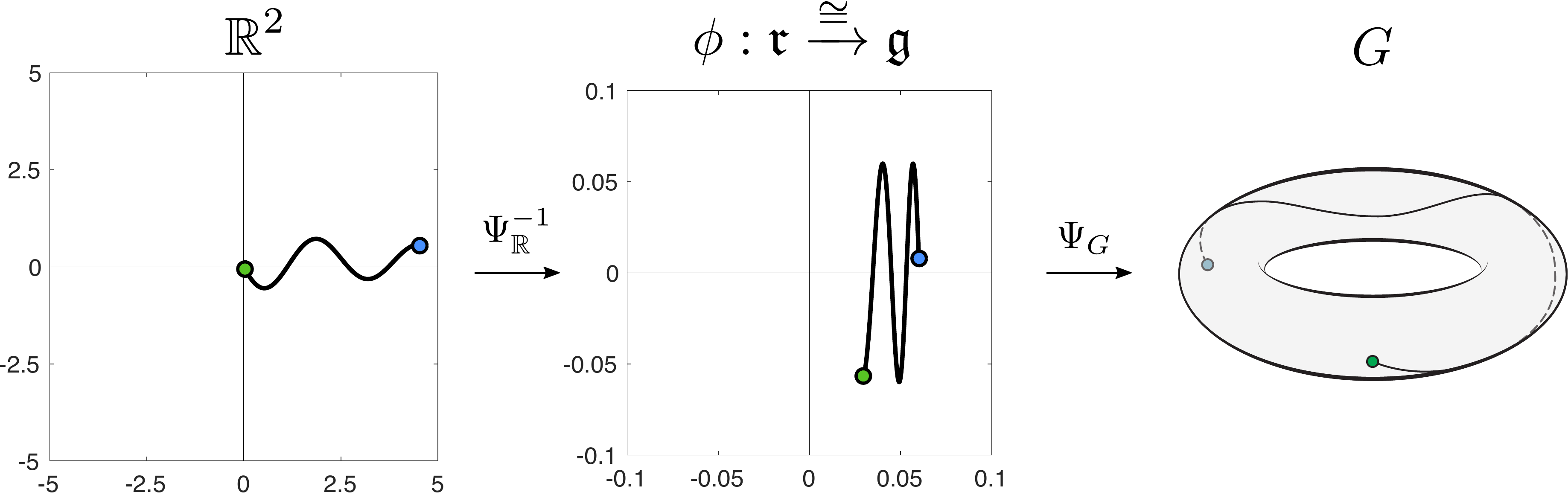}
	\caption{An example of the function $\Phi$ when we take $G = S^1 \times S^1$. (Left) A path $\gamma \in P\R^2_0$. (Middle) The derivative of $\gamma$ as a path in $\fr$ or $\fg$. (Right) The corresponding path $\Phi(\gamma)$.}
    \end{figure}
    
    
    Because all three maps are bijective, $\Phi$ is also bijective. To show that the signatures are invariant under this mapping, let $I = (i_1, \ldots, i_k)$ and $\gamma \in P\R^N_0$. The path signature of $\gamma$ with respect to $I$ is
    \begin{align*}
        S^I_\R(\gamma) = \int_{\Delta^k} dx_{i_1}(\gamma_{t_1}) \ldots dx_{i_k}(\gamma_{t_k}) dt_1 \ldots dt_k.
    \end{align*}
    Note that the derivative of $\Phi(\gamma)$ is given by $\Phi(\gamma)'_t = \phi(\gamma'_t)$ and thus, the path signature of $\Phi(\gamma)$ with respect to $I$ is
    \begin{align*}
        S^I_G(\Phi(\gamma)) &= \int_{\Delta^k} \omega_{i_1}(\phi(\gamma'_{t_1})) \ldots \omega_{i_k}(\phi(\gamma'_{t_k})) dt_1 \ldots dt_k \\ 
        & = \int_{\Delta^k} \phi^*(\omega_{i_1})(\gamma'_{t_1}) \ldots \phi^*(\omega_{i_k})(\gamma'_{t_k}) dt_1 \ldots dt_k \\
        & = S^I_\R(\gamma).
    \end{align*}
    The final equality holds because the dual isomorphism $\phi^*$ takes $\omega_i$ to $dx_i$. Thus, $S^I_\R(\gamma) = S^I_G(\Phi(\gamma))$ for all $\Gamma \in P\R^N_e$ and all multi-indices $I$. 
\end{proof}

\subsection{Stability of the path signature}
\label{ssec:stability}

In this section, we will discuss the stability of path signatures, which is of crucial importance in machine learning applications. Because we will only be interested in truncated path signatures in applications, we will consider the truncated signature map $S_M : PG \rightarrow T^{\leq M}(\bar{\fg})$, where
\begin{align*}
    T^{\leq M}(\bar{\fg}) = \bigoplus_{m=0}^M (\bar{\fg})^{\otimes m},
\end{align*}
where $S_M$ only retains information about the first $M$ levels of the path signature. In addition, we define the projection map
\begin{align*}
    \pi_m: T((\bar{\fg})) \rightarrow \bar{\fg}^{\otimes m}
\end{align*}
to a particular tensor level. Such a map can also be defined on the truncated tensor algebra $\pi_m: T^{\leq M}(\bar{\fg}) \rightarrow \bar{\fg}^{\otimes m}$, and we denote all such maps in the same manner. \medskip

By stability of the path signature, we mean to say that the truncated signature map $S_M : PG \rightarrow T^{\leq M}(\bar{\fg})$ is Lipschitz continuous. In order to disucss such a notion, we must provide both $PG$ and $T^{\leq M}(\bar{\fg})$ with metrics. We begin with the metric on $T((\bar{\fg}))$, which is required in Section~\ref{ssec:universal} and is analogous to the metric on $T^{\leq M}(\bar{\fg})$. \medskip

Recall that a basis $(e_1, \ldots, e_N)$ of $\bar{\fg}$ induces a natural inner product on $\bar{\fg}$ by defining the basis to be orthonormal. This extends to an inner product structure on $\bar{\fg}^{\otimes m}$, and given $\mathbf{t}_k \in \bar{\fg}^{\otimes m}$, we will denote the norm by $\|\mathbf{t}\|_m$. In addition, this also extends to an inner product on $T^{\leq M}(\bar{\fg})$. Let $\mathbf{s}, \mathbf{t} \in T^{\leq M}(\bar{\fg})$. Such an inner product and norm are defined to be
\begin{align}
\label{eq:trunc_innerprod}
    \langle \mathbf{s}, \mathbf{t} \rangle  = \sum_{m=0}^M \sum_{|I|=m} \mathbf{s}^I \mathbf{t}^I, \quad
    \|\mathbf{t}\| & = \sqrt{ \sum_{m=0}^M \sum_{|I|=m}  (\mathbf{t}^I)^2 }.
\end{align}
Then, we can use the norm to define a metric on both $\bar{\fg}^{\otimes m}$ and $T^{\leq M}(\bar{\fg})$. Namely, given $\mathbf{s}_m, \mathbf{t}_m \in \bar{\fg}^{\otimes m}$ and $\mathbf{s}, \mathbf{t} \in T^{\leq M}(\bar{\fg})$, we have
\begin{align*}
    d_m(\mathbf{s}_m, \mathbf{t}_m) &= \|\mathbf{s}_m - \mathbf{t}_m\|_m \\
    d(\mathbf{s}, \mathbf{t}) &= \|\mathbf{s} - \mathbf{t}\|.
\end{align*}

Note that this norm on the tensor algebra extends to $T((\bar{\fg}))$, where the inner product and norm for $\mathbf{s}, \mathbf{t} \in T((\bar{\fg}))$ are defined as
in (\ref{eq:trunc_innerprod}) with $M\to\infty$.
In this case, the inner product and norm may be infinite. However, image of the path signature lies in a subalgebra of $T((\bar{\fg}))$ where the norm is finite. Namely, we define 
\begin{equation}
\label{eq:tensorsubspace}
    T_1((\bar{\fg})) \coloneqq \left\{ \mathbf{t} \in T((\bar{\fg})) \quad : \quad \|\mathbf{t}\| < \infty, \, \mathbf{t}_0 = 1 \right\}.
\end{equation}
We nowshow that the signature of any path $\gamma \in PG$ lies in this subspace.

\begin{lemma}
\label{lem:sig_finite_norm}
    Let $\gamma \in PG$. Then $\|S(\gamma)\| < \infty$. 
\end{lemma}
\begin{proof}
    Without loss of generality, we suppose that $\gamma$ is parametrized by length such that it is defined as $\gamma: [0,L] \rightarrow G$, where $L$ is the length, and $\|\gamma'_t\| = 1$ for all differentiable $t$; this assumption is valid due to the reparametrization invariance of the signature. We will inductively bound each signature term. At the first level, we have
    \begin{align*}
        |S^i(\gamma)(t)| & \leq \int_0^t |\omega_i(\gamma'_s)| ds \\
        & \leq t,
    \end{align*}
    using the fact that $|\omega_i(\gamma'_t)| \leq \|\gamma'_t\| = 1$. Assume that for any multi-index $I = (i_1, \ldots, i_{m-1})$ of length $m-1$, we have
    \begin{align*}
        |S^I(\gamma)(t)| \leq \frac{t^{m-1}}{(m-1)!}.
    \end{align*}
    Now consider the multi-index $I = (i_1, \ldots, i_m)$ of length $m$. Using the induction hypothesis, and the recursive definition of the signature, we have
    \begin{align*}
        |S^I(\gamma)(t)| 
        & \leq 
        \int_0^t |S^{(i_1, \ldots, i_{m-1})}(s)| |\omega_{i_m}(\gamma'_s)| ds 
        \\
        & 
        \leq 
        \int_0^t \frac{s^{m-1}}{(m-1)!} ds = \frac{t^m}{m!}.
    \end{align*}
    Therefore, for any multi-index $I$ of length $m$, we have
    \begin{align}
    \label{eq:sigbound}
        |S^I(\gamma)| \leq \frac{L^m}{m!},
    \end{align}
    and the norm of $S(\gamma)$ is bounded by
    \begin{align*}
        \|S(\gamma)\|^2 & = \sum_{m=0}^\infty \sum_{|I|=m} (S^I(\gamma))^2 \\
        & \leq \sum_{m=0}^\infty \frac{N^m L^{2m}}{(m!)^2} < \infty
    \end{align*}
    where the last inequality uses the fact that there are $N^m$ multi-indices of length $m$. 
\end{proof}

Next, we consider a metric structure on $PG$. We mentioned in Section~\ref{ssec:paths_liegroups} that a metric on $P\R^N_0$ can be obtained by using the 1-variation of paths. Namely, given $\alpha, \beta \in P\R^N_0$, we may consider the distance between these two paths as $|\alpha - \beta|_{1-var}$, where the difference is performed pointwise. Such an approach could be used for $PG$ in theory. Now suppose $\alpha, \beta \in PG_e$, we can use the metric on $G$ to define the distance to be $|\beta^{-1}\alpha|_{1-var}$, where the inversion and multiplication are both performed pointwise. However, this notion of distance is not well-suited for the path signature.

The main reason for this is that the computation of $|\beta^{-1}\alpha|_{1-var}$ depends fundamentally on the adjoint action of the Lie group on the Lie algebra, which is governed by the Lie bracket. Namely, the adjoint action is trivial if and only if the Lie bracket is zero. However, the path signature ignores the Lie bracket structure, so the prospect of Lipschitz continuity of the signature with respect to this metric is problematic.

We therefore consider a different metric. Note that our path signature computations have consistently been performed on the underlying vector space of the Lie algebra $\bar{\fg}$, so it seems natural to directly define a metric using the derivatives $\alpha', \beta' \in \hat{P}\fg$. One such notion of a distance would be the $L^1$ distance between these derivatives
\begin{align*}
    \|\alpha'-\beta'\|_{L^1} = \int_0^1 \|\alpha'_t - \beta'_t\|_\fg \ dt
\end{align*}
which in particular does not use the Lie bracket structure. In fact this $L^1$ distance is exactly the 1-variation of the corresponding paths $\Phi^{-1}(\alpha), \Phi^{-1}(\beta) \in P\R^N_0$, given by the bijection in Proposition~\ref{prop:relationship}. Thus, we can define the metric on $PG_e$ to be
\begin{align*}
    d_\R(\alpha, \beta) \coloneqq |\Phi^{-1}(\alpha) - \Phi^{-1}(\beta)|_{1-var}.
\end{align*}

Note that equipped with this metric, the map $\Phi$ is trivially an isometry.
\begin{lemma}
    Suppose $\Phi: P\R^N_0 \rightarrow PG_e$ is the map defined in Proposition~\ref{prop:relationship}. Suppose $P\R^N_0$ is equipped with the $1$-variation metric, and $PG_e$ is equipped with the metric $d_\R$. Then, $\Phi$ is an isometry.
\end{lemma}

Using this isometry, stability for Lie group path signatures is a direct corollary of stability for Euclidean path signatures.
\begin{proposition}[\cite{friz_multidimensional_2010}]
    Let $\alpha, \beta \in P\R_0$, and let 
    \begin{align*}
    	L \geq \max\{|\alpha|_{1-var}, |\beta|_{1-var}\}.
    \end{align*}
     Then, for all $k \geq 1$, there exists some $C_k > 0$ such that
    \begin{align*}
        \left\| \pi_k\Big( S(\alpha) - S(\beta)\Big) \right\|_k \leq C_k L^{k-1} |\alpha - \beta|_{1-var}.
    \end{align*}
\end{proposition}

\begin{corollary}
\label{cor:stability}
    Let $\alpha, \beta \in PG_e$, and let $L \geq \max\{|\alpha|_{1-var}, |\beta|_{1-var}\}$. Then, for all $k \geq 1$, there exists some $C_k > 0$ such that
    \begin{align*}
        \left\| \pi_k\Big( S(\alpha) - S(\beta)\Big) \right\|_k \leq C_k L^{k-1} d_\R(\alpha, \beta).
    \end{align*}
\end{corollary}

\subsection{Equivariance of the path signature}
\label{ssec:equivariance}
At this point, a natural question to consider is how do path signatures behave under Lie group morphisms? Let $G_1$ and $G_2$ be Lie groups of dimensions $N_1$ and $N_2$ respectively. Given a Lie group morphism $F: G_1 \rightarrow G_2$, we have an induced Lie algebra morphism $F_*: \fg_1 \rightarrow \fg_2$ between the corresponding Lie algebras. In particular, all Lie algebra morphisms are linear transformations, so if we forget the Lie bracket, this results in a map $F_*: \bar{\fg}_1 \rightarrow \bar{\fg}_2$ between the underlying vector spaces. Because linear transformations induce maps on tensor products of the space $F_*^{\otimes m} : \bar{\fg}_1^{\otimes m} \rightarrow \bar{\fg}^{\otimes m}_2$, we also get an induced map of algebras between tensor algebras
\begin{align*}
    F_*: T((\bar{\fg}_1)) \rightarrow T((\bar{\fg}_2)).
\end{align*}

If $(e_1, \ldots, e_{N_1})$ is an ordered basis for $\fg_1$ and $(f_1, \ldots, f_{N_2})$ is an ordered basis for $\fg_2$, then we can write $F_*: \fg_1 \rightarrow \fg_2$ as an $N_2 \times N_1$ matrix in terms of these bases, which we call $M$. We can describe the action of $F_*$ in the tensor algebra using this matrix. Let $\mathbf{t} \in T((\bar{\fg}_1))$. In general, the action on the order $m$ elements $\mathbf{t}_m \in \bar{\fg}^{\otimes m}$ is a tensor-matrix multiplication, as described in~\cite{pfeffer_learning_2019}, in which all $m$ sides of the tensor $\mathbf{t}_m$ are multiplied by the matrix $M$. This can be written out as
\begin{align*}
    F_* \mathbf{t} = \sum_{m=0}^\infty \sum_{|I|=m} \mathbf{t}^I (Me_{i_1}) \otimes (Me_{i_2}) \otimes \ldots \otimes (Me_{i_k}).
\end{align*}

The low order tensors can be written out in usual matrix notation. Consider $\mathbf{t}_1$ as a column vector. The action on first order elements is matrix multiplication,
\begin{align*}
    (F_* \mathbf{t})_1 = M \mathbf{t}_1.
\end{align*}
Considering $\mathbf{t}_2$ as a matrix, the action on the second order elements is conjugation by $M$,
\begin{align*}
    (F_* \mathbf{t})_2 = M \mathbf{t}_2 M^{\intercal}.
\end{align*}

For higher orders, we can no longer use matrix notation, so we explicitly define the action for a given index. Let $J=(j_1, \ldots, j_n)$ be a multi-index where $j_k \in [N_2]$. Then, the element of $F_*\mathbf{t}$ corresponding to the multi-index $J$ is
\begin{align*}
    (F_* \mathbf{t})^J = \sum_{i_1 = 1}^{N_1} \sum_{i_2 = 1}^{N_1} \ldots \sum_{i_n = 1}^{N_1} \mathbf{t}^{(i_1, \ldots, i_n)} M_{j_1, i_1} M_{j_2, i_2}, \ldots, M_{j_n, i_n}.
\end{align*}

The following is a generalization of the equivariance of the path signature in Euclidean space, which is discussed in~\cite{friz_multidimensional_2010} and~\cite{pfeffer_learning_2019}. Here, suppose $(\omega_1, \ldots, \omega_{N_1})$ is the dual basis to $(e_1, \ldots, e_{N_1})$ and $(\nu_1, \ldots, \nu_{N_2})$ is the dual basis to $(f_1, \ldots, f_{N_2})$.

\begin{proposition}
    Let $G_1$ and $G_2$ be Lie groups, with Lie algebras $\fg_1$ and $\fg_2$ respectively. Suppose $F: G_1 \rightarrow G_2$ is a Lie group morphism and $\gamma \in P(G_1)$. Then
    \begin{align*}
        S(F \gamma) = F_* S(\gamma). 
    \end{align*}
\end{proposition}

\begin{proof}
    The proof of this claim is simply due to the linearity of integrals and $1$-forms. Consider the multi-index $J = (j_1, \ldots, j_m)$. Then,
    \begin{align*}
        S^J(F\gamma) & = \int_{\Delta^m} \nu_{j_1}(F_*\gamma'_{t_1}) \ldots \nu_{j_m}(F_*\gamma'_{t_m})  dt_1 \ldots dt_m.
    \end{align*}
    Consider a single factor in the integrand. Using the basis $(e_1, \ldots, e_{N_1})$ for $\fg$, write the derivative $\gamma'$ as
    \begin{align*}
        \gamma'_t = \sum_{i=1}^{N_1} c^i_t e_i
    \end{align*}
    where $c^i: [0,1] \rightarrow \R$ are the component paths. Then, since $\nu_j(F_*\gamma_t')$ denotes the $j^{th}$ component of $F_* \gamma_t'$, we can write this as
    \begin{align*}
        \nu_{j}(F_*\gamma_t') = \sum_{i=1}^{N_1} M_{j,i} \omega_i(\gamma'_t).
    \end{align*}
    Substituting this back into the formula for $S^J(F\gamma)$, we get
    \begin{align*}
        S^J(F\gamma) & = \sum_{i_1=1}^{N_1} \ldots \sum_{i_n=1}^{N_1} \left(M_{j_1,i_1} \ldots M_{j_m,i_m}\right) \int_{\Delta^m} \omega_{i_1}(\gamma_t') \ldots \omega_{i_m}(\gamma_t') dt_1 \ldots dt_m \\
        & = \sum_{i_1=1}^{N_1} \ldots \sum_{i_n=1}^{N_1}  \left(M_{j_1,i_1} \ldots M_{j_m,i_m}\right) S^{(i_1, \ldots, i_m)}(\gamma) \\
        & = (F_* S(\gamma))^J.
    \end{align*}
\end{proof}


\subsection{Lead-lag relationships}
\label{ssec:leadlag}

For path signatures defined on Euclidean space, a certain linear combination of second degree signature terms provides a reparametrization invariant indicator of lead-lag behavior in cyclic real-valued time series, as initially introduced in~\cite{baryshnikov_cyclicity_2016}. In this subsection, we will extend this interpretation to time series valued in Lie groups. \medskip

A \style{cyclic} time series in $G$ is one which is periodic up to a time-reparametrization. More precisely, a time series $\gamma$ is cyclic if it factors through the circle,
\begin{align*}
    \gamma: [0,1] \xrightarrow{\phi} S^1 \rightarrow G,
\end{align*}
with $\phi$ monotone and winding around the circle at least twice, the winding condition enforcing nontrivial repetition. 

Consider the interpretation for Euclidean paths in $\R^2$. Suppose $\gamma = (\gamma^1, \gamma^2) \in P\R^2$ is a cyclic time series. We say that the component $\gamma^1$ exhibits a cyclic leading behavior with respect to the component $\gamma^2$ if the following two conditions hold:
\begin{enumerate}
    \item when $\gamma^1$ is large (small), then $\gamma^2$ is increasing (decreasing), and 
    \item when $\gamma^2$ is large (small), then $\gamma^1$ is decreasing (increasing).
\end{enumerate}
The first condition can be viewed as a reparametrization invariant definition of a time series $\gamma^1$ leading another time series $\gamma^2$. The second condition is used because we are working with cyclic time series, so we also consider the negative influence of $\gamma^2$ on $\gamma^1$. We may think of this phenomena as a feedback loop in which $\gamma^1$ positively influences $\gamma^2$ and $\gamma^2$ negatively influences $\gamma^1$. The standard example of such behavior is $\gamma_t = (\sin(t), -\cos(t))$. 

To quantify what we mean by large or small in the two conditions above, we translate the time series such that $\gamma_0 = (0,0)$ and interpret large (small) to mean positive (negative). Then, a measure for these two conditions are given by $S^{1,2}(\gamma)$ and $-S^{2,1}(\gamma)$ respectively,
\begin{align*}
    S^{1,2}(\gamma) = \int_0^1 \gamma^1_t (\gamma^2)'_t dt, \quad S^{2,1}= (\gamma) \int_0^1 \gamma^2_t (\gamma^1)'_t dt.
\end{align*}

Thus, a measure for cyclic leading behavior can be defined as
\begin{equation*}
    A^{1,2}(\gamma) = \frac12 \left( S^{1,2}(\gamma) - S^{2,1}(\gamma)\right) = \frac12 \int_0^1 \gamma^1_t (\gamma^2)'_t - \gamma^2_t (\gamma^1)'_t dt.
\end{equation*}

Because the signature is translation invariant, the translation to the origin described above does not affect this measure. Moreover, if we consider a time series $\gamma \in P\R^N$, then we can consider all pairwise cyclic leading behavior between components. We can place all of this information into a matrix called the \style{lead matrix}, $A(\gamma)$, which has entries
\begin{align}
\label{eq:leadmatrix}
    A^{i,j}(\gamma) = \frac12 \left(S^{i,j}(\gamma) - S^{j,i}(\gamma)\right).
\end{align}

The entries $A^{i,j}(\gamma)$ have a geometric interpretation in terms of the \style{signed area} of the path, as per~\cite{baryshnikov_cyclicity_2016}. 



An example of the second level signatures and the signed area is shown in the figure below.

\begin{figure}[!htbp]
\centering
	\includegraphics[width=0.9\textwidth]{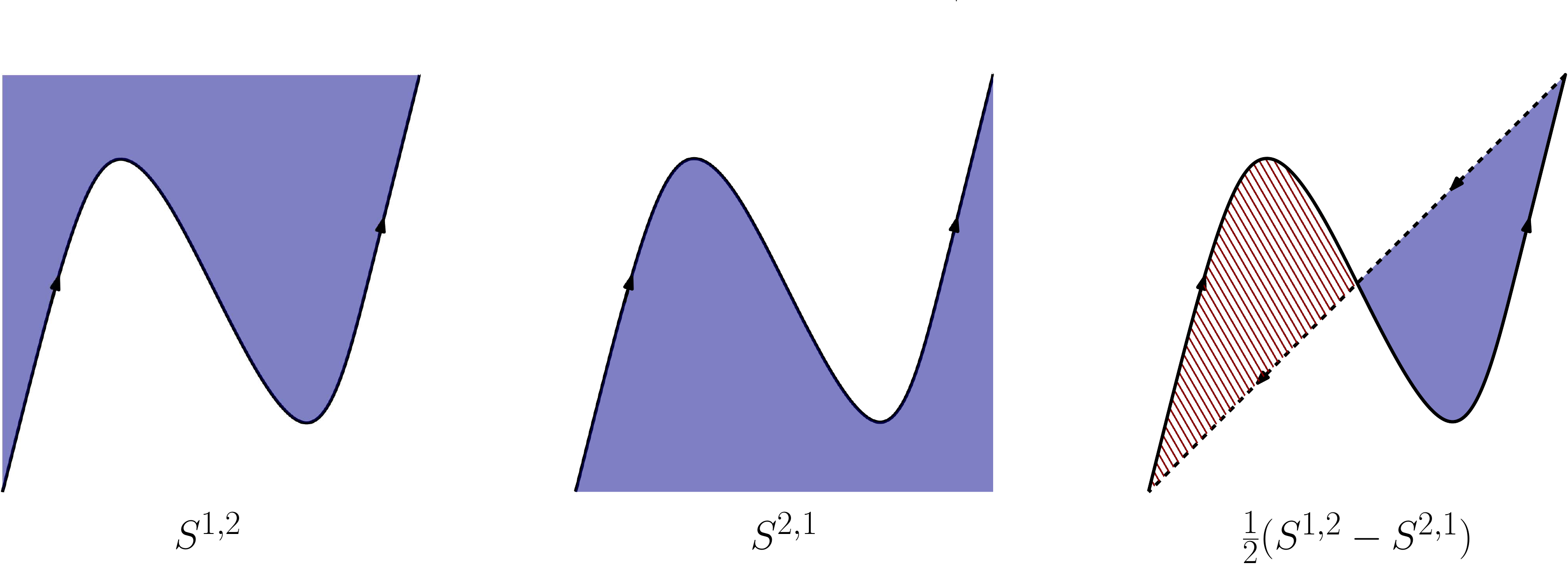}
	\caption{Second level signature computations $S^{1,2}$ (left), $S^{2,1}$ (middle), and the signed are $A^{1,2}$ (right). Blue represents positive area, while red represents negative area.}
\end{figure}

Returning to the setting of Lie groups, we can define the lead matrix of a path $\gamma \in PG$ in the same manner, but the interpretation must be slightly modified. Writing out the integral for $S^{i,j}(\gamma)$ given a basis $(\omega_1, \ldots, \omega_N)$ of $\fg^*$, we get
\begin{align*}
    S^{i,j}(\gamma) = \int_0^1 \left( \int_0^{t} \omega_i(\gamma'_{s}) ds \right) \omega_j(\gamma'_t) dt.
\end{align*}
The inner integral is simply $S^i(\gamma)_t$ and represents the cumulative variation of the path in the direction of $e_i$ (the dual of $\omega_i$), which is the analogue of the displacement in Euclidean space. Thus, for a cyclic time series $\gamma \in PG$, we say that the $e^i$ direction exhibits cyclic leading behavior with respect to the $e^j$ direction if the following holds:
\begin{enumerate}
    \item \textbf{(positive influence)} when $S^i(\gamma)_t$ is positive (negative), then $\omega_i(\gamma'_t)$ is positive (negative), and
    \item \textbf{(negative influence)} when $S^j(\gamma)_t$ is positive (negative), then $\omega_j(\gamma'_t)$ is negative (positive).
\end{enumerate}
Thus the lead matrix, as defined in Equation~\ref{eq:leadmatrix}, can be interpreted as a measure of this cyclic leading behavior for Lie group time series. An example of this interpretation is given in Section~\ref{ssec:g3d}.

However, the geometric interpretation in terms of signed area is no longer available. This is because any area computation on Lie groups will require second-order differential information about the paths, but path signatures are only defined using first order differential information. This suggests that an interpretation in terms of areas on the Lie group will not be possible. However, by using Proposition~\ref{prop:relationship}, the value $A^{i,j}(\gamma)$ can still be interpreted as the signed area of the corresponding path $\Phi^{-1}(\gamma)$, where $\Phi$ is the bijection given in Proposition~\ref{prop:relationship}.


\subsection{Topological considerations}
\label{ssec:topological}

In this section, we will consider the topological interpretation of the first level signature terms for some Lie groups. Namely, the first level signature term $S^i$ is homotopy invariant if the differential 1-form corresponding to the basis vector $\omega_i \in \fg^*$ is a closed form. For simplicity, we will assume that all paths are piecewise smooth in this section. Note in particular that the continuous interpretation of discrete time series is piecewise smooth, so the discussion in this section holds for analysis of these discrete time series. \medskip

Recall the following definition of a homotopy between paths.

\begin{definition}
Suppose $\alpha, \beta:[0,1] \rightarrow G$ are \style{homotopic relative to endpoints} if $\alpha_0 = \beta_0$, $\alpha_1 = \beta_1$ and there exists a continuous function $h: [0,1]^2 \rightarrow G$, called a \style{homotopy}, such that
\begin{align*}
    h(0,t) = \alpha_t, \quad h(1,t) = \beta_t, \quad h(s,0) = \alpha_0 = \beta_0, \quad h(s,1) = \alpha_1 = \beta_1.
\end{align*}
We use the notation $\alpha \simeq \beta$ if the paths $\alpha$ and $\beta$ are homotopic relative to endpoints.
\end{definition}

Loosely speaking, two paths are homotopic relative to endpoints if their endpoints coincide, and there exists a continuous deformation from one path to the other. Namely, homotopy relative to endpoints forms an equivalence relation in $PG$. We say that a map $f: PG \rightarrow \R$ is \style{homotopy invariant} if $f(\alpha) = f(\beta)$ whenever $\alpha \simeq \beta$. \medskip

Recall that a differential form $\omega$ is \style{closed} if its exterior derivative is trivial, $d\omega = 0$. By Stokes' theorem, the first level signature terms for closed forms are homotopy invariant. Indeed, let $\alpha, \beta \in PG$, and let $h: [0,1]^2 \rightarrow G$ be a homotopy between $\alpha$ and $\beta$. By Stokes' theorem, we have
\begin{align*}
    \int_{\partial h} \omega = \int_{h} d\omega = 0,
\end{align*}
but the boundary of $h$ is exactly $\alpha* \beta^{-1}$. Thus, we have
\begin{align*}
    \int_\alpha \omega - \int_\beta \omega = \int_{\partial h} \omega = 0.
\end{align*}

For left-invariant forms on Lie groups, there is a simple way to determine whether the form is closed. We begin with the invariant formula for the exterior derivative~\citep{lee_introduction_2003}. Let $\omega$ be a $1$-form on $G$, and $X,Y$ are vector fields on $G$, then
\begin{align*}
    d\omega(X,Y) = X(\omega(Y)) - Y(\omega(X)) - \omega([X,Y]).
\end{align*}
If in particular, if $\omega \in \fg^*$ is a left-invariant $1$-form and $X,Y \in \fg$ are left-invariant vector fields, then this formula reduces to
\begin{align*}
    d\omega(X,Y) = \omega([X,Y])
\end{align*}
since $\omega(X)$ and $\omega(Y)$ are constant functions. Thus, the left invariant form $\omega$ is closed if and only if $\omega([X,Y]) = 0$ for all $X, Y \in \fg$. In particular, this implies that all left-invariant $1$-forms are closed on abelian Lie groups such as $\R^N$ and $T^N$, since $[X,Y] = 0$ for all $X, Y \in \fg$. However, there are no closed left invariant $1$-forms on $SO(3)$ since a nontrivial $\omega \in \fso(3)^*$ must be nonzero for at least some $Z \in \fso(3)$. However, for all $Z \in \fso(3)$, there exist $X,Y \in \fso(3)$ such that $Z = [X,Y]$. In fact, this argument extends to all semisimple Lie groups, and thus there are no closed left-invariant $1$-forms on any semisimple Lie group. \medskip



\subsection{Discretization of the path signature}
\label{ssec:discrete_ps}

We have focused our discussion of the path signature so far on the continuous setting in order to discuss the theoretical framework. However, applications are studied in the discrete setting.  In this section, we provide the explicit computation of path signatures for discrete time series, and discuss a useful discrete approximation.\medskip

Let $v = (v^1, \ldots, v^N) \in \fg$, where we have written out the components of $v$ in terms of some choice of basis on $\fg$. Consider the continuous time series $\gamma_t = \exp(vt)$, where $\exp$ is the Lie exponential. Note that the derivative $\gamma'_t = v$ is constant. In this case, the path signature of $\gamma$ is straightforward to compute. Given $I = (i_1, \ldots, i_m)$, the path signature is
\begin{align*}
    S^I(\gamma) & = \int_{\Delta^m} \omega_{i_1}(\gamma'_{t_1}) \ldots \omega_{i_m}(\gamma'_{t_m}) dt_1 \ldots dt_m \\
    & = \int_{\Delta^m} v_{i_1} \ldots v_{i_m} dt_1 \ldots dt_m \\
    & = \frac{v_{i_1} \ldots v_{i_m}}{m!}.
\end{align*}
The entire path signature can be written concisely as the tensor exponential.

\begin{definition}
    Let $V$ be a real vector space. The tensor exponential $\exp_\otimes : V \rightarrow T((V))$ is defined to be
    \begin{align*}
        (\exp_\otimes(v))_m = \frac{v^{\otimes m}}{m!}.
    \end{align*}
\end{definition}

Then, we may write the path signature of $\gamma = \exp(vt)$ to be
\begin{align*}
    S(\gamma) = \exp_\otimes(v).
\end{align*}

Now, suppose we have a discrete time series $\hat{\gamma}:[T+1] \rightarrow G$. Recall from Equation~\ref{eq:disc_derivative} that we may compute the discrete derivative by
\begin{align*}
    \hat{\gamma}'_t = \log \left( \hat{\gamma}^{-1}_t \hat{\gamma}_{t+1}\right)
\end{align*}
to obtain a discrete time series $\hat{\gamma}': [T] \rightarrow \fg$. As discussed in Section~\ref{ssec:discretets}, we interpret these discrete time series as continuous paths by interpolating using the exponential. Thus, the continuous interpretation of the discrete time series can be thought of as a concatenation of several exponential paths
\begin{align*}
    \gamma = \exp(\hat{\gamma}'_1 t) * \exp(\hat{\gamma}'_2 t) * \ldots * \exp(\hat{\gamma}'_T t).
\end{align*}
Therefore, by the above computation of the path signature of an exponential path and Chen's identity, we define the continuous path signature of the discrete time series to be
\begin{align}
    S(\hat{\gamma}) = \exp_\otimes(\hat{\gamma}'_1) \otimes \exp_\otimes(\hat{\gamma}'_2) \otimes \ldots \otimes \exp_\otimes(\hat{\gamma}'_T).
\end{align}
By using tensor operations, this formula provides an effective implementation for the computation of the path signature. 

An alternative approach is to compute an approximation of the path signature for discrete time series. 

\begin{definition}
\label{def:discrete_ps}
    Let $\hat{\gamma}:[T+1] \rightarrow G$. Suppose $\omega_1, \ldots, \omega_N \in \fg^*$ form a basis of $\fg^*$. Let $I = (i_1, \ldots, i_m)$ be a multi-index, where $i_j \in [N]$. Let $\hat{\gamma}':[T] \rightarrow \fg$ be the discrete derivative of $\hat{\gamma}$. We define the discrete $m$-simplex with length $T$ to be
    \begin{align*}
        \hat{\Delta}^m_T = \{(t_1, \ldots, t_m) \in [T]^m \quad : \quad 0 \leq t_1 < t_2 < \ldots < t_m \leq 1\}.
    \end{align*}
    The discrete path signature of $\hat{\gamma}$ with respect to $I$ is defined to be
    \begin{align*}
        \hat{S}^I(\hat{\gamma}) \coloneqq \sum_{(t_1, \ldots, t_m) \in \hat{\Delta}^m_T} \omega_{i_1}(\hat{\gamma}'_{t_1}) \ldots \omega_{i_m}(\hat{\gamma}'_{t_m}),
    \end{align*}
    where $t_1, \ldots, t_m \in [T]$. The discrete path signature can be viewed as a map
    \begin{align*}
        \hat{S} : \hat{P}G \rightarrow T((\overline{\fg})).
    \end{align*}
\end{definition}

The discrete path signature can be viewed as an approximation to the continuous path signature. Let $\gamma \in PG$ be a continuous path. Given a partition $\pi = (0 = t_1 < t_2 < \ldots < t_{T+1} = 1)$, the discretization of $\gamma$ with respect to $\pi$, denoted $\hat{\gamma}^{(\pi)} : [T+1] \rightarrow G$, is defined to be
\begin{align*}
    \hat{\gamma}^{(\pi)}_i \coloneqq \gamma_{t_i}.
\end{align*}
The following proposition in~\cite{kiraly_kernels_2019} shows that the discrete signature indeed approximates the continuous path signature.

\begin{proposition}[Corollary 4.3,~\cite{kiraly_kernels_2019}]
    Let $\gamma \in P\R^N$, and define a partition $\pi = (0 = t_1 < t_2 < \ldots < t_{T+1} = 1)$. Then,
    \begin{align*}
        \|\hat{S}(\hat{\gamma}^{(\pi)}) - S(\gamma) \| \leq |\gamma|_{1-var}e^{|\gamma|_{1-var}} \max_{i=1, \ldots, T} |\gamma_{[t_i, t_{i+1}]}|_{1-var}.
    \end{align*}
\end{proposition}

By applying the map $\Phi: P\R^N \rightarrow PG$ from Proposition~\ref{prop:relationship}, and using the fact that it is an isometry, we immediately get the following corollary for Lie group valued paths.

\begin{corollary}
    Let $\gamma \in PG$, and define a partition $\pi = (0 = t_1 < t_2 < \ldots < t_{T+1} = 1)$. Then,
    \begin{align*}
        \|\hat{S}(\hat{\gamma}^{(\pi)}) - S(\gamma) \| \leq |\gamma|_{1-var}e^{|\gamma|_{1-var}} \max_{i=1, \ldots, T} |\gamma_{[t_i, t_{i+1}]}|_{1-var}.
    \end{align*}
\end{corollary}

We consider the discrete path signature since it is more amenable to computation. In particular,~\cite{kiraly_kernels_2019} derived efficient algorithms to compute the discrete path signature kernel, and we extend these algorithms to Lie groups in Section~\ref{ssec:kernel_trick}. In addition, the discrete path signature is of independent interest and the algebraic properties of the discrete path signature are studied in~\cite{diehl_time-warping_2020}, where it is called the iterated sums signature.


\subsection{Path transformations}
\label{ssec:pathtrans}
We turn to discussion of several transformations which add information to paths.

\subsubsection{Time transformation}
The \style{time transformation} is a simple method to remove reparametrization invariance (common in the path signatures literature such as in~\cite{chevyrev_primer_2016}), defined by appending the time parameter to the path:
\begin{align}
    T_{\timerm}: PG_e &\rightarrow \widetilde{P}(G \times \R) \nonumber \\
    \gamma_t &\mapsto (\gamma_t, t).
\end{align}

\begin{lemma}
    The map $T_{\timerm}: PG_e \rightarrow \widetilde{P}(G\times \R)$ is injective.
\end{lemma}
\begin{proof}
    Consider $\alpha, \beta \in PG_e$. The time parameter of $T_{\timerm}(\alpha)$ is monotone increasing, implying the path $T_{\timerm}(\alpha)$ is irreducible. Thus, since we also assume that $\alpha_0 = \beta_0 = e$, the paths $T_{\timerm}(\alpha)$ and $T_{\timerm}(\beta)$ are tree-like equivalent if and only if they differ by a reparametrization. However, all paths in the image of $T_{\timerm}$ have the same parametrization in the time coordinate. Thus, $T_{\timerm}(\alpha)$ is tree-like equivalent to $T_{\timerm}(\beta)$ if and only if $\alpha = \beta$. 
\end{proof}
Because the path signature is injective with respect to tree-like equivalence classes of paths, this lemma implies that by first embedding a path $\gamma \in PG_e$ into $\widetilde{P}(G\times \R)$, we obtain a parametrization-dependent feature map for $PG_e$ using the path signature. More generally, this also removes tree-like invariance, so in fact we obtain an injective feature map for $PG_e$.

\subsubsection{Identity initialized transformation}
The \style{identity initialized (IdInit) transformation} is a simple method to remove translation invariance: to the authors' knowledge, this has not been explicitly discussed in the context of path signatures. Suppose $g \in G$ and we define $\ell_g \in PG$ to be the exponential path from the identity $e \in G$ to the point $g \in G$. Then the IdInit transformation is defined to be
\begin{align}
    T_{\IdInit} : PG &\rightarrow PG_e \nonumber \\
    \gamma &\mapsto \ell_{\gamma_0} * \gamma .
\end{align}

For discrete time series, this simply amounts to appending the identity element to the beginning of the time series. The following lemma is clear by definition. 
\begin{lemma}
    The map $T_{\IdInit}: PG \rightarrow PG_e$ is injective.
\end{lemma}

The implication of this lemma is that by embedding $PG$ into $PG_e$, we can obtain a translation-dependent feature map for $PG$ using the path signature. Moreover, by combining the time and identity start transformations, the path signature provides an injective feature map for $PG$

\begin{corollary}
    The composition $T_{\timerm}\circ T_{\IdInit} : PG \rightarrow \widetilde{P}(G\times\R)$ is injective. 
\end{corollary}

\subsubsection{Sliding window transformation}
\label{sssec:sliding_window}
The sliding window transformation (often called the lead-lag transformation in the path signature literature) has been used for several applications of the path signatures, such as~\cite{gyurko_extracting_2013} and~\cite{yang_developing_2019}, and has been found to produce good results. An empirical study by~\cite{fermanian_embedding_2019} has shown that the sliding window transformation performs well on classification tasks. The exact definition may vary slightly between these papers, and we give the definition from~\cite{yang_developing_2019} and~\cite{fermanian_embedding_2019}.

Let $\tau \in \R^+$, and $m \in \N$. Given a path $\gamma \in PG$, which we recall is defined on the unit interval $\gamma:[0,1] \rightarrow G$, we can extend its definition to all of $\R$ as
\begin{align*}
    \gamma_t = \left\{
        \begin{array}{cl}
            \gamma_0 & : t < 0 \\
            \gamma_t & : t \in [0,1] \\
            \gamma_1 & : t > 1.
        \end{array}
    \right.
\end{align*}
Then, we define the sliding window transformation with $m$ lags to be
\begin{align}
    T_{\SW, m}: PG &\rightarrow PG^{m+1} \nonumber \\
    \gamma_t & \mapsto (\gamma_t, \gamma_{t-\tau}, \gamma_{t-2\tau}, \ldots, \gamma_{t-m\tau}).
\end{align}

For discrete time series, we assume that the data is temporally uniformly sampled, and we choose $\tau$ to be the time in between samples. Then, if we consider a discrete time series to be $\tilde{\gamma}: \{0, \ldots, n\} \rightarrow G$, then the sliding window transformation will be
\begin{align*}
    T_{\SW, m}(\tilde{\gamma})_i \coloneqq (\tilde{\gamma}_i, \tilde{\gamma}_{i-1}, \ldots, \tilde{\gamma}_{i-m}).
\end{align*}

In the context of Euclidean path signatures, \cite{fermanian_embedding_2019} empirically shows that the sliding window embedding often performs well on classification tasks, though there is no theoretical explanation. We note that due to the choice of padding the start of the delayed time series with the identity, this transformation breaks translation invariance. We suggest that breaking the translation symmetry is one reason the sliding window transformation performs well in practice. This is discussed in Remark~\ref{rem:translation_breaking} in Section~\ref{ssec:g3d}.


\section{The path signature kernel for Lie groups}
\label{sec:kernel}

In this section, we show that a normalized variant of the path signature can be used to define a universal and characteristic kernel for Lie group valued time series. We begin by giving an overview of kernel methods and the types of problems that kernel methods can solve. Next, we extend the results of~\cite{chevyrev_signature_2018} to show that path signatures for Lie group valued paths are universal and characteristic. Finally, we show that the algorithms introduced in~\cite{kiraly_kernels_2019} for efficient computation of the signature kernel can also be used for Lie group valued paths.


\subsection{Background on universal and characteristic kernels}

Suppose $\cX$ is a topological space which represents the space of data we would like to consider; we will call this the \style{input space}. Many tasks in machine learning can be separated into two broad classes of problems.
\begin{enumerate}
    \item Those which involve making inferences about functions $f \in \cF$, where $\cF \subset \R^{\cX}$ is the function class which we are considering. Performing binary classification reduces to learning a function $f \in \cF$ such that the level set $f=0$ represents the \style{decision boundary} between the two classes of data.
    
    \item Those which involve making inferences about probability measures $\mu \in \cP(\cX)$, where $\cP(\cX)$ denotes the space of Borel probability measures on $\cX$. For example, in two sample hypothesis testing, we begin with samples $\{x_1, \ldots, x_n\}$ and $\{y_1, \ldots, y_m\}$ taken from probability distributions $p$ and $q$ on $\cX$ respectively. Testing the null hypothesis that $p=q$ then corresponds to learning about the underlying measures of $p$ and $q$. 
\end{enumerate}

The general philosophy behind kernel methods is to map the input space $\cX$ into a reproducing kernel Hilbert space (RKHS) $\cH$ using a \style{feature map}
\begin{equation*}
    \Phi: \cX \rightarrow \cH_\kappa,
\end{equation*}
where the corresponding kernel is
\begin{align*}
    \kappa: \cX \times \cX \rightarrow \R, \quad \kappa(x,x') = \langle \Phi(x), \Phi(x')\rangle.
\end{align*}

Problems involving learning nonlinear functions $f \in \R^\cX$ given some input data $\{x_i\}$, where $x_i \in \cX$, can be reformulated as problems involving learning an element $g \in \cH$ (which can be thought of as a function $g \in \R^{\cX}$) given the data $\{\Phi(x_i)\}$. Additionally, the norm induced by the Hilbert space provides a metric between points $x, y \in \cX$ as $\|x-y\|$. In essence, this translates a \style{nonlinear} learning problem into a \style{linear} learning problem. This allows the application of linear methods, which are much simpler and better developed in many cases. 

Measures $\mu$ on $\cX$ can be mapped into the RKHS via the \style{kernel mean embedding} (KME),
\begin{equation}
    \overline{\Phi}: \cM(\cX) \rightarrow \cH_\kappa, \quad \overline{\Phi}(\mu) \coloneqq \int_{\cX} \Phi(x) d\mu(x) = \E_\mu[\Phi].
\end{equation}
A priori, this map is not necessarily well-defined, so we will usually require restrictions on the feature map or kernel such that the integral exists.

\begin{lemma}[\cite{sriperumbudur_hilbert_2010}]
    Let $\mu \in \cM(\cX)$. If the kernel $\kappa: \cX \rightarrow \cX \rightarrow \R$ is measurable and the integral $\int_{\cX} \sqrt{\kappa(x,x)} d\mu(x) < \infty$, then $\overline{\Phi}(\mu) \in \cH_\kappa$.
\end{lemma}

Specifically, the bounded integral condition is satisfied if we know that $\kappa(x,x) = \|\Phi(x)\|^2 < C$ for all $x \in \cX$ for some fixed constant $C$. In other words, if the image of the feature map $\Phi$ is contained in a bounded subset of $\cH_\kappa$, then the KME is well defined. 

Similar to the previous case, we can use the norm on $\cH_\kappa$ to define a notion of distance on $\cP(\cX)$. Although this only provides a pseudometric since $\overline{\Phi}$ may not be injective, it coincides with a well known measure of discrepancy between probability measures.

\begin{definition}
    Let $\cF \subset \R^\cX$ be a class of functions and $\mu, \nu \in \cP(\cX)$. The \style{maximum mean discrepancy} (MMD) of $\mu$ and $\nu$ with respect to $\cF$ is
    \begin{equation}
        \MMD[\cF, \mu, \nu] \coloneqq \sup_{f \in \cF} \left( \E_\mu[f] - \E_\nu[f]\right).
    \end{equation}
\end{definition}

When we take the function class $\cF$ to be the unit ball in the RKHS $\cH_\kappa$, the MMD can be written as the distance between the mean embeddings, with respect to the norm on $\cH$.
\begin{lemma}[\cite{borgwardt_integrating_2006}]
    Suppose the KME map $\overline{\Phi}$ is well-defined and suppose $\mu, \nu \in \cM(\cX)$. Let $\cF = \{ f \in \cH \subset \R^\cX \, : \, \|f\| \leq 1\}$. Then,
    \begin{equation}
        \MMD^2[\cF, \mu, \nu] = \|\overline{\Phi}(\mu) - \overline{\Phi}(\nu)\|^2.
    \end{equation}
\end{lemma}
This simplifies the study of probability measures by considering them as elements of a linear space, and also provides a straightforward method to compute an unbiased finite sample estimate of the MMD in terms of the kernel.

\begin{lemma}[\cite{gretton_kernel_2012}]
    Suppose the KME map $\overline{\Phi}$ is well-defined and suppose $\mu, \nu \in \cM(\cX)$. Let $\cF = \{ f \in \cH \subset \R^\cX \, : \, \|f\| \leq 1\}$. Let $X = (x_1, \ldots, x_n)$ and $Y = (y_1, \ldots, y_m)$ be i.i.d. samples from $\mu$ and $\nu$ respectively. An unbiased estimate of $\MMD^2[\cF, \mu, \nu]$ is given as the MMD of the empirical distributions of $X$ and $Y$,
    \begin{align}
        \MMD^2_u[\cF, X,Y] =&\frac{1}{n(n-1)}\sum_{i=1}^n \sum_{j \neq i}^n  \kappa(x_i, x_j) + \frac{1}{m(m-1)}\sum_{i=1}^m \sum_{j\neq i}^m \kappa(y_i, y_j) \nonumber \\ 
        &- \frac{2}{nm} \sum_{i=1}^n \sum_{j=1}^m \kappa(x_i, y_j). \label{eq:emp_mmd}
    \end{align}
\end{lemma}

From this discussion, kernels provide a unified way to study both nonlinear functions and probability measures using the linear space $\cH$. However, there are deficiencies in both scenarios.
\begin{enumerate}
    \item In the case of nonlinear functions, we usually begin by choosing our function class $\cF \subset \R^\cX$. How do we know that any function $f: \cX \rightarrow \R$ can be represented arbitrarily closely by an element in $\ell \in \cH$ such that $f(x) \approx \langle \ell, \Phi(x)\rangle$ for all $x \in \cX$?
    
    \item In the case of probability measures, it is often crucial that the MMD is in fact a metric instead of just a pseudometric. How do we know that the feature map $\Phi$ is rich enough to distinguish all probability measures $\mu \in \cM(\cX)$?
\end{enumerate}

The answer is given by the definitions of universal and characteristic kernels. This will require us to extend the definition of the KME to Schwarz distributions rather than just measures. We provide a quick exposition of the definitions here, and refer the reader to a more thorough treatment in~\cite{simon-gabriel_kernel_2018}. As usual, let $\cF \subset \R^\cX$ be a function class, and let $\cF'$ denote its topological dual of all continuous linear functionals. The definition of the KME for distributions is analogous to the case of measures
\begin{equation}
    \overline{\Phi}: \cF' \rightarrow \cH_\kappa, \quad \overline{\Phi}(D) \coloneqq \int_{\cX} \Phi(x) dD(x),
\end{equation}
where the integral here is the \style{weak-} or \style{Pettis-} integral~\citep{simon-gabriel_kernel_2018}. Similar to the case of measures, this map is a priori not well-defined. However, we have a simple criterion for the existence of these weak integrals.

\begin{lemma}[\cite{simon-gabriel_kernel_2018}]
\label{lem:KMEdef}
    Let $\cF \subset \R^\cX$. If the map $\iota: \cH_\kappa \rightarrow \R^\cX$ defined by $\iota(\ell) \coloneqq \langle \ell, \Phi(\cdot)\rangle$ has image contained in $\cF$, then the integral $\overline{\Phi}(D)$ exists for all $D \in \cF'$. Thus, the KME map $\overline{\Phi}:\cF' \rightarrow \cH_\kappa$ is well-defined.
\end{lemma}

We can now state the definition of a universal and characteristic feature map. 

\begin{definition}
 Fix an input space $\cX$, and a function space $\cF \subset \R^\cX$. Consider a feature map
 \begin{align*}
     \Phi : \cX \rightarrow \cH_\kappa
 \end{align*}
 into an RKHS $\cH_\kappa$ with respect to a kernel $k$. Suppose that $\langle \ell, \Phi(\cdot ) \rangle \in \cF$ for all $\ell \in \cH_\kappa$. We say that $\Phi$ is
 \begin{enumerate}
     \item \style{universal to $\cF$} if the map
     \begin{align*}
         \iota : \cH_k \rightarrow \cF, \quad \ell \mapsto \langle \ell, \Phi(\cdot) \rangle
     \end{align*}
     has a dense image in $\cF$; and
     \item \style{characteristic to a subset $\cP \subseteq \cF'$} if the KME map
     \begin{align*}
         \overline{\Phi}: \cP \rightarrow \cH_k, \quad D \mapsto \int_\cX \Phi(x) dD(x)
     \end{align*}
     is injective.
 \end{enumerate}
\end{definition}

Note that we have assumed that the image $\iota(\cH_\kappa) \subset \cF$ so by Lemma~\ref{lem:KMEdef}, the KME map is well defined. The property of universality allows us to approximate any function $f \in \cF$ using linear functionals $\langle \ell, \Phi(\cdot) \rangle$ for $\ell \in \cH_\kappa$. The dual of a class of functions $\cF$ is generally much larger than the set of probability measures on $\cX$. If $\cM(\cX) \subset \cF'$, then a characteristic feature map is able to represent probability measures on $\cX$ with elements of $\cH$. Moreover the MMD becomes a metric due to the injectivity of the KME. 

We have the following equivalence between universality and characteristicness, as shown in \cite{simon-gabriel_kernel_2018} and \cite{chevyrev_signature_2018}.
\begin{theorem}
\label{thm:duality}
    Suppose that $\cF$ is a locally convex topological vector space. A feature map $\Phi$ is universal to $\cF$ if and only if $\Phi$ is characteristic to $\cF'$. 
\end{theorem}


\subsection{The path signature kernel}
\label{ssec:universal}
In this subsection, we will define the path signature kernel, and show that it is universal and characteristic. This was shown for the Euclidean case in~\cite{chevyrev_signature_2018}, which states that these properties hold when studying paths evolving in a Hilbert space. Through our definition of the path signature for Lie groups, we provide a clarification: the space itself need not be a Hilbert space, but rather the space of tangent vectors must be a Hilbert space. In our case of Lie groups, this is the Lie algebra with a Riemannian metric. With this setup, the path signature kernel for Lie groups is universal and characteristic. \medskip

Recall that we have defined $T_1((\bar{\fg}))$ in Equation~\ref{eq:tensorsubspace} to be the subspace of $T((\bar{\fg}))$ with constant value $1$ and finite norm. We will view the path signature
\begin{align*}
    S: \widetilde{PG} \rightarrow T_1((\bar{\fg}))
\end{align*}
as a feature map, and recall that $T_1((\bar{\fg}))$ is equipped with an inner product, and is in particular a Hilbert space. However, as discussed in the previous subsection, we will need to ensure that the signature map sends paths to a bounded subset of $T_1((\bar{\fg}))$ in order for the KME to be defined. This will be done by using a tensor normalization, which was first discussed in~\cite{chevyrev_signature_2018}.

\begin{definition}
    A \style{tensor normalization} is a continuous injective map of the form
    \begin{align*}
        \Lambda : T_1((\bar{\fg})) &\rightarrow \{ \mathbf{t} \in T_1((\bar{\fg})) \, : \, \|\mathbf{t}\| \leq K \} \\
        \mathbf{t} & \mapsto \delta_{\lambda(\mathbf{t})} \mathbf{t}
    \end{align*}
    where $K > 0$ is a constant, $\lambda : T_1((\bar{\fg})) \rightarrow (0, \infty)$ is a function, and $\delta_{\lambda}$ is the tensor dilation from Definition~\ref{eq:tensordilation}.
\end{definition}

We will discuss the construction and computational aspects of tensor normalization in Appendix~\ref{apx:tensornorm}. For now, we will assume that a tensor normalization exists, and set $\Lambda: T_1((\bar{\fg})) \rightarrow T_1((\bar{\fg}))$ to be a fixed tensor normalization. Now, we have the \style{normalized signature}
\begin{align*}
    \Phi_S : \widetilde{PG} \rightarrow T_1((\bar{\fg})), \quad \Phi_S = \Lambda \circ S,
\end{align*}
which is a continuous injective map from $\widetilde{PG}$ into a bounded subset of $T_1((\bar{\fg}))$. Note that due to the scaling property of the path signature from Proposition~\ref{prop:ps_scaling}, this is equivalent to
\begin{align*}
    \Phi_S(\gamma) = S( \lambda(S(\gamma)) \cdot \gamma),
\end{align*}
where we first scale the path in $G$ by $\lambda(S(\gamma))$, using the Lie algebra scaling from Definition~\ref{def:scaling}.\medskip

Following~\cite{chevyrev_signature_2018}, we will show universality and then use the duality in Theorem~\ref{thm:duality} to show characteristicness with respect to probability measures. Using the theory discussed in the previous section, the objective is to find a function class $\cF \subset \R^{PG}$ and a topology on $PG$ such that
\begin{enumerate}
    \item the function class $\cF$ can be approximated by linear functionals $\langle \ell, \Phi_S(\cdot)\rangle$, and
    \item the dual $\cF'$ contains probability measures on $PG$.
\end{enumerate}
The difficulty with such a result is due to the fact that $PG$ is not locally compact. However, the class of continuous bounded functions $C_b(PG, \R)$ has such properties when $PG$ is endowed with the \style{strict topology}, originally defined in~\cite{giles_generalization_1971}. 
\begin{definition}
    Let $X$ be a topological space. We say that a function $\psi: X \rightarrow \R$ vanishes at infinity if for each $\epsilon> 0$, there exists a compact set $K \subset X$ such that $\sup_{x \in X-K} |\psi(x)| < \epsilon$. Denote by $B_0(X, \R)$ the set of functions that vanish at infinity. The \style{strict topology} on $C_b(X, \R)$ is the topology generated by the seminorms
    \begin{align*}
        p_\psi(f) = \sup_{x \in X} |f(x) \psi(x)|, \quad \psi \in B_0(X, \R).
    \end{align*}
\end{definition}

\begin{theorem}
\label{thm:sw}
Let $X$ be a metrizable topological space.
\begin{enumerate}
    \item The strict topology on $C_b(X,\R)$ is weaker than the uniform topology and stronger than the topology of uniform convergence on compact sets.
    \item If $\cF_0$ is a subalgebra of $C_b(X,\R)$ such that for all $x, y\in X$, there exists some $f \in \cF_0$ such that $f(x) \neq f(y)$ ($\cF_0$ separates points), and for all $x \in X$, there exists some $f \in \cF_0$ such that $f(x) \neq 0$, then $\cF_0$ is dense in $C_b(X,\R)$ under the strict topology.
    \item The topological dual of $C_b(X,\R)$ equipped with the strict topology is the space of finite regular Borel measures on $X$. 
\end{enumerate}
\end{theorem}
Specifically, note that the space of finite regular Borel measures on $X$ includes all probability measures on $X$. Finally, we are ready to state the universality and characteristicness result.

\begin{theorem}
Let $\Lambda: \overline{T}_1(\R^n) \rightarrow \overline{T}_1(\R^n)$ be a tensor normalization. The normalized signature
\begin{align*}
    \Phi: \widetilde{PG} \rightarrow T_1((\bar{\fg})), \quad \Phi = \Lambda \circ S,
\end{align*}

\begin{enumerate}
    \item is a continuous injection from $\widetilde{PG}$ into a bounded subset of $\overline{T}_1(\R^n)$,
    \item is universal to $\cF \coloneqq C_b(\widetilde{PG} , \R)$ equipped with the strict topology, and 
    \item is characteristic to the space of finite regular Borel measures on $\widetilde{PG}$. 
\end{enumerate}
\end{theorem}
\begin{proof}
    The fact that $\Phi$ is an injection follows from the injectivity of the path signature from Theorem~\ref{thm:injectivity} and the definition of the tensor normalization. Continuity follows from the stability property of Corollary~\ref{cor:stability}. Next, we move on to universality. Define
    \begin{align*}
        L = 1+ \bigoplus_{m=1}^\infty (\bar{\fg})^{\otimes m}
    \end{align*}
    to be a dense subspace of $T_1((\bar{\fg}))$ (note that $L$ only contains finite linear combinations of tensors, whereas $T_1((\bar{\fg}))$ contains all power series of tensors) and define
    \begin{align*}
        \cF_0 = \{\langle \ell, \Phi (\cdot)\rangle : \widetilde{PG} \rightarrow \R \, : \, \ell \in L\}.
    \end{align*}
    We aim to show that $\cF_0$ satisfies the hypotheses of the second point of Theorem~\ref{thm:sw}. By the injectivity of $\Phi$, the class of functions $\cF_0$ separates points, and because the path signature is defined with constant term $1$, the path signature is nonzero for all paths $\gamma \in \widetilde{PG}$. Finally, by the shuffle product identity from Theorem~\ref{thm:shuffle}, the class of functions $\cF_0$ is closed under shuffle multiplication and is therefore a subalgebra of $C_b(\widetilde{PG}, \R)$. Namely, let $I$ and $J$ be multi-indices, and $e_I$ and $e_J$ be the corresponding basis vectors in $L$. Then, we may define multiplication in $\cF_0$ by
    \begin{align*}
        \langle e_I, \Phi(\cdot)\rangle \langle e_J, \Phi(\cdot) \rangle = \left \langle \sum_{K \in I \shuffle J} e_K, \Phi(\cdot) \right\rangle,
    \end{align*}
    which is closed. Thus, $\Phi$ is universal with respect to $\cF$. Finally, by the duality in Theorem~\ref{thm:duality} and the third point of Theorem~\ref{thm:sw}, the function class $\cF$ is characteristic with respect to finite regular Borel measures on $\widetilde{PG}$. 
\end{proof}

\begin{remark}
    Although this theorem is stated for tree-like equivalence classes of paths in $G$, by precomposing with the time transformation or identity start transformation discussed in Section~\ref{ssec:pathtrans}, we can also obtain universal and characteristic feature maps that are not reparametrization or translation invariant.
\end{remark}


\subsection{The kernel trick}
\label{ssec:kernel_trick}
The \style{kernel trick} refers to an efficient method to compute kernels without having to compute explicit representations of elements in the feature space. Several efficient algorithms to compute the Euclidean path signature kernel are provided in~\cite{kiraly_kernels_2019}, who state that their algorithms hold for path signatures computed for Hilbert space-valued data (their methods hold in more generality). The algorithms depend only on an inner product structure in the space where the integrals are being computed, namely in $\bar{\fg}$, and thus also hold in our present context of Lie group-valued data. In this section, we will provide an explicit generalization of their main algorithm. \medskip

As in Section~\ref{ssec:stability}, we will restrict ourselves to the truncated signature
\begin{align*}
    S_M: PG \rightarrow T^{\leq M}(\bar{\fg}).
\end{align*}
Here, the inner product for $T^{\leq M}(\bar{\fg})$ was given in Equation~\ref{eq:trunc_innerprod}. The \style{signature kernel truncated at level $M$} is defined to be
\begin{equation}
    K_{M} : PG \times PG \rightarrow \R, \quad (\alpha, \beta) \mapsto \langle S_M(\alpha), S_M(\beta) \rangle
\end{equation}
We will begin by simplifying the computation of the kernel for continuous paths. 
\begin{proposition}
\label{prop:sigkernel_comp}
    The signature kernel can be computed as
    \begin{equation}
    \label{eq:sigkernel_comp}
        K_M(\alpha, \beta) = \sum_{m=0}^M \int_{(s,t) \in \Delta^m \times \Delta^m} \prod_{i=1}^m \langle \alpha'_{s_i}, \beta'_{t_i} \rangle_\fg\, ds\, dt,
    \end{equation}
    where we view $\alpha'_t$ and $\beta'_t$ as elements of the Lie algebra $\fg$, and the inner product in the integrand $\langle \cdot, \cdot \rangle_\fg$ is computed in the Lie algebra. Also, we denote $s = (s_1, \ldots, s_m)$ and $t= (t_1, \ldots, t_m)$ as elements of $\Delta^m$, and write $ds\coloneqq ds_1 \ldots ds_m$ and $dt \coloneqq dt_1 \ldots dt_m$. 
\end{proposition}
\begin{proof}
    Let's consider the inner product at a single level $m$. Recall that $\pi_m : T((\bar{\fg})) \rightarrow \bar{\fg}^{\otimes m}$ is the projection on to the level $m$ tensors, and $\langle \cdot , \cdot \rangle_m$ refers to the inner product on $\bar{\fg}^{\otimes m}$. Then,
    \begin{align*}
        \langle \pi_m  S(\alpha)&, \pi_m  S(\beta)\rangle_m \\
        & = \sum_{I : |I| = m} S^I(\alpha) \cdot S^I(\beta) \\
        & = \sum_{I = (i_1, \ldots, i_m)} \int_{\Delta^m} \omega_{i_1}(\alpha'_{s_1}) \ldots \omega_{i_m}(\alpha'_{s_m}) ds \int_{\Delta^m} \omega_{i_1}(\beta'_{t_1}) \ldots \omega_{i_m}(\beta'_{t_m})  dt \\
        & = \int_{(s,t) \in \Delta^m \times \Delta^m} \sum_{I = (i_1, \ldots, i_m)} \left[\omega_{i_1}(\alpha'_{s_1}) \omega_{i_1}(\beta'_{t_1})\right] \ldots \left[\omega_{i_m}(\alpha'_{s_m}) \omega_{i_m}(\beta'_{t_m})\right] \, ds\, dt \\
        & = \int_{(s,t) \in \Delta^m \times \Delta^m} \prod_{i=1}^m \langle \alpha'_{s_i}, \beta'_{t_i} \rangle_\fg\, ds\, dt.
    \end{align*}
    Then, adding up all of the levels, we get our desired result.
\end{proof}

As noted by~\cite{kiraly_kernels_2019}, the expression in Equation~\ref{eq:sigkernel_comp} can be efficiently computed by a method that is similar to Horner's scheme for computing polynomial expressions. Suppose we wish to compute the expression $p(x) = \sum_{i=0}^M  x^i$. By expanding this polynomial as
\begin{align*}
    p(x) = 1 + x(1 + x(1 + \ldots + x(1 +  x) ) )
\end{align*}
where the recursion occurs $M$ times and computing the brackets from the inside to the outside, we can evaluate the expression using $M$ additions and $M$ multiplications. In contrast, the naive computation of $p(x)$ would require $M$ additions and $\frac{M^2 + M}{2}$ multiplications. Note that we may write out this recursion explicitly as follows. Let
\begin{align*}
    q_1 = 1+x, \quad q_m = 1 + xq_{m-1}.
\end{align*}
Then, we may write $p(x) = q_M$. We can significantly reduce the number of operations required to compute the integrals in Equation~\ref{eq:sigkernel_comp} by adapting this procedure. 

\begin{corollary}
\label{cor:cont_kerneltrick}
    Let
    \begin{align}
    \label{eq:Q1}
        Q_1(s, t) = 1 + \int_{\begin{subarray}{l} s' \in [0,s]\\ t' \in [0,t]\end{subarray}} \langle \alpha'_{s'}, \beta'_{t'}\rangle_\fg ds' dt'.
    \end{align}
    Then, we recursively define
    \begin{align}
    \label{eq:Qm}
        Q_m(s, t) = 1 + \int_{\begin{subarray}{l} s' \in [0,s]\\ t' \in [0,t]\end{subarray}} Q_{m-1}(s',t') \langle \alpha'_{s'}, \beta'_{t'}\rangle_\fg ds' dt'.
    \end{align}
    The signature kernel can be computed as
    \begin{align}
    \label{eq:sigkernel_rec}
        K_M(\alpha, \beta) = Q_M(1, 1).
    \end{align}
\end{corollary}
\begin{proof}
    We will consider the case of $M=2$. Here, we have
    \begin{align*}
        Q_2(1,1) & = 1 + \int_{\begin{subarray}{l} s_1 \in [0,1]\\ t_1 \in [0,1]\end{subarray}} Q_{1}(s_1,t_1) \langle \alpha'_{s_1}, \beta'_{t_1}\rangle_\fg ds_1 dt_1 \\
        & = 1 + \int_{\begin{subarray}{l} s_1 \in [0,1]\\ t_1 \in [0,1]\end{subarray}} \left(1 + \int_{\begin{subarray}{l} s_2 \in [0,s_1]\\ t_2 \in [0,t_1]\end{subarray}}  \langle \alpha'_{s_2}, \beta'_{t_2}\rangle_\fg ds_2 dt_2\right)  \langle \alpha'_{s_1}, \beta'_{t_1}\rangle_\fg ds_1 dt_1 \\
        & = 1 + \int_{\begin{subarray}{l} s_1 \in [0,1]\\ t_1 \in [0,1]\end{subarray}} \langle \alpha'_{s_1}, \beta'_{t_1}\rangle_\fg ds_1 dt_1  + \int_{\begin{subarray}{l} s_1 \in [0,1]\\ t_1 \in [0,1]\end{subarray}} \int_{\begin{subarray}{l} s_2 \in [0,s_1]\\ t_2 \in [0,t_1]\end{subarray}} \prod_{i=1}^2\langle \alpha'_{s_i}, \beta'_{t_i}\rangle_\fg \, ds_1 ds_2 dt_1 dt_2 \\
        & = K_2(\alpha, \beta).
    \end{align*}
    The general proof proceeds in the same manner. By each successive unfolding of the definition of $Q_m$, we recover an additional summand in Equation~\ref{eq:sigkernel_comp}.
\end{proof}

Next, we will consider the discrete formulation of this expression. For simplicity, we will consider discrete time series of the same length, though the following results also hold when the two discrete time series are of different lengths. Suppose $\hat{\alpha}, \hat{\beta}: [T+1] \rightarrow G$, and let $\hat{\alpha}', \hat{\beta}':[T] \rightarrow \fg$ be their corresponding discrete derivatives. Following the notation in Section~\ref{ssec:discrete_ps} for the discrete signature, we define the \style{discrete signature kernel truncated at level $M$} to be
\begin{equation}
    \hat{K}_M(\hat{\alpha}, \hat{\beta}) \coloneqq \langle \hat{S}_M(\hat{\alpha}), \hat{S}_M(\hat{\beta})\rangle_M.
\end{equation}
Notice we are using the discrete signature given in Definition~\ref{def:discrete_ps}. Then, the discrete analogues of Proposition~\ref{prop:sigkernel_comp} and Corollary~\ref{cor:cont_kerneltrick} are as follows.

\begin{proposition}
The discrete signature kernel can be computed as 
\begin{equation}
    \hat{K}_M(\hat{\alpha}, \hat{\beta}) = \sum_{m=0}^M\, \sum_{(s,t) \in \hat{\Delta}^m_T \times \hat{\Delta}^m_T} \,\prod_{i=1}^m \langle \alpha'_{s_i}, \beta'_{t_i} \rangle_\fg.
\end{equation}
\end{proposition}

\begin{corollary}
    Let 
    \begin{align}
        \hat{Q}_1(s,t) &= 1 + \sum_{s' \in [s], \, t' \in [t]} \langle \alpha'_{s'}, \beta'_{t'}\rangle_\fg
    \end{align}
    and recursively define
    \begin{align}
        \hat{Q}_m(s,t) & = 1 + \sum_{s' \in [s], \, t' \in [t]} \hat{Q}_{m-1}(s',t')\langle \alpha'_{s'}, \beta'_{t'}\rangle_\fg.
    \end{align}
    The discrete signature kernel can be computed as
    \begin{align}
        \hat{K}_M(\hat{\alpha},\hat{\beta}) = \hat{Q}_M(T,T).
    \end{align}
\end{corollary}
The proofs for these discrete formulas proceed in exactly the same manner as their continuous counterparts. This final recursive formula for the kernel provides an efficient computation of the discrete signature. We set some notation before writing down the algorithm. Suppose $A, B$ are $T \times T$ arrays. We will use the notation $A[i, j]$ to denote elements of the array, and we suppose that the arrays are $1$-indexed. The notation for the pseudocode is explained in Appendix~\ref{apx:implementation}.

For the algorithm, suppose we have discrete time series $\alpha, \beta:[T+1] \rightarrow G$, and the corresponding derivatives $\alpha', \beta' : [T] \rightarrow \fg$ are already computed, as per Appendix~\ref{apx:implementation}. We assume that the Lie group $G$ is $N$ dimensional. In the pseudocode, we let $\texttt{a}, \texttt{b}$ be the discrete derivatives $\alpha', \beta'$ respectively.

\begin{algorithm}[H]
\SetKwInOut{Input}{Input}\SetKwInOut{Output}{Output}
\SetAlgoLined
\Input{$\texttt{a,b}$ ($T \times N$ arrays): Two paths as discrete derivatives\\
    $\texttt{M}$ (Int): The truncation level}
\Output{\texttt{R} (Float): The kernel value $K_M(\alpha,\beta)$.}
\nosemic Compute the Gram matrix of the derivatives\;
\pushline \dosemic \nonl $\texttt{K} \leftarrow \texttt{ab}^\intercal$\;
\popline Initialize $(T \times T)$ arrays $\texttt{A}$ and $\texttt{Q}$ \;
\nosemic Initialize the first step of the recursion\;
\pushline \dosemic \nonl $\texttt{A} \leftarrow \texttt{K}$ \;
 \popline \For{$\texttt{m=2..M}$}{
  $\texttt{Q} \leftarrow \texttt{1+ A[}\boxplus, \boxplus\texttt{]}$ \;
  $\texttt{A} \leftarrow \texttt{K} \cdot \texttt{Q}$ \;
 }
 $\texttt{R} \leftarrow \texttt{1 + A[}\Sigma, \Sigma\texttt{]}$\;
 Return $\texttt{R}$
 \caption{Discretized Signature Kernel}
\end{algorithm}

As discussed by~\cite{kiraly_kernels_2019}, the runtime of this algorithm is $O(T^2 \cdot M)$. Now, consider the naive computation of the kernel where we first compute the truncated signatures of $\hat{\alpha}$ and $\hat{\beta}$, and then compute the inner product. From the analysis in Appendix~\ref{apx:implementation}, a signature computation requires $O(TN^M)$ operations. The leading-order term in the computation of the inner product is computing the inner product of the respective level $M$ tensors. This requires $N^M$ operations. Thus, the complexity of the naive computation is $O(TN^M)$.

This suggests that if the length $T$ of the time series is large and the truncation level is small, then a naive computation may be more efficient. However, if we wish to compute the kernel at a high truncation level, the recursive algorithm provided here scales significantly better. Furthermore, several variants of this algorithm for Euclidean path signatures are considered in~\cite{kiraly_kernels_2019} such as by incorporating low-rank approximations. These algorithms can similarly be extended to the setting of Lie groups by applying them to the discrete derivatives.


\section{Experiments}
\label{sec:experiments}

In this section, we provide two detailed experiments to demonstrate the universal and characteristic properties of the path signature. First, we consider the human action recognition problem from computer vision, using a Lie group representation of the data. We find that the path signature method is simple to implement, achieves comparable classification performance to shallow learning methods, and provides an interpretable feature set. Second, we perform a kernel two-sample hypothesis test aiming to distinguish between two different random walks on $SO(3)$. Here, we find that the path signature for $SO(3)$ significantly outperforms the same hypothesis test done using the Euclidean representation of $SO(3)$.


\subsection{Human action recognition}
\label{ssec:g3d}
In this subsection, we aim to utilize the universal property of the signature kernel to study a classification problem. Namely, we will be working in the domain of \style{human action recognition}, in which several recent works, such as~\cite{huang_deep_2017},~\cite{rhif_action_2018} and~\cite{li_skeleton-based_2019}, have used Lie group representations and deep learning to achieve state of the art results. Euclidean path signatures have also been used to study the real valued and joint-based representation of this problem in~\cite{yang_developing_2019}, and preliminary experiments using Lie group path signatures have been performed in~\cite{celledoni_signatures_2019}, though detailed quantative comparisons with larger datasets were not provided.

Our goal in this section is to demonstrate the utility, simplicity and interpretability of the Lie group path signature using the G3D-Gaming dataset from~\cite{bloom_g3d_2012}. Thus, our focus will be on establishing baseline results using support vector machines and random forests, and our main comparison will be with the ``shallow'' learning methods described in~\cite{vemulapalli_human_2014} and~\cite{vemulapalli_rolling_2016}. We achieve comparable classification results using a random forest algorithm and the second level signature as a feature set. A major advantage of the second level signature is that we can interpret this feature set in terms of the underlying movement, and we provide a brief discussion of this analysis. 

\subsubsection{Lie group representation}

In this experiment, we aim to classify actions based on human skeletal motion data, which we represent as a $SO(3)^k$ valued time series. This Lie group representation of human motion was introduced in~\cite{vemulapalli_human_2014} and~\cite{vemulapalli_rolling_2016}, which describes a pose as an element of $SO(3)^k$. The relative rotation of two body parts is an element of $SO(3)$, and by recording the relative rotation of $k$ pairs of body parts, we can describe the full pose of a human. 

A skeleton $S= (V, E)$ can be described using a set of joints $V = \{v^i\}_{i=1}^{n_V}$, where $v^i \in \R^3$, and a set of body parts $E = \{e^i\}_{i=1}^{n_E}$, where each body part is a function $e^i : \{0,1\} \rightarrow V$. We may consider $e^i(0)$ and $e^i(1)$ to respectively be the start and end point of the body part $e^i$. Additionally, we will denote the unit vector describing the direction from the start to the end of the body part by
\begin{align*}
    \hat{e}^i = \frac{e^i(1) - e^i(0)}{\|e^i(1) - e^i(0)\|}.
\end{align*}
We will consider all pairs of body parts $(e^i, e^j)$ that share a joint, such that $e^i(c_1) = e^j(c_2)$ for some $c_1, c_2 \in \{0, 1\}$, so $k$ is the number of adjacent pairs of body parts. To obtain the rotation matrix for a chosen pair $(e^i, e^j)$, we rotate the global coordinate system (with minimum rotation) such that $\hat{e}^i$ is the $x$-axis. Then, the rotation matrix $R_{i,j} \in SO(3)$ is the minimum rotation from $\hat{e}^i$ to $\hat{e}^j$ in this coordinate system. By repeating this for all adjacent pairs, we obtain an element of $SO(3)^k$, and further repeating this for all time steps, we can represent this motion as a time series in $SO(3)^k$. 

\subsubsection{Data and preprocessing}
The dataset we use is the G3D-Gaming dataset from~\cite{bloom_g3d_2012}. This contains 663 sequences of 20 different gaming motions performed by 10 subjects, and each subject performed each action at least two times. The 3D locations of 20 joints are provided for every frame. However, the number of frames for each recorded action varies widely from 3 frames to 330 frames. We use the code provided in~\cite{vemulapalli_rolling_2016} to generate the $SO(3)$ representation. 

\begin{figure}[!htbp]
\centering
	\includegraphics[width=0.4\textwidth]{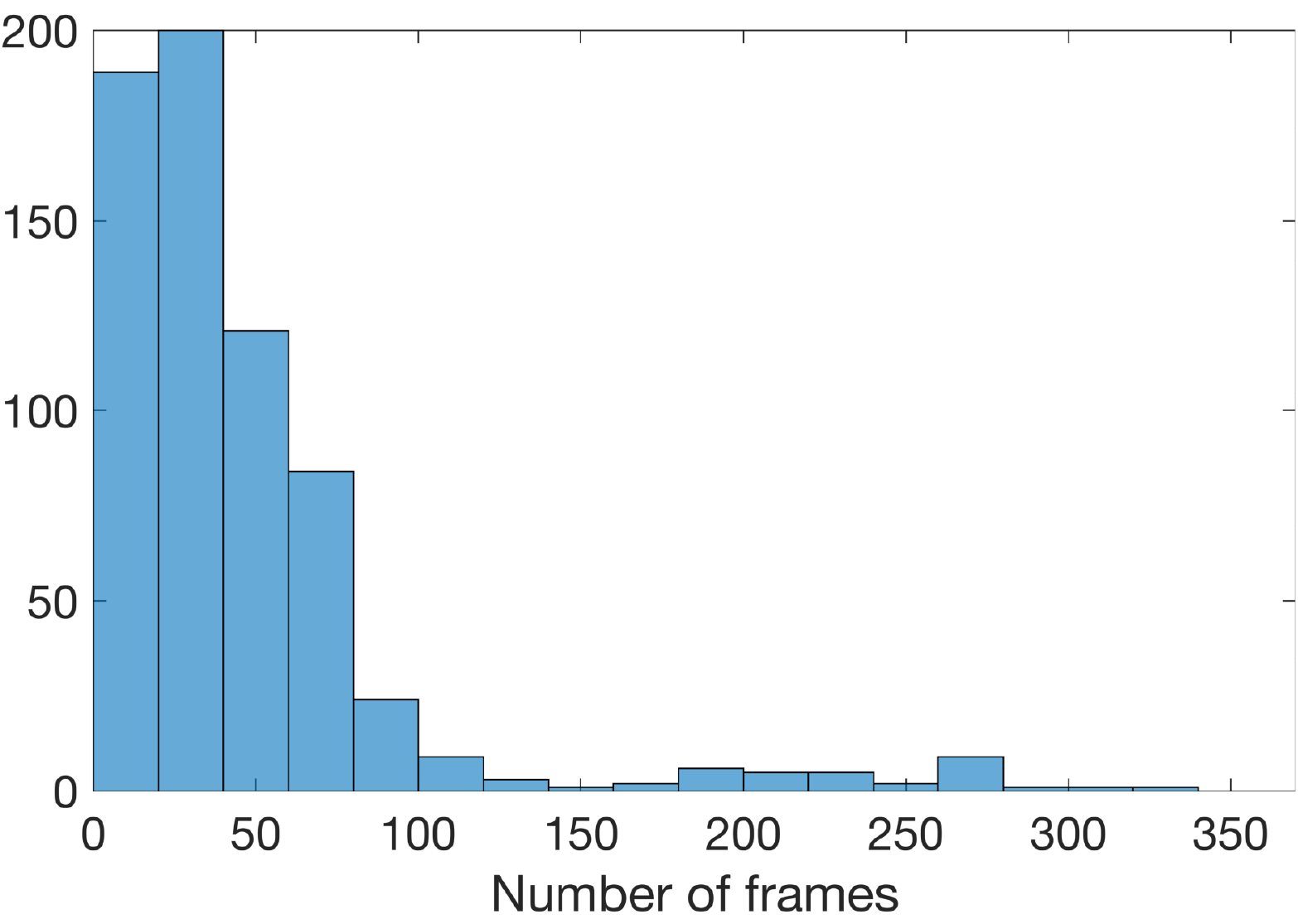}
	\caption{Histogram of number of frames per recorded action in G3D dataset.}
\end{figure}

Note that in~\cite{vemulapalli_human_2014} and~\cite{vemulapalli_rolling_2016}, all pairs of body parts are used in the Lie group representation, which results in $k=342$ pairs for this dataset. In contrast, we only use all adjacent pairs of body parts, which we call the \style{primary pairs}, resulting in $k=18$ pairs for this dataset. We use significantly less data because the path signatures take into account the relationships between all input pairs, so we have information regarding non-primary pairs through the higher order signature terms. The numbering of the primary pairs are given in the following figure.
\clearpage
\begin{figure}[!htbp]
\centering
	\includegraphics[width=0.4\textwidth]{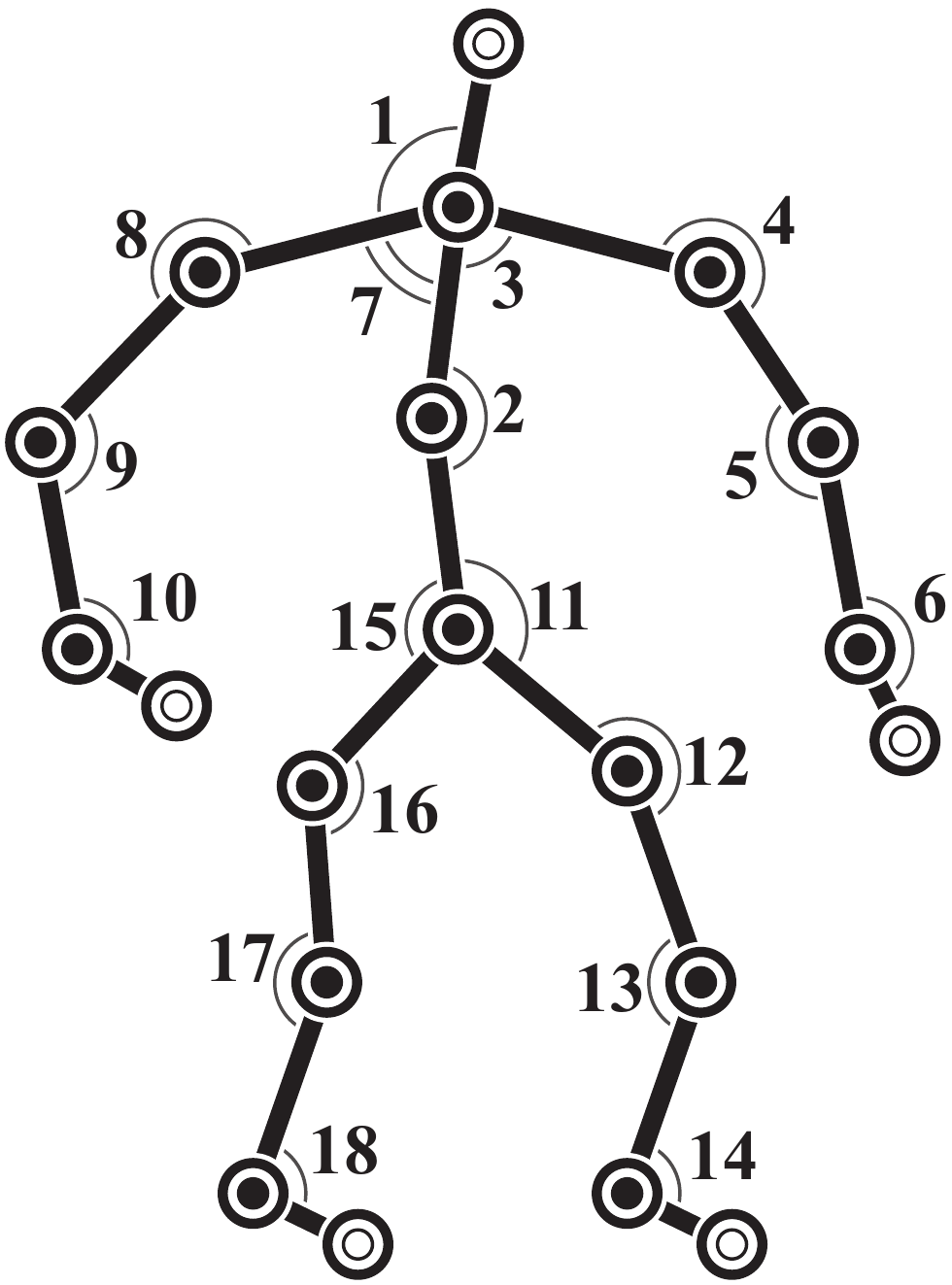}
	\caption{Numbering of the primary pairs of body parts.}
	\label{fig:primary_pairs}
\end{figure}

We categorize these 18 primary pairs into 5 classes.
\begin{table}[h!]
\centering
\begin{tabular}{|c|c|} 
 \hline
    Body & 1-2 \\ \hline
    L Arm & 3-6 \\ \hline
    R Arm & 7-10 \\ \hline
    L Leg & 11-14 \\ \hline
    R Leg & 15-18 \\ \hline
\end{tabular}
\caption{The body regions corresponding to groups of primary pairs.}
\label{tab:classes}
\end{table}

Extensive preprocessing of this data is performed in~\cite{vemulapalli_human_2014} and~\cite{vemulapalli_rolling_2016} in order to deal with several difficulties. During the training stage, the following steps were taken.
\begin{enumerate}
    \item Each time series is resampled via interpolation so that all time series have a fixed length.
    \item A nominal curve in $SO(3)^k$ is generated for each action class.
    \item To handle issues of rate variation and temporal misalignment, dynamic time warping (DTW) is used to warp each time series to its corresponding nominal curve.
    \item A rolling and unwrapping procedure is computed with respect to its nominal curve to obtain a curve in the Lie algebra $\fso(3)^k$. 
    \item (Optional) A Fourier temporal pyramid (FTP) representation of the Lie algebra curve may be computed to further deal with temporal misalignment.
\end{enumerate}
A classifier such as a support vector machine (SVM) is then trained for every action class (one vs. rest). In the testing stage, the data is preprocessed with respect to all nominal curves and the corresponding SVM is used for prediction. This amounts to a large preprocessing cost, especially for test samples, which must be preprocessed with respect to all action classes.

In contrast, we perform minimal preprocessing since we can compare time series with varying numbers of frames using path signatures, and the issues of rate variation and temporal misalignment are handled by reparametrization invariance. 

\subsubsection{Results}
Following~\cite{vemulapalli_human_2014, vemulapalli_rolling_2016}, we use a cross-subject test setting, where we use half of the subjects for training, and the other half for testing. All of the reported classification results are averaged over ten different combinations of the train/test split. We perform the classification using a kernel SVM, as well as a random forest. For the kernel SVM, we report the results using the signature kernel truncated at level 6. We use the Julia implementation of \texttt{scikit-learn}, with the \texttt{SVC} implementation for support vector machines, which uses the one-against-one approach~\citep{knerr_single-layer_1990} for multi-class classifcation. For the random forest, we compute the level 2 signature of the time series and treat it as a feature set. We use the random forest implementation in the Julia \texttt{DecisionTrees} package. We follow the suggested default random forest hyperparameters in~\cite{probst_hyperparameters_2019}, and use 1,000 trees, $\sqrt{n_f}$ features for each tree, where $n_f$ is the total number of input features, a maximum depth of $100$, and use $70\%$ of the data to train each tree. A tensor renormalization is used for all path signature computations as described in Proposition~\ref{prop:tensconstruction}, and using the function
\begin{align*}
    \psi(x) & = \left\{
        \begin{array}{ll}
            x^2 & \textrm{if} \, x \leq \sqrt{M} \\
            M + M^{1+a}(M^{-a} - x^{-a})/a & \textrm{if} \, x > \sqrt{M}
        \end{array}
    \right.
\end{align*}
where $M = 4$ and $a = 1$ as in~\cite{chevyrev_signature_2018}.

We apply each algorithm to several embeddings of the discrete time series. In addition to using the raw time series, we also use the identity initialialized (IdInit) embedding to remove translation invariance, and the sliding window (SWin) embedding with 1 to 6 lags. 

\begin{table}[h!]
\centering
\begin{tabular}{ |c|c|c| } 
 \hline
  & Random Forest & SVM \\ \hline \hline
 Raw & 69.99\% & 65.20\% \\ \hline
 IdInit & 80.48\% & 77.86\% \\ \hline
 SW, 1 lag & 80.32\% & 79.24\% \\ \hline
 SW, 2 lags & 82.71\% & 80.27\% \\ \hline
 SW, 3 lags & 83.77\% & 80.84\% \\ \hline
 SW, 4 lags & 84.47\% & 81.08\% \\ \hline
 SW, 5 lags & 84.86\% & 81.44\%\\ \hline
 SW, 6 lags & 85.56\% & 81.44\% \\ \hline
\end{tabular}
\caption{Classification results using a random forest and SVM for different embeddings.}
\label{tab:g3d_results}
\end{table}

Based on our results, a random forest trained using level 2 signatures outperforms an SVM trained using level 6 signatures for all embeddings. This may be due to the fact that random forests are in general better suited for multi-class classification tasks, since the SVM approach is to split the multi-class problem into ${20 \choose 2}$ binary classification tasks.

We note by using the random forest classifier, and introducing lags, we are able to achieve results which are comparable to the 87.95\% accuracy of~\cite{vemulapalli_rolling_2016}, given the fact that we are using significantly less input data, minimal preprocessing, and default hyperparameters on the random forest. \medskip

\begin{remark}
\label{rem:translation_breaking}
    In Section~\ref{sssec:sliding_window}, we mentioned that one possible explanation for the strong empirical performance of the sliding window embedding is the breaking of translation invariance. In these results, we can isolate the effect of breaking translation invariance by using the IdInit embedding. We see that for both the random forest and the SVM, the performance improves significantly after using the IdInit embedding when compared to the raw time series. The sliding window embedding with 1 lag also significantly improves the performance when compared to the raw time series, but we see that the increase is comparable to that of the IdInit embedding. When we increase the number of lags, the path signature is able to capture more information by integrating with respect to past values, and thus the performance continues to improve. \cite{fermanian_embedding_2019} suggested the problem of explaining why the sliding window embedding performs well in practice. This empirical evidence indicates that one such reason is the breaking of the translation invariance of the path signature.
\end{remark}

\subsubsection{Interpretation of the second level signature matrix}
In addition to comparable classification accuracy, the second level signature matrix affords an interpretable feature set. This is one of the highlights of this method, as none of the cited classification methods are easily interpretable. Namely, in both the narrow and deep learning approaches, the input to the classification algorithm is the stacked (and in some cases preprocessed) Lie group valued time series, resulting in an extremely high dimensional vector. In this section, we will discuss the interpretation of the raw second level signature matrix, which was used for the random forest classification. 

The second level path signature of a path $\gamma \in PG$, where $G$ is $N$-dimensional, can be viewed as a $(N \times N)$ matrix where the $(i,j)$ entry is simply $S^{(i,j)}(\gamma)$. For these applications, our Lie group is $G = SO(3)^{18}$, which has $N=54$ dimensions. We will consider the absolute value of the averaged second level signature matrices for each action class of the G3D data set. Specifically, because we are focused only on the interpretation in this section, we will use all of the data when computing averages. Also, we take the absolute value to simplify the interpretation of these matrices. We begin by considering a single action class (walk) in order to describe how to read the matrix.

\begin{figure}[!htbp]
\centering
	\includegraphics[width=0.5\textwidth]{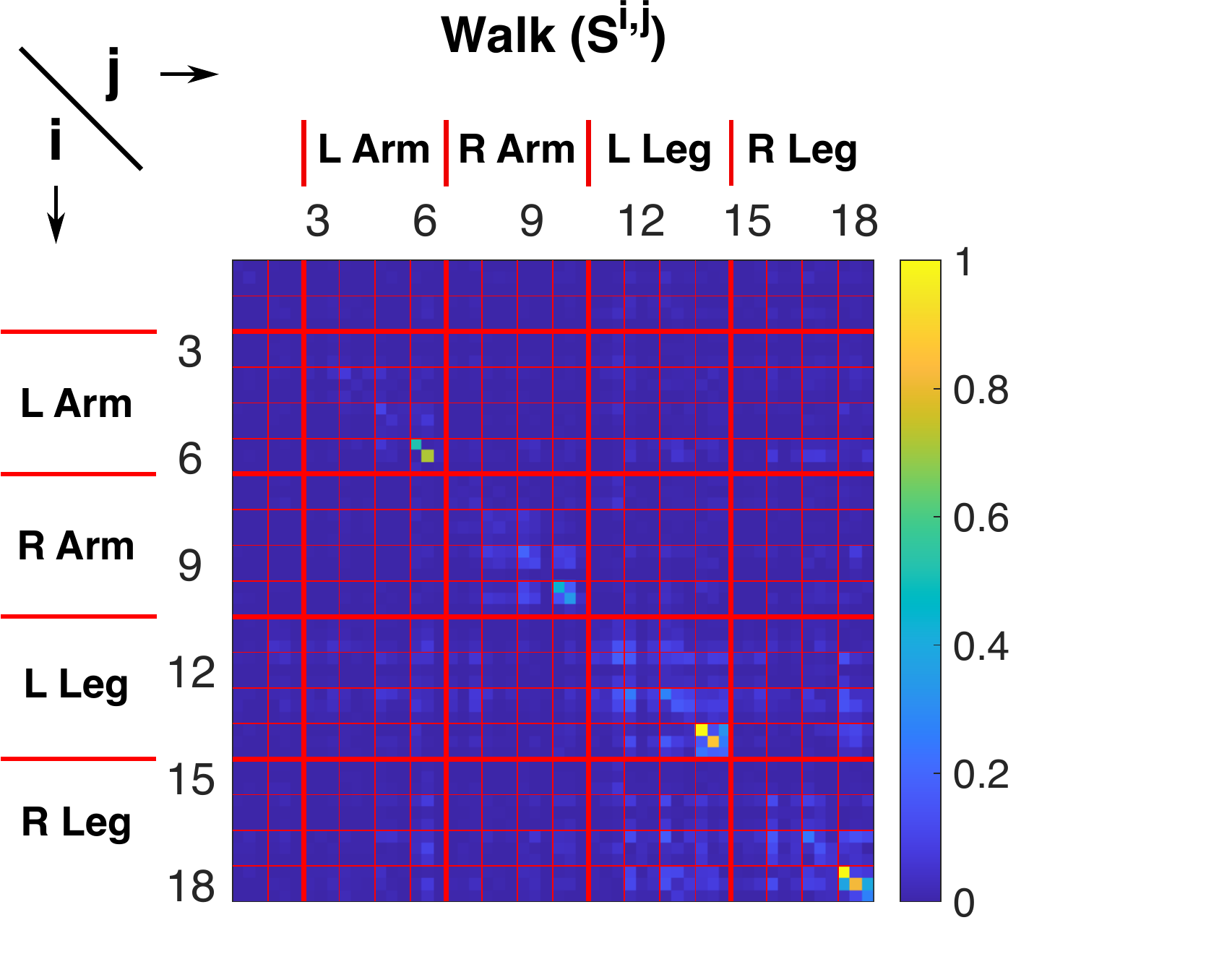}
	\caption{Averaged absolute second level signature matrix for the action class ``walk.''}
	\label{fig:walk_example}
\end{figure}
Here we have visualized the $(54 \times 54)$ matrix as a heatmap, where the entries are arranged in accordance with standard matrix notation. The thin red grid lines separate the matrix into $(3 \times 3)$ blocks, where each triplet of entries corresponds to a single primary pair. The numerical labels on the axes enumerate the blocks (or primary pairs), which correspond to the primary pairs as shown in Figure~\ref{fig:primary_pairs}. The thicker red grid lines are used to separate the different categories of primary pairs, as outlined in Table~\ref{tab:classes}, and labelled in the figure. The diagonal entries of the matrix measure the square of the cumulative rotation of a given basis direction for a given primary pair. The off diagonal entries measure the positive or negative influence of the basis direction $i$ on the basis direction $j$, as defined in Section~\ref{ssec:leadlag}. For a closer look at this example, we isolate the blocks corresponding to the left and right legs, and reduce the color threshold.

\begin{figure}[!htbp]
\centering
	\includegraphics[width=0.45\textwidth]{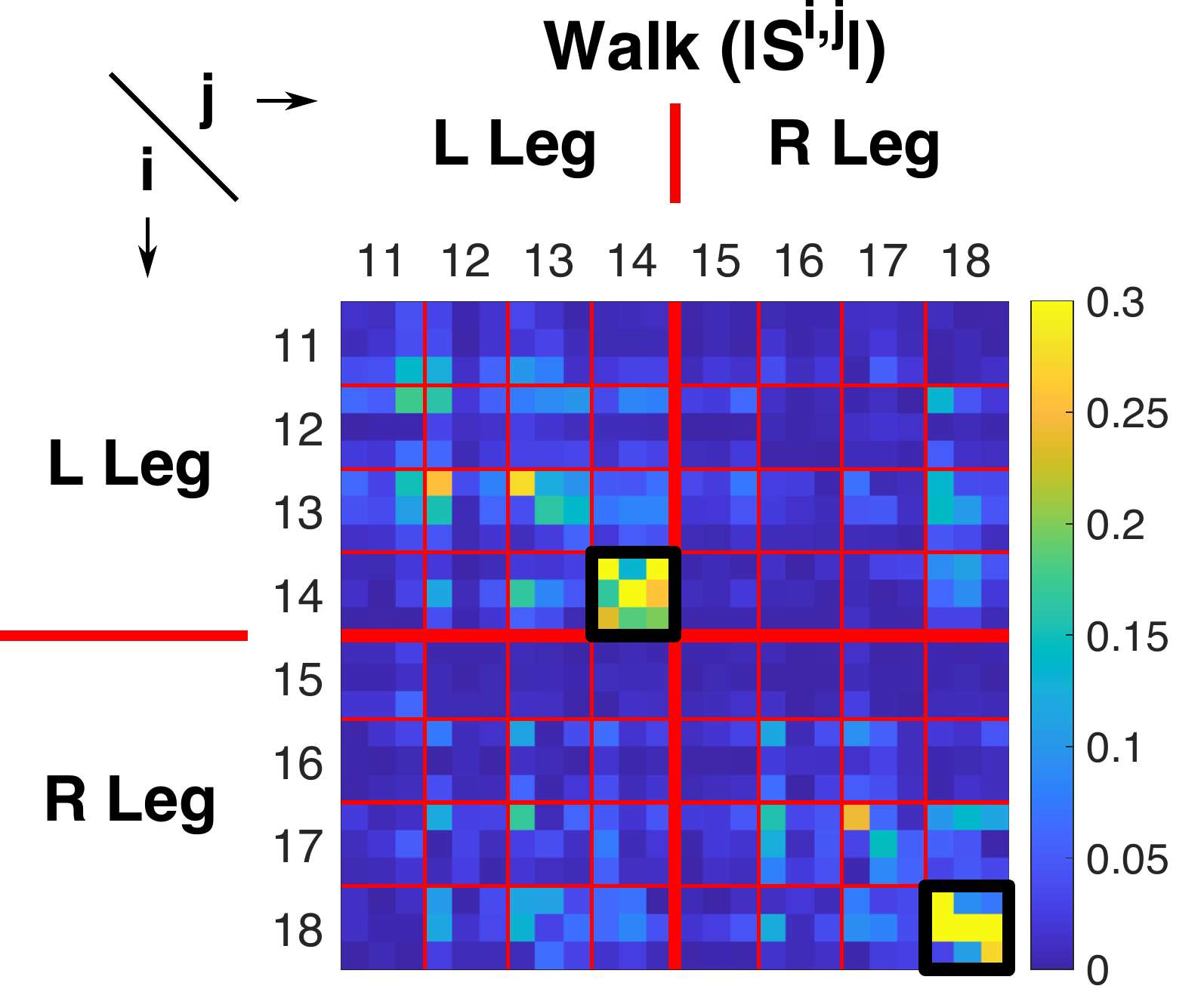}
	\caption{Averaged absolute second level signature matrix for the action class ``walk.''}
\end{figure}

The blocks with the largest magnitude are the diagonal blocks 14 and 18, which correspond to the joints of the left and right foot (see Figure~\ref{fig:primary_pairs}. In addition, many of the off-diagonal entries are nonzero, which corresponds to the action of walking. While walking, we alternate moving our left and right legs, and the signature matrix measures this as an influence of rotations about joints in one leg on the rotations about joints in the other leg. \medskip

In Figure~\ref{fig:all_actions}, we have plotted the absolute value of the averaged second level signature matrix for all action classes in the data set. We omit the labelling to simplify the figure, but all labelling is the same as Figure~\ref{fig:walk_example}. In particular, the colors range from $0$ to $1$. 

\begin{figure}[!htbp]
\centering
	\includegraphics[width=\textwidth]{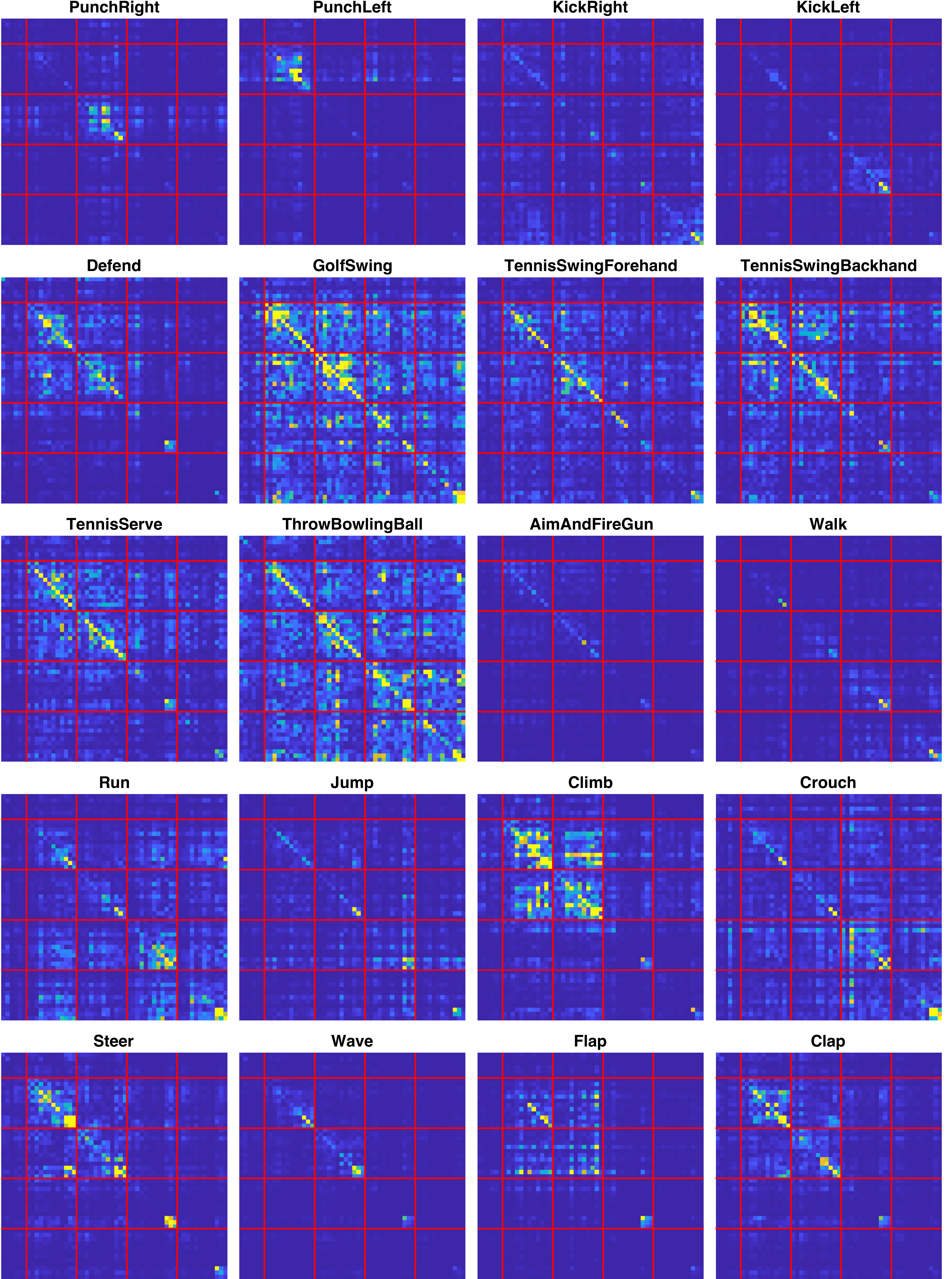}
	\caption{Averaged absolute $S^2$ matrices for all actions in G3D dataset.}
	\label{fig:all_actions}
\end{figure}


\subsection{Random walks on Lie groups}
\label{ssec:randomwalk}
In this experiment, we aim to take advantage of the characteristic property of the signature kernel and perform a kernel two-sample hypothesis testing problem, as described in~\cite{gretton_kernel_2012}. We aim to distinguish between two distinct discrete random walks on $SO(3)$, or in other words, between two distinct probability distribution on $P(SO(3))$. In addition, by using the matrix representation of $SO(3)$, we can also view these as probability distributions on $P\R^9$, and we also perform the hypothesis testing in this context. We find that for this problem, the path signature defined using left-invariant forms on $SO(3)$ significantly outperforms the Euclidean signature. \medskip

\subsubsection{Hypothesis testing with MMD}
Suppose $\mu$ and $\nu$ are two probability measures on $PG$ and we have  i.i.d. samples $\{X^i\}_{i=1}^n$ of $\mu$ and $\{Y^i\}_{i=1}^n$ of $\nu$, each with $n$ samples. Note that we use the superscript to index the separate samples, since the subscript is reserved for the time parameter. We wish to distinguish between the null hypothesis $H_0: \mu = \nu$ against the alternative hypothesis $H_1 : \mu \neq \nu$. This is done by computing the unbiased empirical MMD (Equation~\ref{eq:emp_mmd}) between our two collections of samples, and comparing this against a chosen threshold. A Type I error occurs when the test falsely rejects the null hypothesis, and a Type II error occurs when the test falsely accepts the null hypothesis. The \style{level} $\alpha$ of a test is a chosen upper bound on the probability of a Type I error.

The threshold for the MMD is chosen such that the test is \style{consistent}, meaning the test achieves a level $\alpha$ and a Type II error of zero in the asymptotic large sample limit. One method for choosing the threshold is by using large deviation bounds for the MMD under the assumption that the null hypothesis is true, which are derived in~\cite{gretton_kernel_2012}. However, because such bounds must hold for all measures, the threshold chosen in this way is conservative and may not yield optimal results with a finite number of samples. An alternative method is to use a data-dependent threshold. Our method here will be to generate a null distribution for the MMD by performing a permutation test, and computing the $1-\alpha$ quantile for a level $\alpha$ test. Alternative methods for deriving data-dependent thresholds are explored in~\cite{gretton_fast_2009, gretton_kernel_2012}.

\subsubsection{Random walk generation}
For this experiment, we will consider discrete random walks on $SO(3)$. Each walk is initialized by sampling a point uniformly at random from $SO(3)$. A method to sample uniformly from $SO(3)$ is given in~\cite{diaconis_subgroup_1987}. Next, we sample $T$ i.i.d. points $\{v_i\}_{i=1}^T$ from the unit sphere in the Lie algebra $\fso(3)$. We use a unimodal probability distribution on $S^2$, called the von Mises-Fisher distribution. This is parametrized by the mean direction $x \in S^2$, and the concentration parameter $\kappa \geq 0$, which describes the concentration of the distribution about the mean. 

\begin{figure}[!htbp]
\centering
	\includegraphics[width=0.5\textwidth]{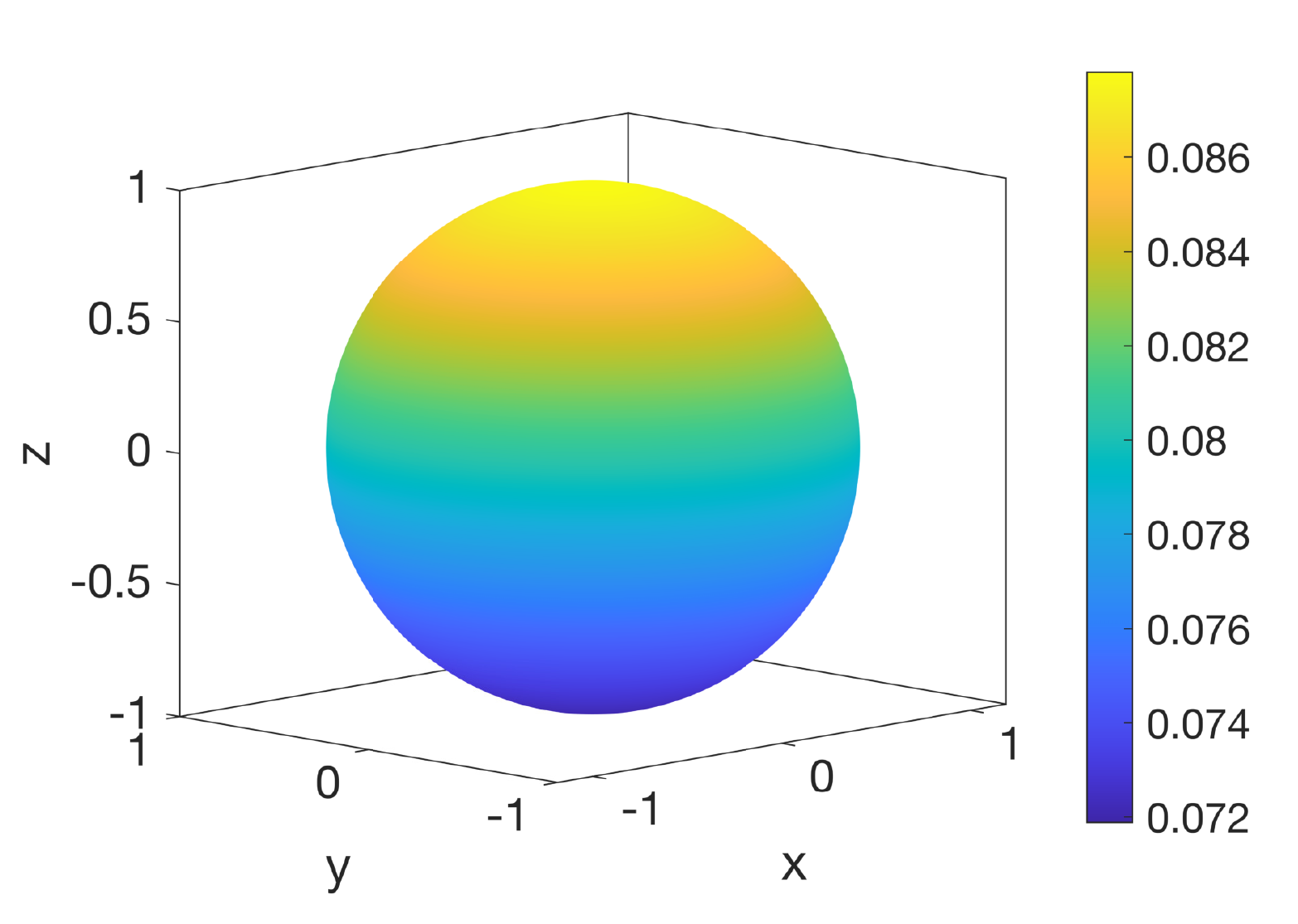}
	\caption{The von-Mises Fisher density on $S^2$ with mean direction $x =(0, 0, 1)$ and $\kappa = 0.1$.}
\end{figure}

We may think of this as a random walk with drift in the mean direction. A numerically stable method of sampling from this distribution is discussed in~\cite{jakob_numerically_2012}. The random walk $\hat{\gamma}: [T+1] \rightarrow SO(3)$ is defined recursively as
\begin{align*}
    \hat{\gamma}_{i+1} = \hat{\gamma}_i \exp(c v_i),
\end{align*}
where $c>0$ is the step size and $\hat{\gamma}_1$ is a randomly sampled point on $SO(3)$. \medskip

\subsubsection{Results}
We perform two classes of tests.
\begin{enumerate}
    \item $H_0$ is false. Here, we attempt to distinguish between distributions with different mean directions. We take $x_\mu = (1, 0, 0)$ and $x_\nu = (0, 1, 0)$. 
    \item $H_0$ is true. Here we take the mean directions of both distributions to be $x_\mu = x_\nu = (1, 0, 0)$.
\end{enumerate}
All distributions will have a concentration parameter of $\kappa = 0.1$. For each test, we will sample $n=50$ random walks from each distribution, and each random walk will have $L=100$ steps. To generate the null distribution of the MMD, we perform a permutation test with 2,000 permutations for a given set of samples. A level of $\alpha=0.05$ is chosen, so we compute the $0.95$ quantile of the null distribution as the threshold used for the MMD. The path signature is truncated at level 4 in the MMD computation. We perform 1,000 tests for each class, and each test is done using both Lie group and Euclidean path signatures. \medskip

The following table provides the error rates (false positive/negative) for the two classes of tests using the two methods.

\begin{table}[h]
\centering
\begin{tabular}{ |c|c|c| } 
 \hline
  & Lie Group & Euclidean \\ \hline \hline
 $H_0$ false & 6.1\% & 83.7\% \\ \hline
 $H_0$ true & 5.7\% & 4.9\% \\ \hline
\end{tabular}
\caption{Error rates for hypothesis testing. Each test was run 1,000 times.}
\end{table}

Additionally, we provide histograms that summarize the test results. The test distributions show the distribution of $\MMD_u$ over the 1,000 independent trials. The null distribution shown is generated from a permutation test for a single trial. The red line shows the $0.95$ quantile, and represents the threshold for that trial. The histograms for $H_0$ being false are shown in Figure~\ref{fig:hist_h0false} and the histograms for $H_0$ being true are shown in Figure~\ref{fig:hist_h0true}.

We find that the Lie group path signatures significantly outperform Euclidean path signatures. This is due to the fact that the Euclidean representation of the data is ill-suited for this problem. We are aiming to detect a slight drift in the direction of the rotation, which is a translation invariant feature in $SO(3)$. However, this is not a translation invariant feature in the Euclidean representation of the problem, so the effect of the drift is confounded in the Euclidean path signature.

\clearpage

\begin{figure}[!htbp]
\centering
	\subfigure{\includegraphics[width=0.36\textwidth]{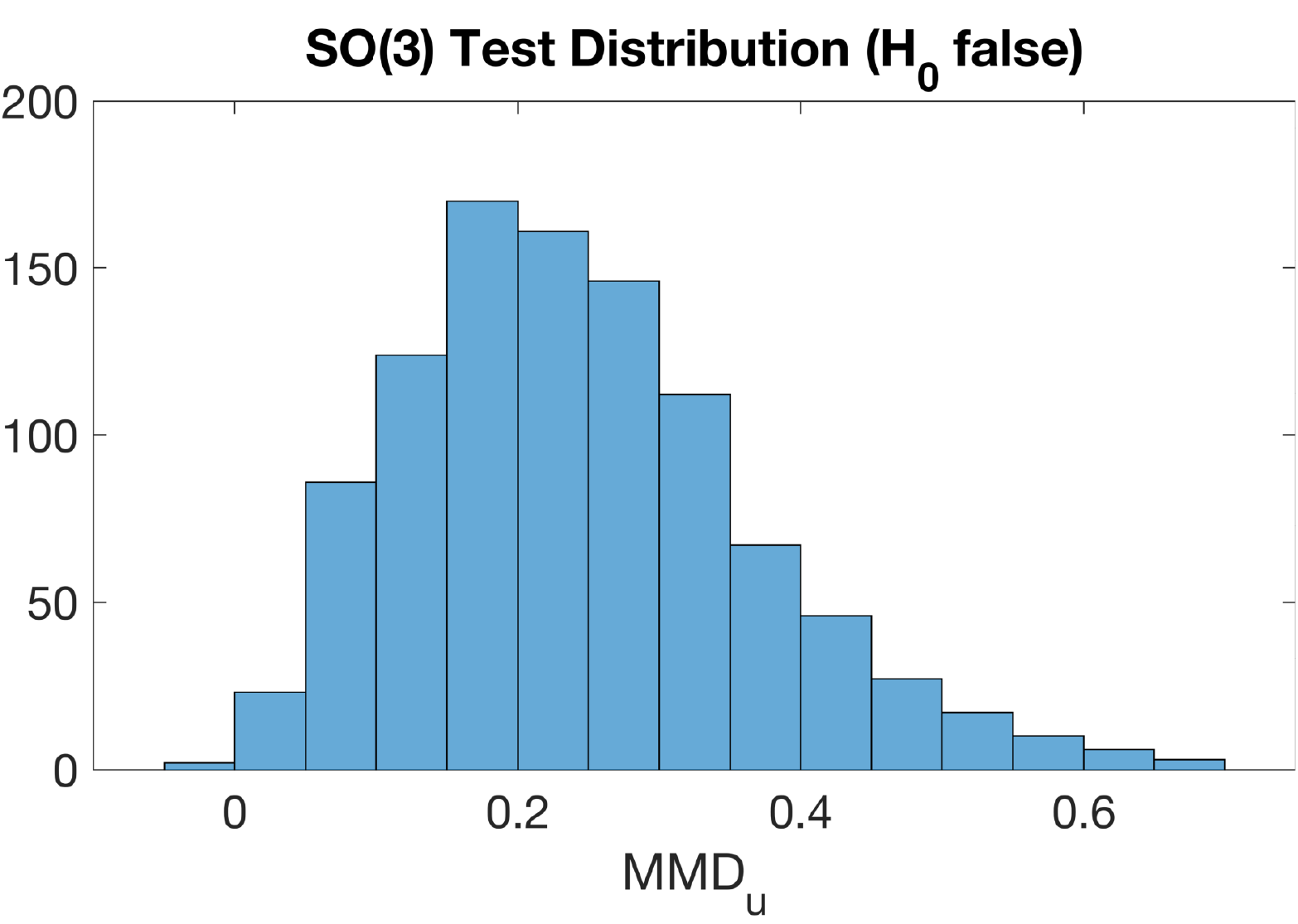}} \hspace{50pt}
	\subfigure{\includegraphics[width=0.36\textwidth]{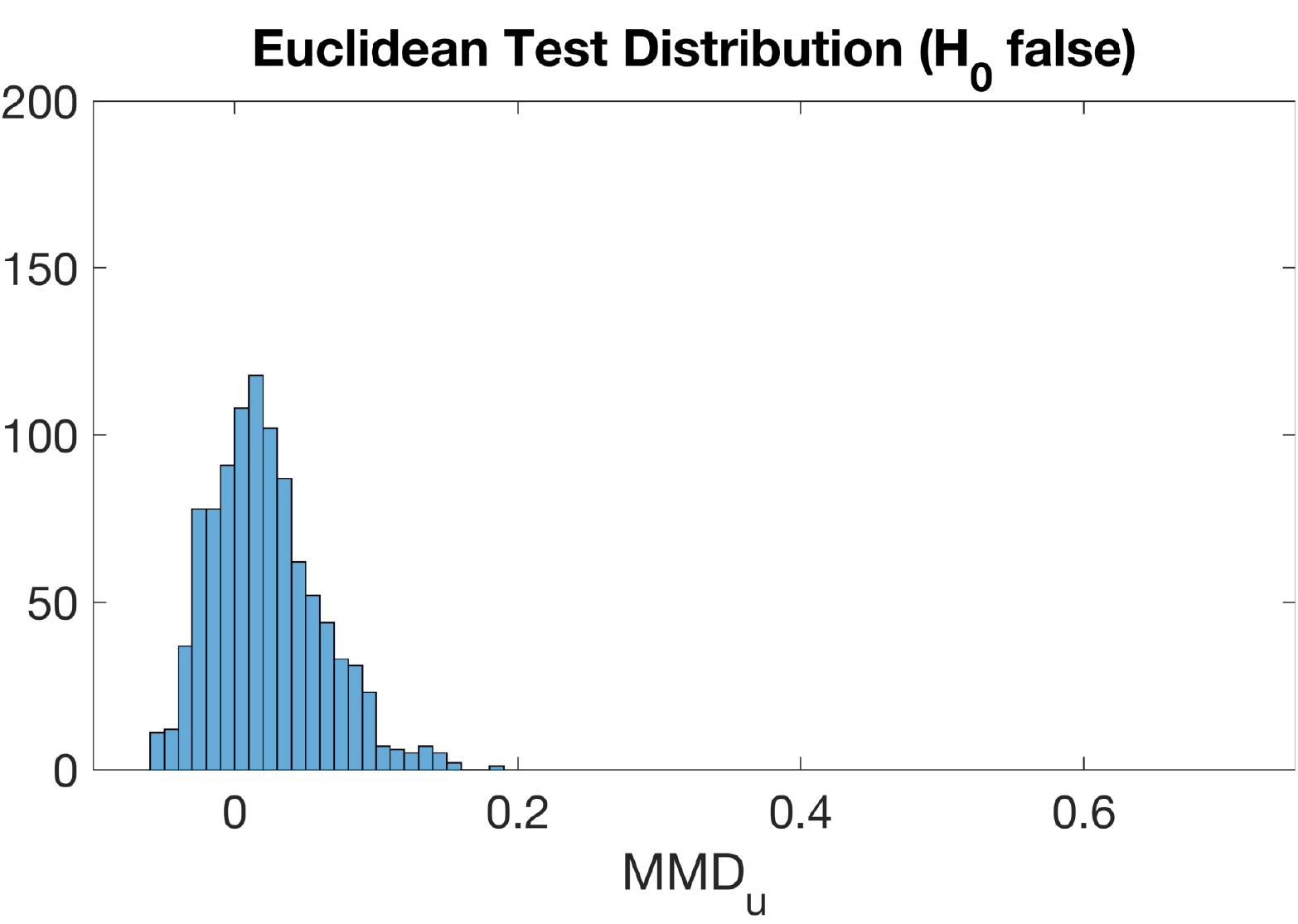}} 
\end{figure}
\vspace{-30pt}
\begin{figure}[!htbp]
\centering
	\subfigure{\includegraphics[width=0.36\textwidth]{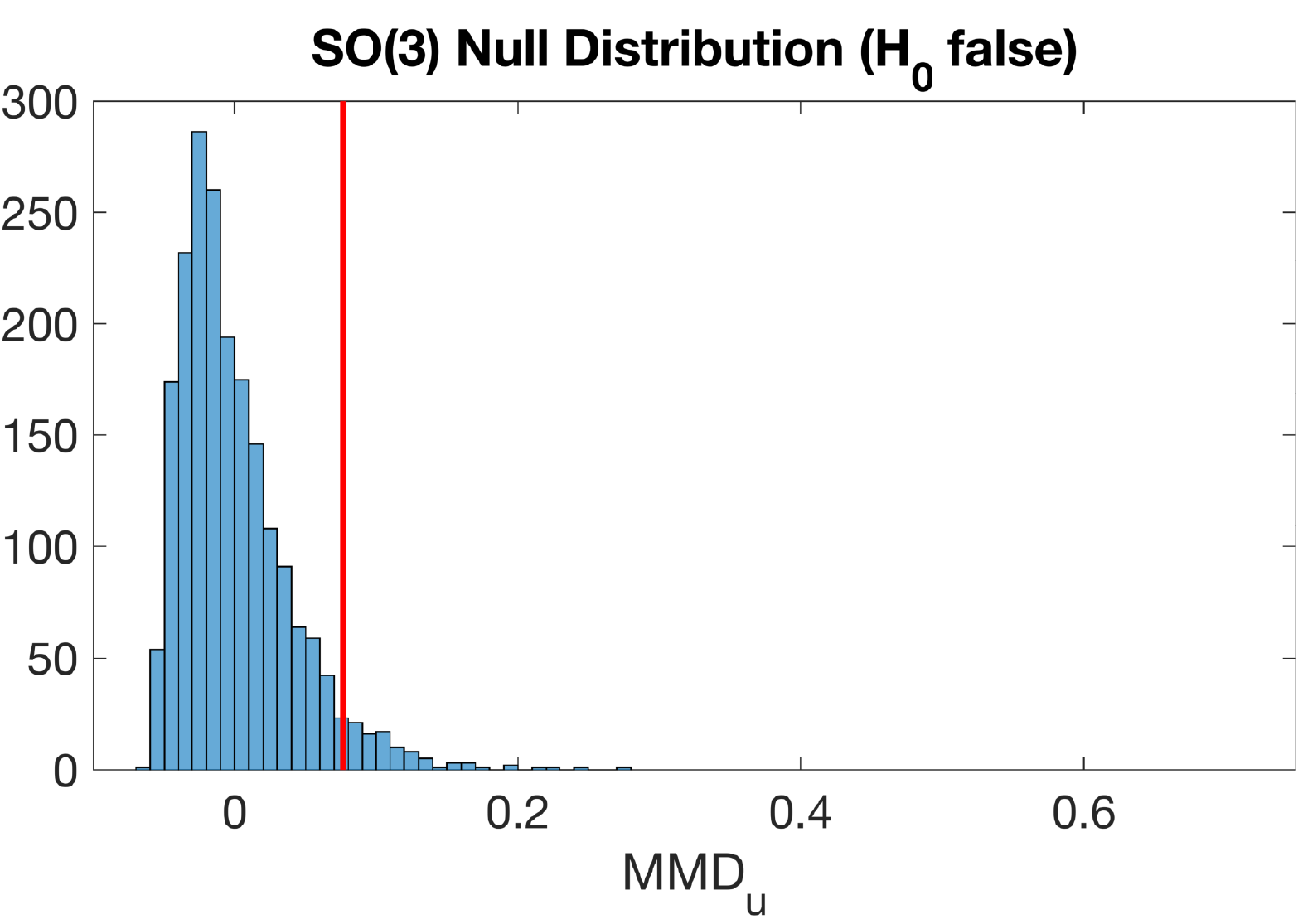}} \hspace{50pt}
	\subfigure{\includegraphics[width=0.36\textwidth]{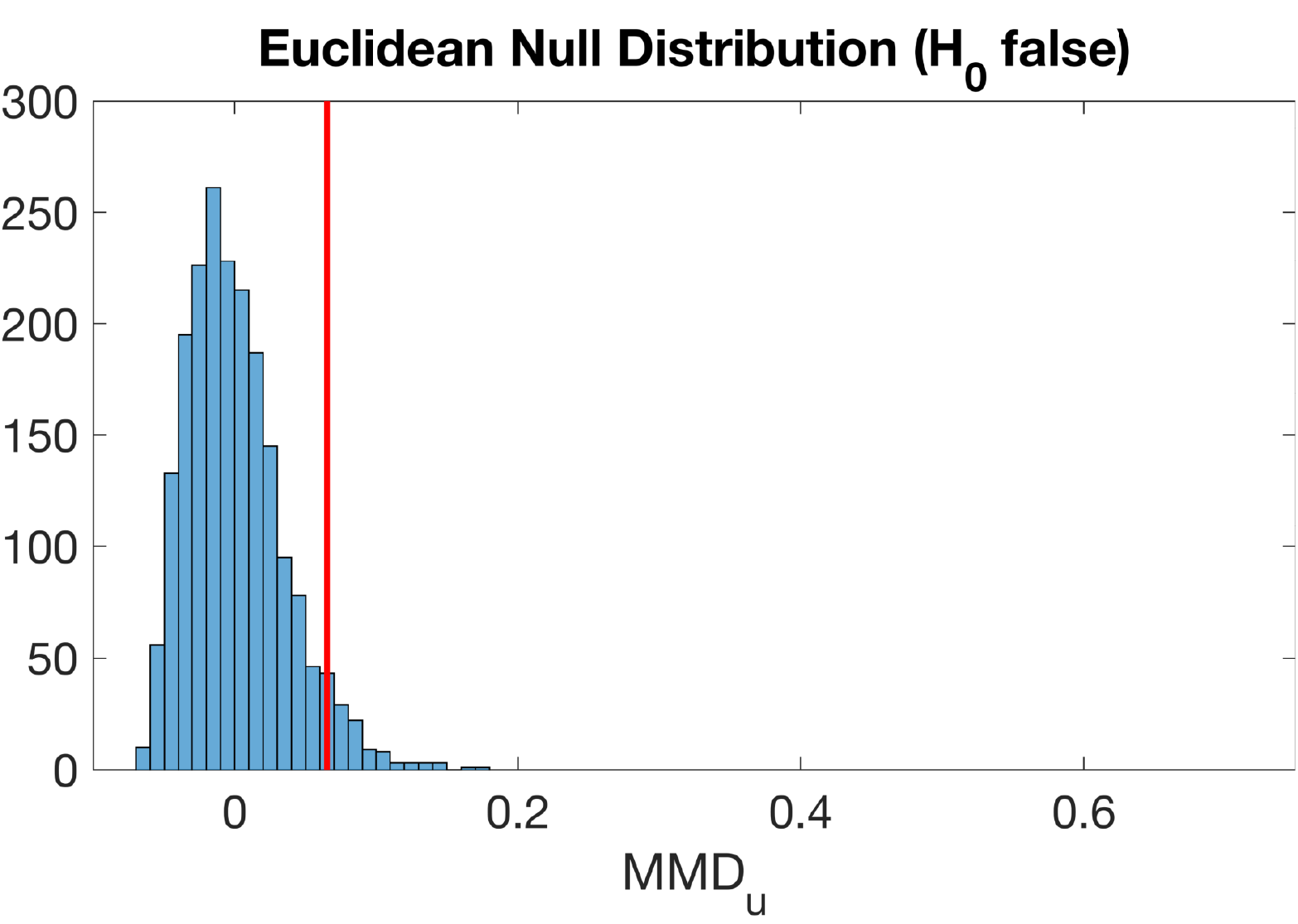}}
	\caption{Test (top) and null (bottom) distributions of $\MMD_u$ when $H_0$ is false.}
	\label{fig:hist_h0false}
\end{figure}
\vspace{-15pt}
\begin{figure}[!htbp]
\centering
	\subfigure{\includegraphics[width=0.36\textwidth]{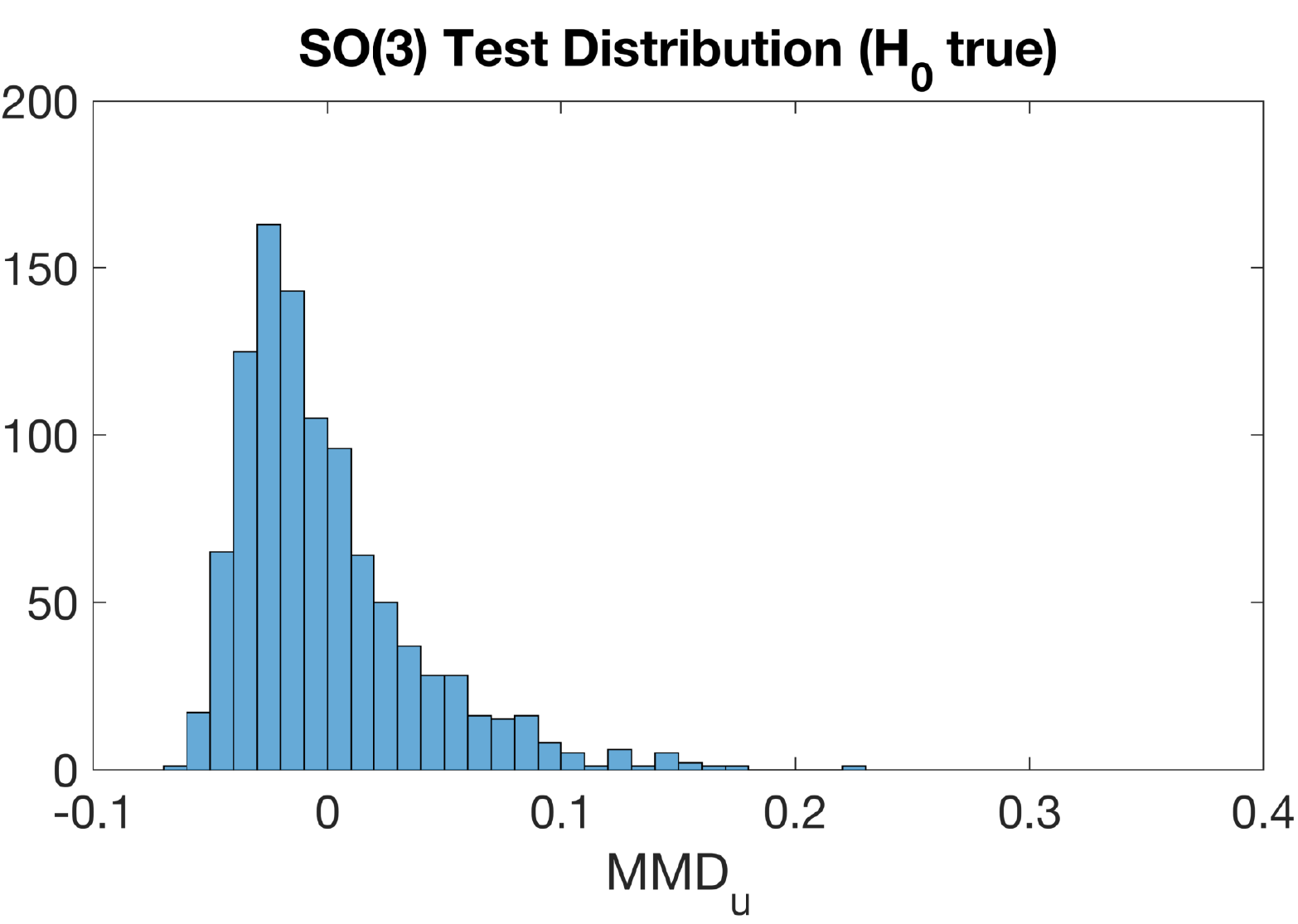}} \hspace{50pt}
	\subfigure{\includegraphics[width=0.36\textwidth]{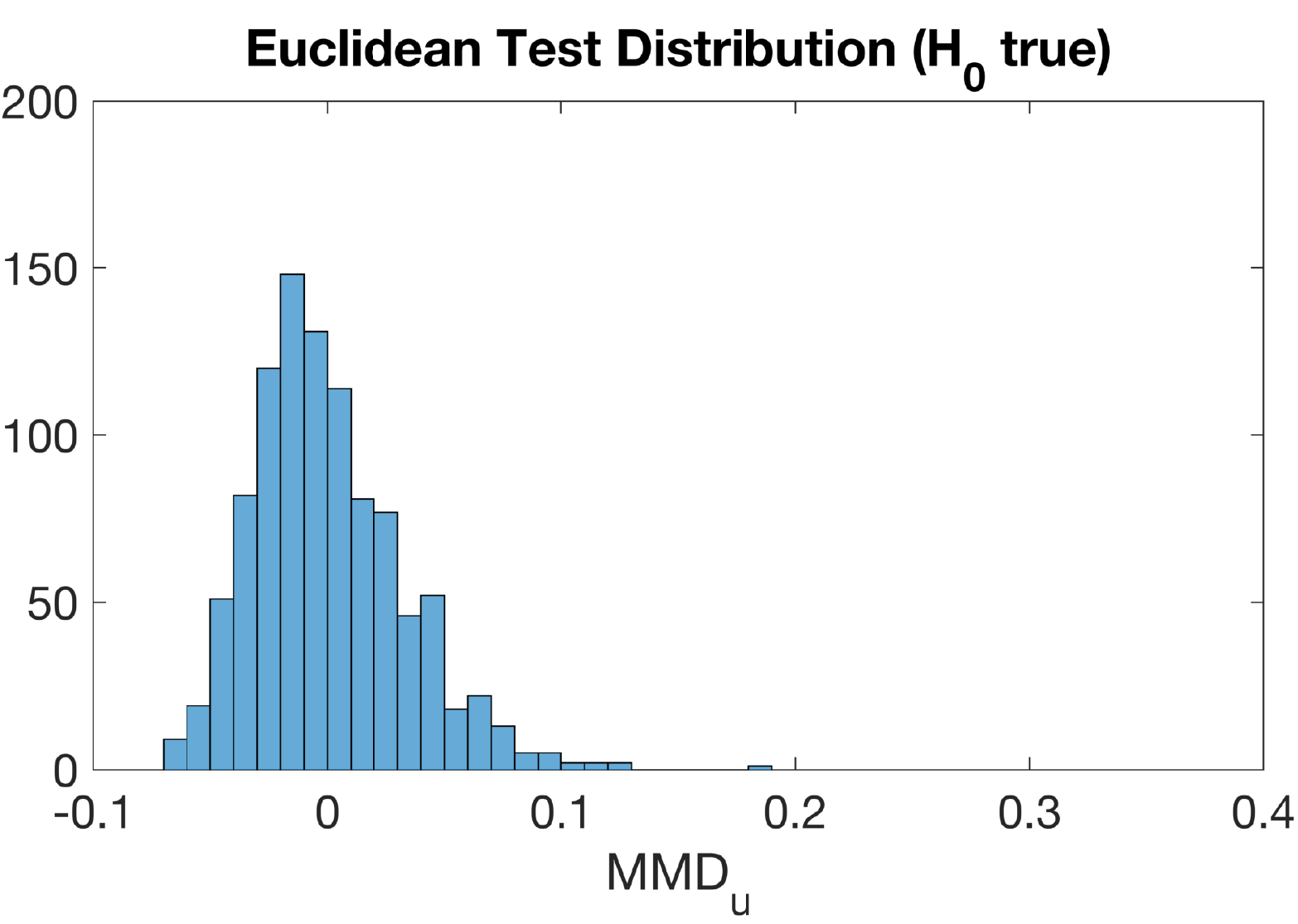}}
\end{figure}
\vspace{-30pt}
\begin{figure}[!htbp]
\centering
	\subfigure{\includegraphics[width=0.36\textwidth]{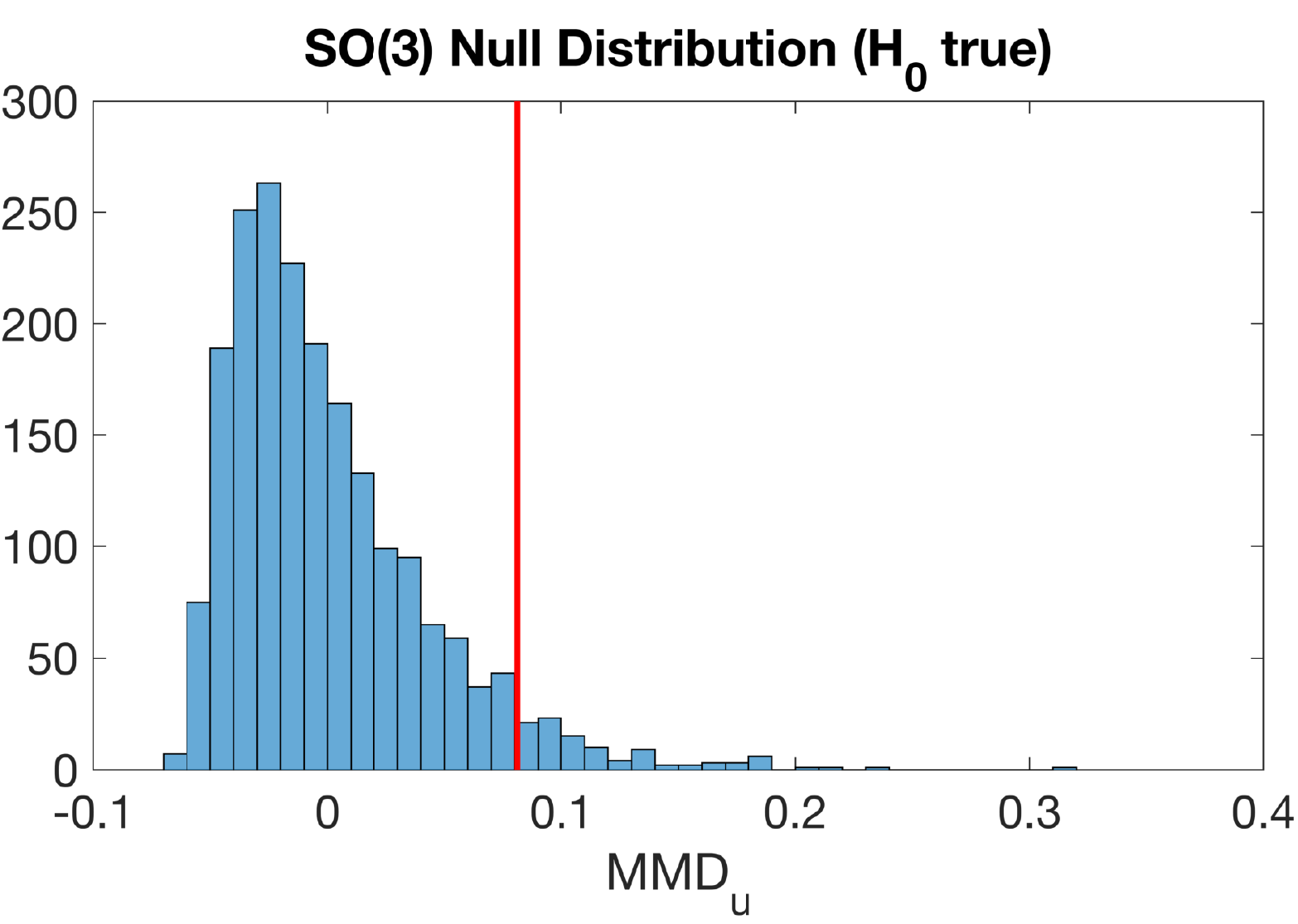}} \hspace{50pt}
	\subfigure{\includegraphics[width=0.36\textwidth]{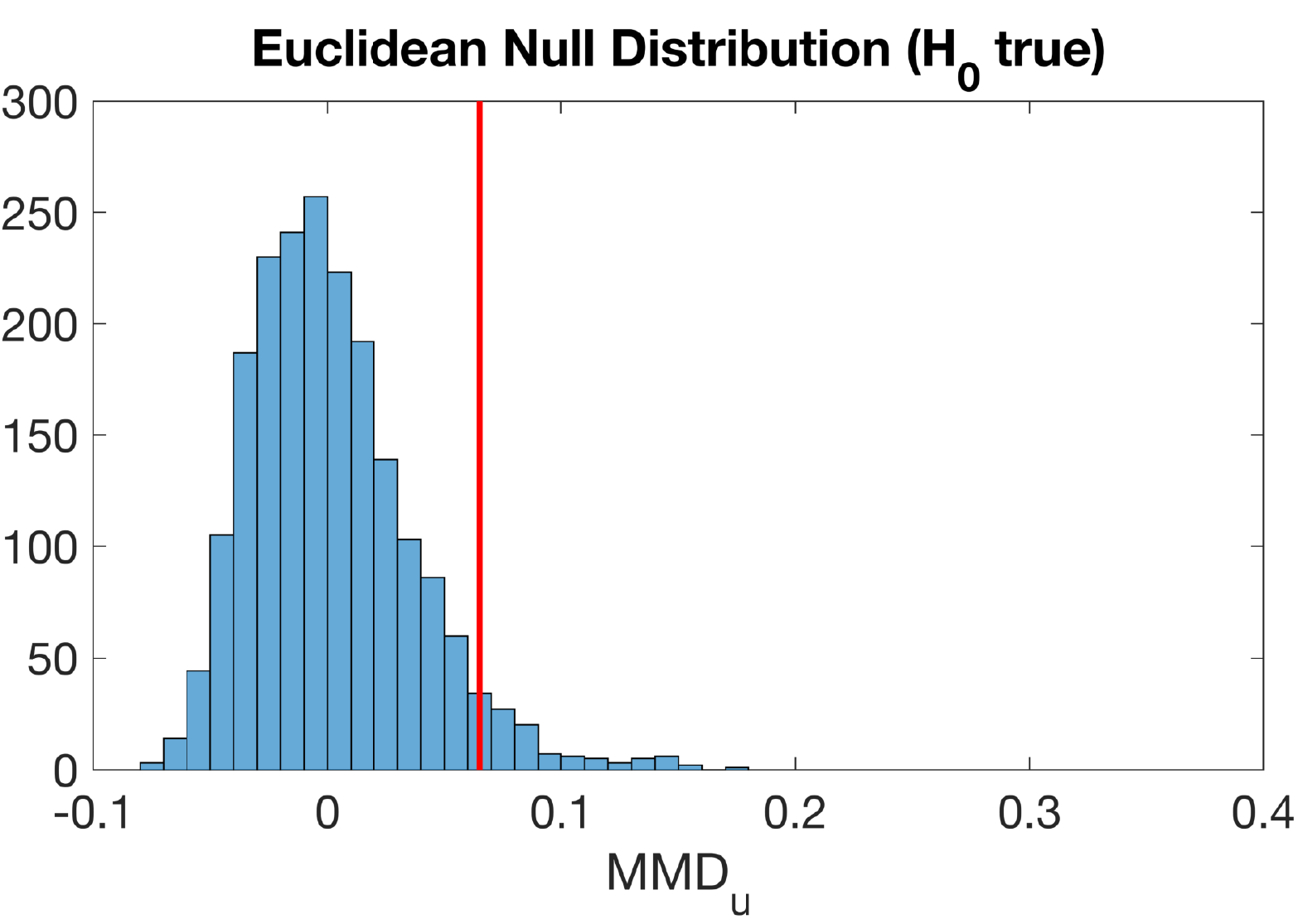}}
	\caption{Test (top) and null (bottom) distributions of $\MMD_u$ when $H_0$ is true.}
	\label{fig:hist_h0true}
\end{figure}

\clearpage


\section{Conclusion}

We have defined path signature for Lie group valued time series, and studied several of its properties, the main result being the universal and characteristic properties of the signature kernel. By defining signature using only the derivative of the path, computational techniques from Euclidean valued paths to Lie group valued paths can be exported cleanly. Our theory is validated using two detailed experiments highlighting both the universal and characteristic properties, showing that the path signature has strong empirical performance, while providing an interpretable feature set which can be used to better understand the underlying phenomena.

Lie group valued data is ubiquitous; however, in previous studies analyzing such data, the analysis pipeline can be complicated due to the ostensible complexities when dealing with Lie groups (as described in Section~\ref{ssec:g3d}). Our derivations show that Lie group valued data can be treated in a manner nearly identical to standard Euclidean valued data. \medskip

The work in this paper provides the foundations for further studying properties and applications of path signatures for Lie groups and more. We highlight two directions for possible future research.

\begin{enumerate}
    \item Although the path signature is a complete characterization of paths in Euclidean space, the inversion problem of recovering a path given its path signature is difficult~\citep{lyons_hyperbolic_2017, chang_signature_2017}. \cite{pfeffer_learning_2019} studies a restricted and approximate inversion problem from an algebro-geometric perspective of the path signatures, which was initiated in~\cite{amendola_varieties_2019}.
%
    
    One may consider a similar question for Lie groups. Suppose we fix a path $\gamma \in PG$, and we are given the path signature of a path $\alpha \in PH$, where $\alpha = F\gamma$, where $F: G \rightarrow H$ is a Lie group homomorphism. Is it possible to recover the Lie group homomorphism $F$?
    
    \item In~\cite{chen_integration_1958}, the path signature is defined for a manifold $M$ by choosing a collection of $1$-forms $\{\omega_i\}_{i=1}^N$ (as no natural Lie algebra basis is available). Choose $x \in M$ to be a basepoint, and let $S: PM_x \rightarrow T((\R^N))$ be the path signature defined with respect to these $1$-forms and let $\alpha, \beta \in PM_x$. If the $1$-forms span the cotangent bundle of $M$ at every point $x \in M$, Chen's injectivity result states that $S(\alpha) = S(\beta)$ if and only if $\alpha$ and $\beta$ are tree-like equivalent (in a slightly modified sense). The injectivity theorem given in Theorem~\ref{thm:injectivity} is a customization of Chen's result to the case $M = G$ since a basis of $\fg^*$ spans the cotangent bundle at every point.
    
    Recently, there has been interest in studying time series evolving on manifolds, and the development of a path signature kernel for paths on manifolds would provide a powerful tool for geometric time series analysis. One of the difficulties of this definition would be the representation of data on manifolds and the choice of $1$-forms used. One could begin with path signatures on \style{parallelizable} manifolds, which by definition admit a smooth basis of vector fields and $1$-forms. This would encompass all orientable 3-dimensional manifolds. 
\end{enumerate}

\clearpage
\acks{D.L. would like to acknowledge Chad Giusti and Jakob Hansen for several helpful discussions throughout this project. D.L. and R.G. are supported by the Office of the Assistant Secretary of Defense Research \& Engineering through ONR N00014-16-1-2010.  D.L. is also supported by the Natural Sciences and Engineering Research Council of Canada (NSERC) PGS-D3.}


\appendix
\section{Implementation details of path signature computation}
\label{apx:implementation}

We provide implementation details of the path signature computation for Lie groups, which includes the case for Euclidean space. Several libraries containing path signature computations for Euclidean paths exist, listed below.
\begin{enumerate}
    \item The \texttt{esig} package~\citep{lyons_esig_nodate} implemented in C++ and Python was one of the first libraries for path signature computations. It is a CPU-only and single-threaded implementation.
    \item The \texttt{iisignature} package~\citep{reizenstein_algorithm_2020} is also implemented in Python. It is also a CPU-only and single-threaded implementation.
    \item The \texttt{Signatory} library~\citep{kidger_signatory_2020} is the most recent and complete implementation of signature computations. It is written in C++ with a built-in Python wrapper. It includes functionality such as backpropagation which was not included in the previous packages. In addition, it provides GPU as well as both single and multi-threaded CPU support.
\end{enumerate}

Our current package is the first Julia implementation of path signature computations, and is also the first package to support matrix Lie group valued time series. We provide functions for both the continuous and discrete path signatures. The package can be found in \url{https://github.com/ldarrick/PathSignatures}. \medskip

Suppose $P: [T+1] \rightarrow G$ is a discrete time series in an $N$-dimensional matrix Lie group $G \subseteq GL_D(\R)$ (elements are $D \times D$ matrices), and let $p:[T] \rightarrow \fg$ be the discrete derivative of the time series. Note our change in notation for the paths in order to match the code. The continuous signature computation is performed by using Chen's identity (Equation~\ref{eq:chens_identity}) as described in Section~\ref{ssec:discrete_ps}. Specifically, once the discrete derivative is computed, the path signature is computed by tensor exponentiation and multiplication:
\begin{align*}
    S(P) = \exp_\otimes(p_1) \otimes \exp_\otimes(p_2) \otimes \ldots \otimes \exp_\otimes(p_T).	
\end{align*}

Our goal in this section is to compute the truncated discrete signature $\hat{S}_M(P)$. Recall that the discrete path signature with respect to the multi-index $I = (i_1, \ldots, i_m) \in [N]^m$ is computed as
\begin{align*}
    \hat{S}^I(\hat{\gamma}) \coloneqq \sum_{(t_1, \ldots, t_m) \in \hat{\Delta}^m_T} \omega_{i_1}(p_{t_1}) \ldots \omega_{i_m}(p_{t_m}).
\end{align*}
We note that the path signature computation relies only on the derivative $p$, so we split the computation into two steps.
\begin{enumerate}
    \item Compute the discrete derivative $p$ from a discrete time series $P$ in a Lie group $G$ with respect to a chosen basis of $\fg$. Note that this step is dependent on both the Lie group $G$, and the choice of basis of $\fg$.
    \item Compute the truncated discrete path signature $\hat{S}_M(\hat{\gamma})$ given the discrete derivative $p$. Note that this step is independent of both the Lie group $G$ as well as the choice of basis of $\fg$.
\end{enumerate}
This abstraction allows us to write a single path signature function, though we must write a new discrete derivative function for each Lie group we wish to consider. Let us first fix some notation used in the pseudo-code.

\begin{enumerate}
    \item \textbf{Matrix Multiplication.} For a $T_1 \times T_2$ array $A$ and a $T_2 \times T_3$ array $B$, we define
    \begin{align*}
        AB[i,j] = \sum_{k=1}^{T_2} A[i,k] B[k,j].
    \end{align*}
    \item \textbf{Element-wise Multiplication.} For two arrays $A$ and $B$ of the same size, we define
    \begin{align*}
        (A\cdot B)[i,j] = A[i,j] B[i,j].
    \end{align*}
    \item \textbf{Transpose.} For an array $A$, we define
    \begin{align*}
        A^{\intercal}[i,j] = A[j,i].
    \end{align*}
    \item \textbf{Sum.} For a $T_1 \times T_2$ array $A$, we define
    \begin{align*}
        A[\Sigma, \Sigma] = \sum_{i=1}^{T_1} \sum_{j=1}^{T_2} A[i,j].
    \end{align*}
    \item \textbf{Cumulative Sum.} For a $T_1 \times T_2$ array $A$, we define the $T_1 \times T_2$ array
    \begin{align*}
        A[\boxplus, \boxplus][i,j] = \sum_{k_1=1}^i \sum_{k_2=1}^j A[k_1,k_2].
    \end{align*}
    \item \textbf{End.} For a length $T$ array $A$, we define 
    \begin{align*}
        A[\texttt{end}] = A[T].
    \end{align*}
    \item \textbf{Splat.} For a length $T$ array $A$, we define the splat operator
    \begin{align*}
        A... = A[1], A[2], \ldots, A[T],
    \end{align*}
    to use the entries of $A$ as arguments for some function. 
\end{enumerate}

The discrete derivative function for a path $P: [T+1] \rightarrow \R^N$, which can be represented by a $((T+1) \times N)$ array, is simply a difference operation in the first dimension. Now, let's consider a path $P: [T+1] \rightarrow G$, where $G$ is an $N$-dimensional matrix Lie group $G \subseteq GL_D(\R)$. This path can be represented by a $((T+1) \times D \times D)$ array. Suppose that we have implemented a function $\texttt{convert\_to\_basis()}$ which takes in a $(D \times D)$ array, which represents an element of $\fg$, and converts it into a length $N$ vector. Let $\texttt{logm()}$ be the matrix logarithm function.

\begin{algorithm}
\SetKwInOut{Input}{Input}\SetKwInOut{Output}{Output}
\SetAlgoLined
\Input{$\texttt{P}$ ($((T+1) \times D \times D)$ array): A path in the matrix Lie group $G$}
\Output{$\texttt{p}$ ($(T \times N)$ array): The discrete derivative}
Initialize $(T \times N)$ array $\texttt{p}$ \;
Initialize $(D \times D)$ array $\texttt{p}$ \;
 \For{$\texttt{t=1..T}$}{
    \nosemic Compute the discrete derivative at time $\texttt{t}$\;
    \pushline \dosemic \nonl $\texttt{d}\leftarrow \texttt{logm(P[t+1,:,:]}^{\texttt{-1}}\texttt{P[t,:,:])}$\;
    \popline \nosemic Convert the derivative to a representation in terms of a chosen basis of $\fg$\;
    \pushline \dosemic \nonl $\texttt{p[t,:]} \leftarrow \texttt{convert\_to\_basis(d)}$\;
 }
 Return $\texttt{p}$\;
 \caption{Discrete Derivative for Matrix Lie Groups}
\end{algorithm}

Next, we will describe the discrete signature computation. Note that the following algorithm is not novel, and is similar to the implementation provided in~\cite{kidger_signatory_2020}. The idea is to compute the truncated signature using forward recursion, using the computations for lower levels to also compute higher levels. Thus, this computation is made up of two functions: the initialization and the recursion.
The truncated signature will be an element of $T^{(\leq M)}(\R^N)$. This is represented as a 1D array of multidimensional arrays, which we label $s$. For example, $s[m]$ is an $m$-dimensional array of length $N$ in each dimension, and we denote access to the inner arrays using separate square brackets. Namely, $s[2][3,4]$ is the $[3,4]$ element of the 2D array $s[2]$. 

\begin{algorithm}
\SetKwInOut{Input}{Input}\SetKwInOut{Output}{Output}
\SetAlgoLined
\Input{$\texttt{p}$ ($(T \times N)$ array): The discrete derivative \\ 
$\texttt{M}$ (Int): The truncation level}
\Output{$\texttt{s}$ (Tensor polynomial): The truncated discrete signature}
Initialize 1D array of multidimensional arrays $\texttt{s}$\;
Initialize length $M$ array $\texttt{cur\_ind}$\;
Initialize length $T$ array $\texttt{Q}$\;
 \For{$\texttt{n=1..N}$}{
    \nosemic Set the first index \;
    \pushline \dosemic \nonl $\texttt{cur\_ind[1]} \leftarrow \texttt{n}$\;
    \popline \nosemic Compute first level signature path and store in $\texttt{s}$\;
    \pushline \dosemic \nonl $\texttt{Q} \leftarrow \texttt{K[}\boxplus\texttt{,n]}$\;
     \dosemic \nonl $\texttt{s[1][n]} \leftarrow \texttt{Q[end]}$\;
    \popline \nosemic Perform forward recursion \;
    \pushline \dosemic \nonl $\texttt{s}\leftarrow \texttt{sig\_forward(p, Q cur\_ind, 2, M)}$\;
 }
 Return $s$
 \caption{Discrete Signature Initialization}
\end{algorithm}

\begin{algorithm}
\SetKwInOut{Input}{Input}\SetKwInOut{Output}{Output}
\SetAlgoLined
\Input{$\texttt{s}$ (Tensor polynomial): The current signature\\
$\texttt{p}$ ($(T \times N)$ array): The discrete derivative\\
$\texttt{lastQ}$ (Length $T$ array): The previous signature path\\
$\texttt{cur\_ind}$ (Length $M$ array): The current index\\
$\texttt{cur\_level}$ (Int): The current level\\
$\texttt{M}$ (Int): The truncation level
}
\Output{$\texttt{s}$ (Tensor polynomial): Partially computed truncated discrete signature}
Initialize length $T$ array $\texttt{Q}$\;
\eIf{$\texttt{cur\_level < M}$}{
    \For{\texttt{n=1..N}}{
        \nosemic Update the current index \;
        \pushline \dosemic \nonl $\texttt{cur\_ind}[\texttt{cur\_level}] \leftarrow \texttt{n}$\;
        \popline \nosemic Compute the current signature path\;
        \pushline \dosemic \nonl $\texttt{Q} \leftarrow \texttt{lastQ}\cdot \texttt{p[:,n]}$ \;
        \dosemic \nonl $\texttt{Q} \leftarrow \texttt{Q[}\boxplus\texttt{]}$ \;
        \popline \nosemic Store the current signature\;
        \pushline \dosemic \nonl $\texttt{s[cur\_level][cur\_ind[1:cur\_level]...]} = \texttt{Q[end]}$ \;
        \popline \nosemic Perform forward recursion \;
        \pushline \dosemic \nonl $\texttt{s} \leftarrow \texttt{sig\_forward(p, Q, cur\_ind, cur\_level+1, M)}$ \;
    }
}{
    \For{$\texttt{n=1..N}$}{
        \nonl \nosemic $\backslash \backslash$ Last level: perform the same operations, but don't perform recursion \;
        $\texttt{cur\_ind[cur\_level]} \leftarrow n$ \;
        $\texttt{Q} \leftarrow \texttt{lastQ}\cdot \texttt{p[:,n]}$\;
        $\texttt{s[cur\_level][cur\_ind...]} \leftarrow \texttt{Q(}\Sigma\texttt{)}$ \;
    }
}
Return $\texttt{s}$
 \caption{Signature Forward Recursion ($\texttt{sig\_forward}$)}
\end{algorithm}

Here, we may easily compute the complexity of the truncated signature algorithm in terms of the number of elementary operations. We note that there are $N^M$ entries of the level $M$ signature, and all lower levels are computed along the way. To compute an entry of the level $M$ signature from the level $M-1$ signature, we require $O(T)$ elementary operations. Thus, the complexity of this algorithm is $O(TN^M)$.

\clearpage
\section{Tensor Normalization}
\label{apx:tensornorm}
In this appendix, we will discuss the construction of a tensor normalization, as shown in Appendix A of~\cite{chevyrev_signature_2018}, as well as the computation of the normalization. The following proposition provides the construction that we will follow. Suppose $H$ is a Hilbert space throughout this appendix.

\begin{proposition}
\label{prop:tensconstruction}
Let $\psi: [1, \infty) \rightarrow [1, \infty)$ with $\psi(1) = 1$. For $\mathbf{t} \in T_1((H))$ and define $\lambda: T_1((H)) \rightarrow (0, \infty)$ to be the unique non-negative number such that $\|\delta_{\lambda(\mathbf{t})} \mathbf{t}\|^2 = \psi(\|\mathbf{t}\|)$. Define
\begin{align*}
    \Lambda : T_1((H)) & \rightarrow T_1((H)) \\
    \mathbf{t} & \mapsto \delta_{\lambda(\mathbf{t})} \mathbf{t}.
\end{align*}
Denote further $\|\psi\|_\infty = \sup_{x \geq 1} \psi(x)$. Then the following holds.
\begin{enumerate}
    \item The function $\Lambda$ takes values in the set $\left\{ \mathbf{t} \in T_1((H)) \, : \, \|\mathbf{t}\| \leq \sqrt{\|\psi\|_\infty}\right\}$.
    \item If $\psi$ is injective, then so is $\Lambda$.
    \item Suppose that $\sup_{x \geq 1} \psi(x)/s^2 \leq 1$, $\|\psi\|_\infty < \infty$, and that $\psi$ is $K$-Lipschitz for some $K > 0$. Then
    \begin{align*}
        \|\Lambda(\mathbf{s}) - \Lambda(\mathbf{t})\| \leq (1 + \sqrt{K} + 2 \sqrt{\|\psi\|_\infty})(\sqrt{\|\mathbf{s} - \mathbf{t}\|} \vee \|\mathbf{s} - \mathbf{t}\|).
    \end{align*}
\end{enumerate}
\end{proposition}

\begin{corollary}
    Let $\psi: [1, \infty) \rightarrow [1, \infty)$ be injective satisfying $\psi(1) = 1$ and the conditions of item (3.) in Proposition~\ref{prop:tensconstruction}. Then, the function $\Lambda$ constructed in Proposition~\ref{prop:tensconstruction} is a tensor normalization.
\end{corollary}

This method of constructing a tensor normalization is done by defining the normalized norm $\|\delta_{\lambda(\mathbf{t})} \mathbf{t}\|$ of an element $\mathbf{t}$, given its original norm $\psi(\|\mathbf{t}\|)$. In practice, we will be working with a truncated element of the tensor algebra. Thus, given an element $\mathbf{t} \in T_1^{(\leq M)}((H))$, we can obtain the value of $\lambda(\mathbf{t})$ by solving the equation $\|\delta_{\lambda(\mathbf{t})} \mathbf{t}\|^2 = \psi(\|\mathbf{t}\|)$. This reduces to finding the zero of the polynomial equation
\begin{align*}
    \sum_{k=1}^M \lambda^{2k} \|\mathbf{t}_k\|^2 - \psi(\|\mathbf{t}\|) = 0.
\end{align*}

\bibliographystyle{natbib}
\bibliography{liegroups}

\end{document}